\documentclass{article}
\usepackage[T1]{fontenc}
\usepackage[utf8]{inputenc}
\usepackage[a4paper, margin=0.8in]{geometry}

\usepackage{graphicx}
\usepackage{subcaption}

\usepackage[colorlinks,urlcolor=blue,citecolor=blue,linkcolor=blue]{hyperref}

\usepackage{amsmath}
\usepackage{amssymb}
\usepackage{mathtools}
\usepackage{amsthm}

\usepackage[capitalize,noabbrev]{cleveref}
\usepackage{bbm}
\usepackage{overpic}
\usepackage{enumitem}
\setlist[itemize,1]{leftmargin=1.6em}
\setlist[enumerate,1]{leftmargin=1.8em}
\usepackage{nicematrix} 
\usepackage{stmaryrd} 
\usepackage{natbib}
\usepackage{tabularx}
\usepackage{pifont}
\usepackage[dvipsnames]{xcolor}

\usepackage{changes}
\let\highlight\relax
\usepackage{url}            
\usepackage{booktabs}       
\usepackage{amsfonts}       
\usepackage{nicefrac}       
\usepackage{microtype}      
\usepackage{soul}

\sethlcolor{lightgray!25}

\usepackage{authblk}
\usepackage{graphicx}
\usepackage{subcaption}
\usepackage{booktabs}
\usepackage{array}
\usepackage{amsmath}
\usepackage{amssymb}
\usepackage{mathtools}
\usepackage{amsthm}
\usepackage{bm}
\usepackage{thmtools}
\usepackage{thm-restate}
\usepackage{etoc}
\hypersetup{
    colorlinks = true,
    linkcolor = blue,
    citecolor=blue,
    urlcolor=black,
}
\definecolor{mblue}{rgb}{0.0,0.0,1.0}
\definecolor{mgreen}{rgb}{0.0,0.501960784314,0.0}
\definecolor{mred}{rgb}{1.0,0.0,0.0}
\usepackage[capitalize]{cleveref}
\usepackage{algorithm}
\usepackage[noend]{algpseudocode}

\usepackage{bbm}
\usepackage{xcolor}
\usepackage{parskip}
\usepackage{setspace}
\usepackage{comment}
\usepackage{caption}
\usepackage{enumitem}
\usepackage{wrapfig}

\newtheorem{theorem}{Theorem}[section]
\newtheorem{example}{Example}[section]

\newtheorem{corollary}{Corollary}[section]

\newtheorem{lemma}{Lemma}[section]

\newtheorem{assumption}{Assumption}[section]

\newcommand{\highlight}[1]{\textcolor{blue!70!black}{#1}}
\usepackage{xcolor}

\definecolor{supportblue}{RGB}{24,90,160}
\definecolor{goodgreen}{RGB}{20,125,70}
\definecolor{neutralgray}{RGB}{95,95,95}
\definecolor{badred}{RGB}{180,55,45}
\definecolor{projectpurple}{RGB}{115,65,160}

\newcommand{\towardsupport}[1]{\textcolor{goodgreen}{\textbf{#1}}}
\newcommand{\neutralscale}[1]{\textcolor{neutralgray}{\textbf{#1}}}
\newcommand{\awayfromsupport}[1]{\textcolor{badred}{\textbf{#1}}}

\DeclareMathOperator*{\argmin}{arg\,min}
\usepackage{changepage}

\usepackage{tikz}
\usetikzlibrary{arrows.meta,positioning}

\newcommand\pData{p_{\text{data}}}
\newcommand\pUnif{p_{\text{unif}}}
\newcommand\pMask{p_{\text{mask}}}
\newcommand\Pm{\mathbb{P}}
\newcommand\supp{\text{supp}}

\newcommand\proj{\text{proj}}

\newcommand\R{\mathbb{R}}

\newcommand\indicator[1]{\mathbbm{1}\left\{ #1 \right\}}
\newcommand\Simplex[1]{\Delta_{#1}}

\newcommand\Xc{\mathcal{X}}

\DeclareMathOperator*{\E}{\mathbb{E}}

\newcommand\Dkl{D_{\text{KL}}}

\newcommand\pHat{\hat{p}}

\newcommand{\bc}[1]{\left\{{#1}\right\}}
\newcommand{\br}[1]{\left({#1}\right)}

\makeatletter
\providecommand*{\bigcupdot}{%
  \mathop{%
    \vphantom{\bigcup}%
    \mathpalette\@bigcupdot{}%
  }%
}
\newcommand*{\@bigcupdot}[2]{%
  \ooalign{%
    $\m@th#1\bigcup$\cr
    \sbox0{$#1\bigcup$}%
    \dimen@=\ht0 %
    \advance\dimen@ by -\dp0 %
    \sbox0{\scalebox{1.5}{$\m@th#1\cdot$}}%
    \advance\dimen@ by -\ht0 %
    \dimen@=.5\dimen@
    \hidewidth\raise\dimen@\box0\hidewidth
  }%
}
\makeatother

\newcommand{\metricplot}[2]{%
  \includegraphics[width=\linewidth]{#1/#2}%
}

\title{Support Before Frequency in Discrete Diffusion}

\author{
Adrian M\"uller$^{1}$ \quad
Antoine Gonon$^{2}$ \quad
Zebang Shen$^{1}$ \quad
Ya-Ping Hsieh$^{1,3}$ \quad
Niao He$^{1}$
\and
\small $^{1}$ Department of Computer Science, ETH Z\"urich, Switzerland
\quad
\small $^{2}$ Institute of Mathematics, EPFL, Switzerland
\and
\small $^{3}$ Institute of Statistical Science, Academia Sinica, Taiwan
}

\begin{document}

\maketitle

\begin{abstract}
Discrete diffusion models are increasingly competitive for language modeling,
yet it remains unclear how their denoising objectives organize learning.
Although these objectives target the full data distribution, we show that the
exact reverse process induces a hierarchy between coarse \emph{support}
information and finer \emph{frequency} information. For uniform and absorbing
(a.k.a. masking) diffusion, we prove that, in the small-noise regime of the
final denoising steps, each single-token reverse edit decomposes into a leading
scale, determined by whether it moves toward the data \emph{support} (e.g.,
grammatically valid sentences), and a finer coefficient, determining
\emph{relative probabilities} within the same scale.
Thus, recovering validity structure only requires learning
the correct order of magnitude of reverse probabilities, whereas recovering data frequencies requires
coefficient-level estimation. The separation is mechanism-dependent: uniform diffusion exhibits a
trichotomy into validity-improving, validity-preserving, and validity-worsening edits, while absorbing
diffusion places its leading-order mass on validity-improving moves. Experiments on a masked language
diffusion model and synthetic regular-language tasks support these predictions: support-localization
emerges earlier than within-support frequency ranking, and the contrast between uniform and absorbing
diffusion matches the predicted rate separation. Together, our results suggest that discrete diffusion
models learn data support before data frequencies.
\end{abstract}

\section{Introduction}

Since the influential work of \citet{austin2021structured}, 
discrete Diffusion Language Models (DLMs) have gained significant traction as a promising alternative to auto-regressive models due to their potential for faster inference through parallelization \citep{austin2021structured,wu2025fast,arriola2025block,israel2025accelerating} and improved controllability of generation \citep{li2022diffusion}. These advantages have
driven rapid progress in DLM design and scaling, with recent models beginning to
close the gap with state-of-the-art autoregressive language models
\citep{googledeepmind2025geminidiffusion,labs2025mercury,song2025seed}; see
\citet{li2025survey} for a survey.

\paragraph{Central hypothesis.}

Most successful DLMs rely on reversing either \emph{uniform} or \emph{absorbing}
(masking) diffusion processes \citep{austin2021structured}. While a large body
of work studies how to parameterize, train, and sample from these models, our
focus is different: we ask \textbf{{what structure of the data
distribution is exposed first by the reverse denoising problem}}. Our central
thesis is the following.

\noindent\makebox[\linewidth][c]{%
\setlength{\fboxsep}{3pt}%
\colorbox{yellow!15}{%
\begin{minipage}{0.94\linewidth}
\textbf{\highlight{Support-before-Frequency Hypothesis:}}
DLMs first learn where admissible sequences are, and only later refine the
relative probabilities among admissible sequences.
\end{minipage}%
}%
}

Concretely, write $D\coloneqq\supp(\pData)$
for the support of the population data distribution, i.e., the set of sequences
with nonzero probability under \(\pData\). Intuitively, \(D\) is the set of
admissible sequences, such as context-free languages for programming code generation, 
while the frequencies are the probabilities \(\pData(x)\) assigned to
strings \(x\in D\). In these terms, our hypothesis says that DLMs first recover
support information about \(D\), in a sense made precise below, before they
accurately calibrate the probabilities \(\pData(x)\) within \(D\).

\paragraph{Small-noise expansion of DLMs.}

Our hypothesis is motivated by a small-noise analysis of the \emph{exact}
reverse kernel, corresponding to the final denoising steps of the diffusion
process. This kernel is the population object targeted by standard discrete
diffusion training objectives \citep{austin2021structured}. Thus, its low-noise structure may expose an information hierarchy shared across training paradigms that approximate the true reverse process.

For a string \(x\), let $d$ be the Hamming distance $d(x,z)\coloneqq |\{i: x_i\neq z_i\}|$, and let
\(
\proj_D(x)\coloneqq\{z\in D:d(x,z)=d(x,D)\}
\)
denote the set of nearest in-support strings to \(x\). Let \(\sigma\) be the
noise level, and consider the low-noise limit \(\sigma\to0^+\). 
For a one-token reverse edit from a current string to a candidate string, our main
theorems (\Cref{thm:uniform-rate-separation,thm:absorbing-rate-separation}) show
that the reverse edit probability has the schematic form (up to normalization)
\begin{equation}
\tag{$\star$}
\label{eq:intro-expansion}
   \text{reverse edit probability}
\;\approx\;
\underbrace{\phantom{\frac{j_D}{j_D}}\Gamma_{\mathsf{corr}}(\Delta d)\phantom{\frac{1}{1}}}_{\text{corruption gate}}
\cdot
\underbrace{\phantom{\frac{j_D}{j_D}}\sigma^{\Delta d}\phantom{\frac{1}{1}}}_{\text{\textbf{scale}: support signal}}
\cdot
\underbrace{
\frac{\pData(\proj_D(\text{candidate}))}
{\pData(\proj_D(\text{current}))}
}_{\text{\textbf{coefficient}: frequency signal}},
\end{equation}
where $\Delta d
=
d(\text{candidate},D)-d(\text{current},D)$, 
and the factor
\(\Gamma_{\mathsf{corr}}\) records the effect of the corruption mechanism:
\[
\Gamma_{\mathsf{corr}}(\Delta d)
\coloneqq
\begin{cases}
1, & \mathsf{corr}=\mathsf{unif}, \\[1mm]
\mathbf 1\{\Delta d=-1\}, & \mathsf{corr}=\mathsf{mask}.
\end{cases}
\]

At first sight, \eqref{eq:intro-expansion} appears to mix support and frequency
information in a single leading expression. 
The key point, however, is that \emph{these two pieces enter at different resolutions, and that the corruption mechanism determines which scales are active through the gate function $\Gamma_{\mathrm{corr}}$.} This leads to the two concrete predictions about practical DLMs, which we test in both synthetic and real-data experiments.

\begin{enumerate}[
    label=\textbf{\arabic*.},
    leftmargin=1.0em,
    itemsep=0.25em,
    topsep=0.1em,
    parsep=0pt,
    partopsep=0pt
]
    \item \textbf{\highlight{Prediction 1: support before frequency.}}
    At low noise, recovering \emph{support-improving directions} only requires
    the model to learn the rough \textbf{order-of-magnitude} in \(\sigma\) of the corresponding denoising proposal, to determine whether it moves towards the support. As \(\sigma \to 0\), these different orders become increasingly separated, making the support structure  easier to detect from coarse scaling information alone. By contrast, recovering the \emph{frequencies} of the data distribution requires accurate
    estimation of the \textbf{coefficients} within a fixed order of magnitude.
    Thus, a trained model may acquire a support-like denoising field before it
    matches fine within-support frequencies.

    We make this intuition precise in \Cref{cor:coarse-score-projection}: if the
    learned reverse probabilities approximate the true reverse probabilities up
    to multiplicative error
    \(
    o(\sigma^{-1/2}),
    \)
    then thresholding the scaled reverse probabilities exactly recovers the
    support-improving directions. Importantly, this condition can hold even when
    the learned transition kernel remains far from the true kernel in additive
    metrics such as total variation or KL; see \Cref{ex:multiplicative-not-tv} for an illustration.

    \item \textbf{\highlight{Prediction 2: masking is closer to a support projector.}}
    The same expansion predicts a qualitative contrast between corruption
    mechanisms. For \emph{uniform diffusion}, \towardsupport{support-improving},
    \neutralscale{support-preserving}, and \awayfromsupport{support-worsening} edits appear at three distinct
    scales:
    \[
    {\mathbf{\color{goodgreen}\sigma^{-1}}}, \mathbf{{\color{neutralgray}\qquad 1}}, \qquad {\mathbf{\color{badred}\sigma}}.
    \]
    For \emph{absorbing}, or \emph{masked}, diffusion, the structure is sharper:
    only support-improving unmasking moves contribute at leading order, while
    non-improving moves vanish in the small-noise limit. 
\end{enumerate}

\begin{figure}[t]
    \centering
    \includegraphics[width=.72\linewidth]{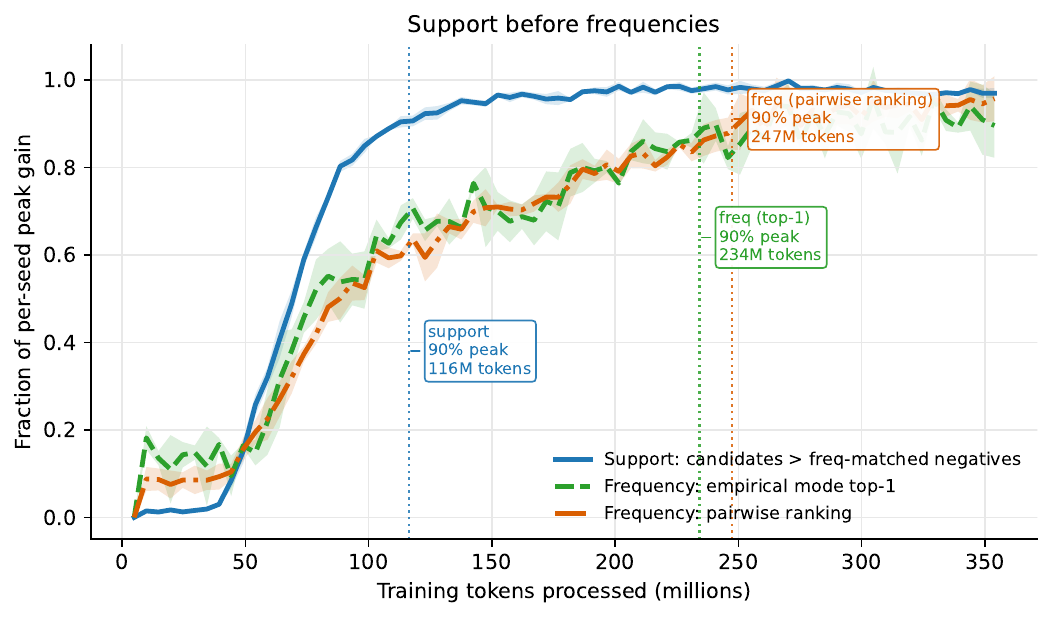}
    \caption{\textbf{Support before frequencies in a web-trained masked DLM.}
    We train a masked DLM on FineWeb and evaluate support and frequency proxies inspired by the separation in \Cref{thm:absorbing-rate-separation} (\Cref{sec:xp-supp-then-freq}). The support-localization proxy reaches its peak gain earlier than
    the frequency-ranking proxies. Curves show means over three seeds with
    \(\pm1\) standard-deviation bands;
    transition markers use the first checkpoint reaching \(90\%\) of each seed's
    peak gain.}
    \label{fig:real-dlm-support-frequency}
\end{figure}

\Cref{sec:xp-supp-then-freq} tests
\textbf{\highlight{Prediction 1}} by designing probes that disentangle
\emph{support localization} from \emph{frequency ranking} for a masked DLM
trained on FineWeb; see \Cref{fig:real-dlm-support-frequency}. This setting
captures a widely used DLM family: an absorbing-mask process trained with a
weighted cross-entropy objective, which coincides with several common
formulations, including D3PM-style posterior prediction, SUBS/mean
parameterizations, and the induced-score view; see
\Cref{app:dlm-parameterizations}. Across seeds, the support-localization probe
reaches its peak gain substantially earlier than the frequency-ranking probes.
This temporal separation supports our hypothesis: although the training
objective can in principle learn support identification and frequency refinement
simultaneously, support identification emerges and stabilizes before the model
refines relative probabilities within the support.

\Cref{sec:xp-abs-vs-unif} tests \textbf{\highlight{Prediction 2}}, namely the predicted contrast between \emph{uniform} and \emph{absorbing} diffusion. To isolate the projection effect, we use controlled synthetic experiments on \emph{regular languages}, where membership in the data support can be evaluated exactly. We then apply a theory-guided \emph{inference-time thresholding} procedure, requiring no additional training, to isolate the leading \(\sigma^{-1}\)-scale component of the learned reverse scores. As predicted, this intervention improves support recovery for uniform diffusion, but gives little additional benefit for absorbing diffusion. This matches the exact reverse-kernel expansion: absorbing diffusion already suppresses non-projective moves at leading order, whereas uniform diffusion retains lower-scale non-projective components that can be removed by isolating the dominant scale. Finally, we stress that our purpose is to highlight the distinct structural biases in the reverse dynamics,
rather than endorsing one corruption process over the other. While absorbing
diffusion is closer to a support projector at small noise, masking also limits
sampling flexibility by preventing incorrectly unmasked tokens from being
revised.



\paragraph{Contributions.}
To summarize, our main contributions are threefold.

First, in \Cref{thm:uniform-rate-separation,thm:absorbing-rate-separation}, we derive small-noise expansions of the exact reverse kernels for both \emph{uniform} and \emph{absorbing} diffusion. These expansions identify a scale separation between support and frequency information, and reveal a sharp contrast between the two corruption mechanisms.

Second, in \Cref{cor:coarse-score-projection}, we formalize why coarse multiplicative accuracy of the reverse scores is sufficient to recover support-improving directions, even when the learned reverse kernel remains far from the true kernel in additive metrics such as total variation.

Third, in \Cref{sec:experiments}, we design probes for support-localization and frequency-ranking inspired by the theory, and show that trained DLMs exhibit the predicted separations in both real-data and synthetic experiments.

\paragraph{Relation to prior work.} The conceptually closest perspective to our hypothesis is the recent observation that continuous score matching may recover geometric information about a data manifold before learning the full density on that manifold \citep{li2025scores}. Our work provides a discrete analogue for discrete spaces, without assuming any latent manifold structure: the ambient geometry is instead the Hamming graph, and the analogue of the manifold is the arbitrary support set $\supp(\pData)$.

\section{Small-Noise Expansion of the Reverse Process Generator} \label{sec:expansion}

This section formalizes the aforementioned small-noise expansion \eqref{eq:intro-expansion} of the exact reverse kernel. 

\subsection{Recap of Uniform and Masking Diffusions}

We begin by recalling the standard discrete-time formulation of diffusion models on finite state spaces, following \citet{austin2021structured}. Our focus is on the two corruption mechanisms most relevant to diffusion language modeling: \emph{uniform diffusion} and \emph{absorbing} (a.k.a. \emph{masking}) diffusion.\footnote{Similar small-noise expansions can be derived for other discrete diffusions, with the Hamming distance replaced by the appropriate shortest-path distance on the underlying diffusion graph. We focus on uniform and absorbing diffusion because they capture the main design choices used in current text diffusion models.}

\paragraph{General setup.} We consider a finite token space $[K]\coloneqq\{1,\dots,K\}$, a fixed sequence length (dimension) $H$, and the corresponding set of all possible sequences $\Xc\coloneqq[K]^H$. We denote the time horizon of the diffusion process by $T$, meaning that generation is done in a $T$-step denoising procedure with time discretization $1/T$. We consider forward diffusion processes that are independent across token positions, given via a starting discrete distribution $q_0 \coloneqq \pData \in \Delta_\Xc$ on the simplex and the Markov kernels $q_{t|t-1}(x_t | x_{t-1}) = \prod_{i=1}^H q_{t|t-1}^i(x_t^i | x_{t-1}^i) = \prod_{i=1}^H [Q^t]_{x_{t-1}^i,x_t^i}$ ($t\in[T]$) for fixed transition matrices $Q^t \in \R^K\times \R^K$. For $t\in[T]$, we let $X_t \coloneqq (X_t^1, \dots, X_t^H) \in \Xc$ denote the random variable this process determines. We abbreviate $q_t(x_t)\coloneqq\Pm[X_t=x_t]$ and $q_{t|s}(x_t|x_s) \coloneqq \Pm[X_t = x_t | X_s = x_s]$ for all $x_t, x_s \in [K]$ such that $\Pm[X_s = x_s] > 0$.  

\paragraph{Uniform diffusion.} For a fixed $q_0 = \pData$, consider the uniform forward diffusion process given by
\begin{align*}
    q_{t|t-1}^i(y|x) \coloneqq (1-\beta_t) \indicator{y = x} + \frac{\beta_t}{K} \tag{uniform}
\end{align*}
for all $i\in[H]$, $x,y\in[K]$.

\paragraph{Absorbing (masking) diffusion.} We derive analogous results for the absorbing diffusion process, which is similarly given by $q_0 = \pData$ and 
\begin{align*}
    q_{t|t-1}^i(y|x) \coloneqq (1-\beta_t) \indicator{y = x} + \beta_t \indicator{y=m} \tag{absorbing}
\end{align*}
for all $i\in[H]$, $x,y\in[K]$, where $m$ is a special masking token not contained in $[K]$. 

When setting the \emph{cumulative noise level} at time step $t$ (probability of being noised before or at time $t$) to be $\sigma_t \coloneqq 1-\prod_{s=1}^t (1-\beta_s)$, we can easily view the induced marginals for token index $i$ as
\begin{align*}
    \text{uniform:}\quad q^i_{t} =& (1-\sigma_t) \pData^i + \sigma_t \pUnif ,\qquad\qquad
    \text{absorbing:} \quad q^i_{t} =& (1-\sigma_t) \pData^i + \sigma_t \pMask, 
\end{align*}

\paragraph{Fitting the reverse kernel.} The way standard DLMs work is to fit the \emph{reverse transition} $q_{t-1|t}$ with a model $p_{t-1|t}$ by minimizing the classical variational bound on the data likelihood (see \cref{app:dlm-parameterizations}). This model is then used to stepwise denoise, starting from the source distribution $x_T \sim \pUnif^{\otimes H}$ (resp. $x_T\sim\pMask^{\otimes H}$) via $x_{t-1} \sim p_{t-1|t}(\cdot | x_t)$. In practice, $p_{t-1|t}$ is virtually always modeled as independent across token positions, meaning that $p_{t-1|t}(y|x_t)=\prod_{h=1}^H p_{t-1|t}^h(y^h | x_t)$ and each component on the right-hand side is parameterized by a neural network.

\subsection{Main Results: Small-Noise Expansion of Reverse Kernels} \label{subsec:main}

We now derive the small-noise expansions of the exact reverse kernels previewed in the introduction. To make the results applicable across different discretization schemes and noise schedules, we formulate the analysis under a general condition that captures the \emph{continuous-noise limit} underlying discrete diffusion: any fixed target noise level can be approximated arbitrarily well, while each individual denoising step becomes infinitesimal. This condition is mild and is satisfied by standard schedules used in practice, including linear cumulative-noise and cosine schedules.

\vspace{2mm}
\begin{assumption}[Noise schedule]
\label{assump:noise-schedule}
\label{assumption:schedule}
Let \((\sigma_t)_{t\in[T]}\) be the cumulative noise levels induced by
\((\beta_t)_{t\in[T]}\). For every fixed \(\sigma\in(0,1)\), let
\begin{equation}
\tag{$t_\sigma$}
\label{eq:t-sig-def}
t_\sigma(T) \in \argmin_{s\in[T]} |\sigma_s-\sigma|
\end{equation}
be a time index whose noise level is closest to \(\sigma\). We assume that $\sigma_{t_\sigma(T)} \to \sigma, \beta_{t_\sigma(T)} \to 0 $ as \(T\to\infty\).
\end{assumption}

For \(x\in\Xc\), a coordinate \(h\in[H]\), and a proposed token \(y^h\in[K]\), write
\(
x^{h\to y^h}\coloneqq(x^{-h},y^h)
\)
for the string obtained by replacing the \(h\)-th token of \(x\) by \(y^h\).
Recall that $d$ denotes the Hamming distance and 
\[
\proj_D(x)\coloneqq\{z\in D:d(x,z)=d(x,D)\},
\qquad
\pData(\proj_D(x))\coloneqq\sum_{z\in\proj_D(x)}\pData(z).
\]

\paragraph{Small-noise expansion: Uniform diffusion.}
We are now ready to state the main theoretical results. The first result shows that, for uniform diffusion, one-token reverse edits split into \emph{three} scale classes.

\vspace{2mm}

\begin{theorem}[Rate separation for uniform diffusion]
\label{thm:uniform-rate-separation}
\label{thm:main}
Let \(q\) denote the law of the uniform diffusion process. Fix a state
\(x_t\in\supp(q_t)=\Xc\), a coordinate \(h\in[H]\), and a proposed replacement
token \(y^h\in[K]\setminus\{x_t^h\}\). Define the corresponding one-token edit by $x_{t-1}\coloneqq x_t^{h\to y^h}.$
Let \(\sigma\in(0,1)\), and let \(t=t_\sigma(T)\) be the time index specified in \eqref{eq:t-sig-def}. Then, under \Cref{assump:noise-schedule}, as \(\sigma\to0\),
\[
\lim_{T\to\infty}
\frac{q^h_{t-1|t}(y^h\mid x_t)}{\beta_t/K}
=
\begin{cases}
\displaystyle
\left(\frac{\sigma}{K}\right)^{-1}
\frac{\pData(\proj_D(x_{t-1}))}{\pData(\proj_D(x_t))}
+ O(1),
&
\begin{gathered}
\text{\towardsupport{toward support}}\\
d(x_{t-1},D)<d(x_t,D),
\end{gathered}
\\[1.5em]
\displaystyle
\frac{\pData(\proj_D(x_{t-1}))}{\pData(\proj_D(x_t))}
+ O(\sigma),
&
\begin{gathered}
\text{\neutralscale{same distance}}\\
d(x_{t-1},D)=d(x_t,D),
\end{gathered}
\\[1.5em]
\displaystyle
\left(\frac{\sigma}{K}\right)
\frac{\pData(\proj_D(x_{t-1}))}{\pData(\proj_D(x_t))}
+ O(\sigma^2),
&
\begin{gathered}
\text{\awayfromsupport{away from support}}\\
d(x_{t-1},D)>d(x_t,D).
\end{gathered}
\end{cases}
\]

\end{theorem}

Equivalently, up to fixed vocabulary factors and lower-order terms, the normalized
one-token reverse probability satisfies
\[
\frac{q^h_{t-1|t}(y^h\mid x)}{\beta_t/K}
=
\Theta\left(
\sigma^{d(x^{h\to y^h},D)-d(x,D)}
\frac{\pData(\proj_D(x^{h\to y^h}))}{\pData(\proj_D(x))}
\right).
\]
This expression separates the reverse edit probability into two conceptually
different components. The exponent
\[
\Delta d\coloneqq d(x^{h\to y^h},D)-d(x,D)
\]
is the \emph{support signal}: it indicates whether the proposed edit moves the
current string closer to the support \(D\), keeps it at the same distance, or
moves it further away. By contrast, the ratio
\[
\frac{\pData(\proj_D(x^{h\to y^h}))}{\pData(\proj_D(x))}
\]
is a \emph{frequency-sensitive coefficient}: once the noise scale $\sigma$ of the
edit is fixed, this coefficient determines how probability is distributed among
edits of the same scale.

Consequently, in the small-noise regime, the reverse kernel has a clear
\emph{three-scale} structure: Support-improving edits, for which \(\Delta d=-1\),
appear at scale \(\sigma^{-1}\); distance-preserving edits, for which
\(\Delta d=0\), appear at scale \(1\); and support-worsening edits, for which
\(\Delta d=1\), appear at scale \(\sigma\). Thus the dominant entries of the
low-noise reverse kernel are precisely the local edits that move the string
toward the data support.

\paragraph{Small-noise expansion: Absorbing diffusion.}

Our next goal is to show that the same \emph{support-seeking} principle takes an even sharper form for absorbing, or
masking, diffusion. At a masked coordinate,
a reverse move can only replace the mask by a clean token. Since the mask token
does not appear in any string in the data support \(D\), such a move can never
increase the Hamming distance to \(D\). Thus the uniform-diffusion trichotomy
collapses: reverse moves are either support-improving or support-preserving. The
next result shows that, in the small-noise limit, only the support-improving
moves survive at leading order.

\vspace{2mm}

\begin{theorem}[Rate separation for absorbing diffusion]
\label{thm:absorbing-rate-separation}
\label{thm:main-absorb}
Let \(q\) be the law of the absorbing diffusion process.
In the same setting as \Cref{thm:uniform-rate-separation}, but with \(x_t^h=m\) a masked token,
\(y^h\in[K]\), and \(x_t\in\supp(q_t)\), we have, as \(\sigma \to 0\),
\[
\lim_{T\to\infty}
\frac{q^h_{t-1|t}(y^h\mid x_t)}{\beta_t}
=
\begin{cases}
\displaystyle
\sigma^{-1}
\frac{\pData(\proj_D(x_{t-1}))}{\pData(\proj_D(x_t))}
+ O(1),
&
\begin{gathered}
\text{\towardsupport{support-improving}}\\
d(x_{t-1},D)<d(x_t,D),
\end{gathered}
\\[1.5em]
0,
&
\begin{gathered}
\text{\neutralscale{non-improving}}\\
d(x_{t-1},D)=d(x_t,D).
\end{gathered}
\end{cases}
\]
\end{theorem}

This gives the formal contrast with uniform diffusion. In uniform diffusion, the
low-noise reverse kernel contains a dominant support-projecting component, but
also retains lower-scale non-projective components. In absorbing diffusion, the
reverse graph is sharper: after normalization, the leading-order mass is carried
\emph{entirely by support-improving unmasking moves}, while non-projective moves are
suppressed in the small-noise limit.

\paragraph{Connection with discrete scores.}

We remark that the normalizations in 
\Cref{thm:uniform-rate-separation,thm:absorbing-rate-separation} remove the
baseline forward proposal rate for a single-token edit: \(\beta_t/K\) for
uniform diffusion and \(\beta_t\) for absorbing diffusion. This is the natural
scaling in the continuous-time limit. Indeed, for uniform diffusion,
\begin{equation}
\label{eq:limit-cscore}
\frac{q^h_{t-1|t}(y^h\mid x)}{\beta_t/K}
\;\longrightarrow\;
\frac{q_t(x^{h\to y^h})}{q_t(x)} ,
\end{equation}
and the absorbing case has the analogous limit with \(\beta_t/K\) replaced by
\(\beta_t\).

The right-hand side of \eqref{eq:limit-cscore} has a standard interpretation:
it is the discrete analogue of a score, often called a \emph{concrete score}
\citep{meng2022concrete,sun2022score,lou2023discrete}. Thus,
\Cref{thm:uniform-rate-separation,thm:absorbing-rate-separation} can also be
read as small-noise expansions of the true discrete score field. In this view,
the order of magnitude of the score identifies support-improving directions,
whereas the within-order coefficients encode frequency information.  

\subsection{Implication: Support Recovery from Coarsely Trained DLMs}

\Cref{thm:uniform-rate-separation,thm:absorbing-rate-separation} show that,
after the continuous-time limit \(T\to\infty\), support information appears in
the \emph{order of magnitude} of the normalized reverse scores, whereas
frequency information appears only in the coefficient within a fixed order. We
now formalize the consequence that coarse multiplicative accuracy is sufficient
to recover support-improving directions. For uniform diffusion, define the
limiting normalized scores\footnote{
A finite-\(T\) analogue of \cref{cor:coarse-score-projection} below follows from the more general estimates in
\Cref{thm:main-generalized}; we state the limiting
case for simplicity.
}
\[
\bar S_q^h(y^h\mid x;\sigma)
\coloneqq
\lim_{T\to\infty}
\frac{q^h_{t_\sigma(T)-1\mid t_\sigma(T)}(y^h\mid x)}
{\beta_{t_\sigma(T)}/K},
\qquad
\bar S_p^h(y^h\mid x;\sigma)
\coloneqq
\lim_{T\to\infty}
\frac{p^h_{t_\sigma(T)-1\mid t_\sigma(T)}(y^h\mid x)}
{\beta_{t_\sigma(T)}/K},
\]
whenever the latter limit exists. We measure the local multiplicative
discrepancy between \(\bar S_p\) and \(\bar S_q\) by the exponentiated \emph{Thompson distance}
\citep{thompson1963certain}, or equivalently by the distortion factor
\[
\bar C_\sigma(x)
\coloneqq
\max_{h\in[H]}\max_{y^h\in[K]\setminus\{x^h\}}
\max\left\{
\frac{\bar S_p^h(y^h\mid x;\sigma)}{\bar S_q^h(y^h\mid x;\sigma)},
\frac{\bar S_q^h(y^h\mid x;\sigma)}{\bar S_p^h(y^h\mid x;\sigma)}
\right\}.
\]
This is the natural notion of accuracy for our purpose: the small-noise
expansion separates edits by powers of \(\sigma\), so preserving the correct
order of magnitude is enough to identify support-improving directions, even
when the leading coefficients are not yet accurately calibrated. The following
result makes this implication precise.

\begin{corollary}[Coarse score accuracy recovers support-improving edits]
\label{cor:coarse-score-projection}
Fix \(x\in\Xc\) and consider uniform diffusion in the low-noise regime of
\Cref{thm:uniform-rate-separation}. If
\[
\bar C_\sigma(x)=o(\sigma^{-1/2})
\qquad \text{as } \sigma\to0^+,
\]
then, for all sufficiently small \(\sigma\),
\(
\bar S_p^h(y^h\mid x;\sigma)>\sigma^{-1/2}
\Longleftrightarrow
d(x^{h\to y^h},D)<d(x,D).
\)
Thus thresholding the limiting normalized learned scores recovers exactly the
support-improving directions.
\end{corollary}

\begin{proof}
By \Cref{thm:uniform-rate-separation}, after taking \(T\to\infty\), there exist
constants \(a_x,A_x>0\) such that, for all sufficiently small \(\sigma\),
\[
\bar S_q^h(y^h\mid x;\sigma)\ge a_x\sigma^{-1}
\quad\text{if }d(x^{h\to y^h},D)<d(x,D),
\qquad
\bar S_q^h(y^h\mid x;\sigma)\le A_x
\quad\text{otherwise}.
\]
Therefore the multiplicative error bound gives
\(
\bar S_p^h(y^h\mid x;\sigma)
\ge \bar C_\sigma(x)^{-1}a_x\sigma^{-1}
=\omega(\sigma^{-1/2})
\)
for support-improving edits, and
\(
\bar S_p^h(y^h\mid x;\sigma)
\le \bar C_\sigma(x)A_x
=o(\sigma^{-1/2})
\)
for all non-support-improving edits. Hence the threshold
\(\sigma^{-1/2}\) separates the two classes.
\end{proof}

The condition \(\bar C_\sigma(x)=o(\sigma^{-1/2})\) is deliberately coarse.
By \Cref{thm:uniform-rate-separation}, after taking the continuous-time limit,
support-improving edits have true normalized scores of order \(\sigma^{-1}\),
whereas non-improving edits have scores of order at most \(1\). Hence such a
multiplicative distortion, though allowed to diverge as \(\sigma\to0^+\), cannot
close the scale gap: support-improving edits remain above the threshold
\(\sigma^{-1/2}\), while non-improving edits remain below it. This is much weaker than coefficient matching, which would require
\(
\frac{\bar S_p^h(y^h\mid x;\sigma)}
{\bar S_q^h(y^h\mid x;\sigma)}
\to 1
\)
on the relevant directions. The reason is that the \(\sigma\)-scale in the
low-noise expansion depends only on how the edit changes the distance to
\(D=\supp(\pData)\), whereas the coefficient depends on the frequencies assigned
by \(\pData\) within \(D\). Thus \emph{two data distributions with the same support but
very different probabilities on that support induce the same scale hierarchy,
but may have very different coefficients.} Support recovery therefore only
requires preserving the correct \emph{order of magnitude}, while frequency
calibration requires coefficient-level accuracy. As a simple illustration, \cref{sec:example-large-TV} gives a concrete example
showing that this type of coarse multiplicative control can preserve scale
information even when the corresponding probability vectors remain far apart in
total variation and KL divergence.

\section{Experiments}
The purpose of this section is to test the two predictions suggested by the
small-noise theory: that support localization emerges before frequency ranking,
and that uniform and absorbing diffusion exhibit distinct projector-like
behaviors.
\label{sec:experiments}
\subsection{Test of prediction 1: support precedes frequency}
\label{sec:xp-supp-then-freq}

The first takeaway from \Cref{sec:expansion} is a support-then-frequency separation: the reverse transitions expose information related to the support before distributional information. We empirically test the effect of this observation in two complementary settings. First, we design a synthetic \emph{regular language} for which the support and the within-support frequencies are known, and directly monitor the corresponding metrics across training. Second, on FineWeb, the true linguistic support is not available in closed
form, so we construct \emph{indirect} held-out probes that mirror the exact
metrics.

\paragraph{Synthetic direct probes.}
We first construct a regular language that serves as a controlled analogue of a
bigram language model: the next token depends only on the previous token, but
with token-dependent transition probabilities. We sample the first token
uniformly from \([K]\). Given \(x_h\), the next token \(x_{h+1}\) is sampled
from the local neighborhood
\[
    \{x_h-1,x_h,x_h+1\},
\]
with probabilities
\(p_{\mathrm{down}}(x_h)\), \(p_{\mathrm{stay}}(x_h)\), and
\(p_{\mathrm{up}}(x_h)\), respectively. The upward and downward probabilities
are specified by a deterministic oscillating transform of \(x_h\), and the walk
is clamped at the boundary tokens \(1\) and \(K\); the exact construction is
given in \Cref{app:def-lang}. Thus the data support is the set of strings whose
successive tokens obey the allowed local-transition rule. This gives exact
access to both support membership and within-support transition probabilities,
allowing us to define direct probes for support recovery and frequency learning.

For the fixed-context transformer trained in the standard D3PM pipeline, we
probe the learned reverse conditionals at low noise \(\sigma = 1/H\). For
absorbing diffusion, we draw a validation string and replace one token by the
mask token. For uniform diffusion, we instead replace that token by \(10\)
tokens sampled uniformly at random from \([K]\), and average the resulting probe
metrics over these corruptions. In both cases, we score candidate clean-token
replacements at the corrupted position. The support probe asks whether valid
replacements receive larger conditional probability than invalid replacements.
The frequency probes restrict attention to valid replacements and test whether
the model recovers the most likely replacement under the true transition law
(top-\(1\) accuracy) and whether it correctly ranks pairs of valid replacements.
Each metric is averaged over \(1{,}000\) validation samples per checkpoint.

The results are shown in \Cref{fig:synthetic-support-frequency}. Both absorbing
and uniform diffusion exhibit the predicted ordering: \emph{support-recovery
metrics improve before frequency-oriented metrics}. These direct probes are the
synthetic counterparts of the indirect FineWeb probes, although the metrics are
not identical because the true support is available here but not for natural
language. Additional experimental details and ablations are given in
\Cref{app:synthetic}, see in particular \Cref{fig:synthetic-support-frequency-app-1,fig:synthetic-support-frequency-app-2,fig:synthetic-support-frequency-app-3}.

\begin{figure}[t]
  \centering
  \begin{subfigure}[t]{0.4\linewidth}
    \centering
    \includegraphics[width=\linewidth]{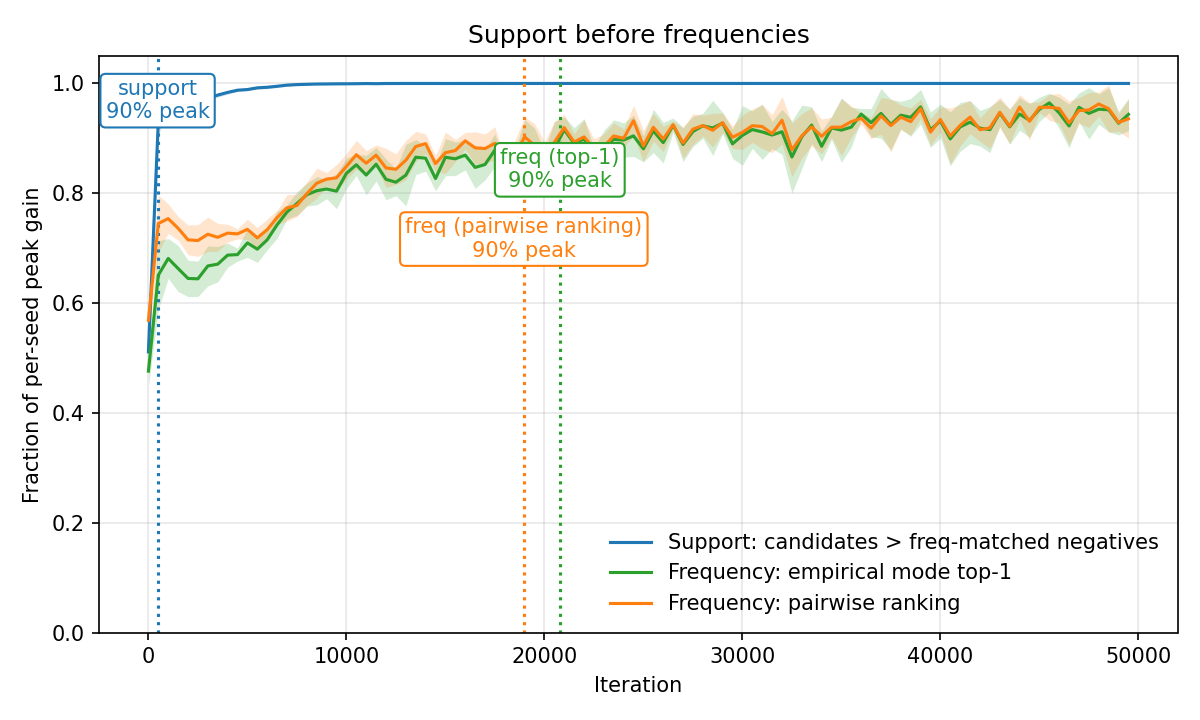}
    \caption{Absorbing diffusion}
  \end{subfigure}\qquad
  \begin{subfigure}[t]{0.4\linewidth}
    \centering
    \includegraphics[width=\linewidth]{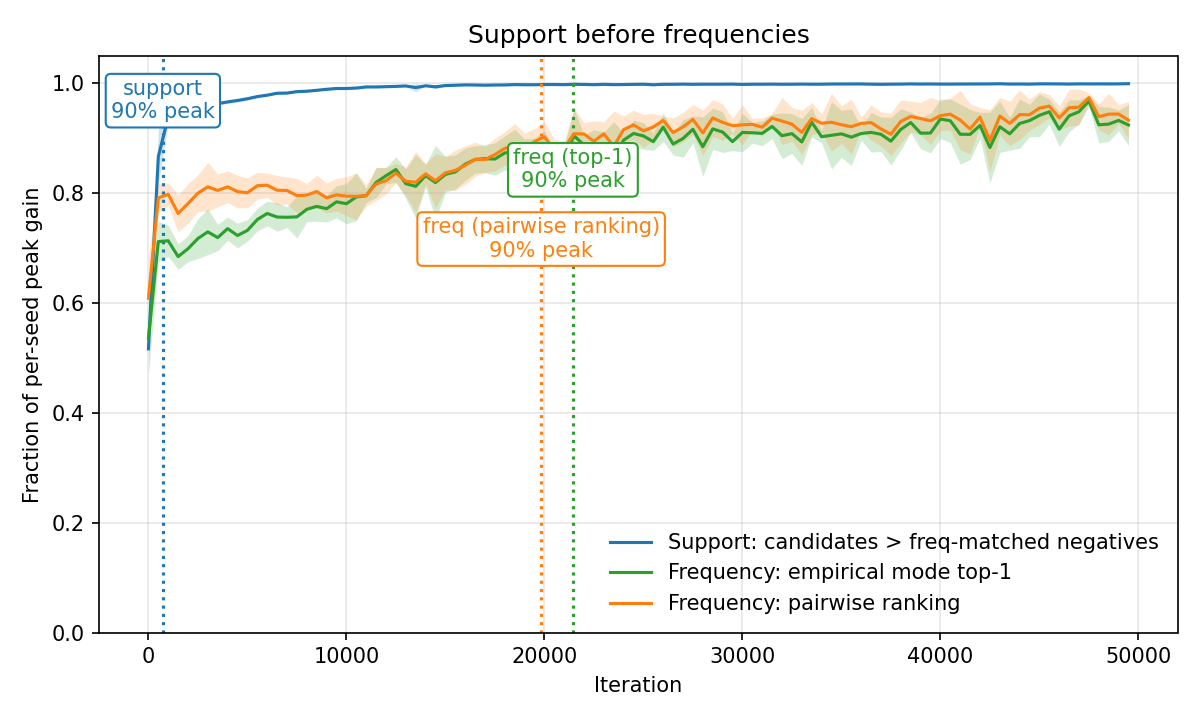}
    \caption{Uniform diffusion}
  \end{subfigure}
  \caption{\textbf{Synthetic echo of support before frequency.}
  In the regular-language setting of \Cref{sec:xp-abs-vs-unif}, direct
  support metrics improve before frequency ones, paralleling the
  FineWeb trend in \Cref{fig:real-dlm-support-frequency}.}
  \label{fig:synthetic-support-frequency}
\end{figure}

\paragraph{FineWeb indirect probes.}
We then ask whether the same separation appears in a masked DLM trained on
FineWeb \citep{penedo2024fineweb}. 
We choose masked diffusion for the real-data experiment since it is 
popular, empirically strong, and the backbone of all modern discrete diffusion language models 
at scale, e.g., \cite{nie2025large,ye2025dream,bie2025llada20,bie2026llada21,song2025seed,liu2025tidar}. 
Moreover, the absorbing case is particularly convenient since several
popular parameterizations all collapse to the \emph{same} loss and parameterization, 
so that the same trained model probes many different masked-DLM formulations and parameterizations 
at the same time, including D3PM-style posterior prediction, Bayes clean-token substitution, mean/SUBS
reverse kernels, and the normalized induced-score parameterization. Thus the
experiment probes a common object shared by many masked-DLM formulations rather
than a parameterization-specific artifact; \Cref{app:dlm-parameterizations}
reviews the equivalences and the cases where they instead differ.  
Since the true support of natural language is unavailable, we use held-out FineWeb as a proxy.
We scan validation text for repeated one-token restoration contexts: for a left
token $\ell$ and right token $r$, we record which center tokens actually
occurred in the slot $(\ell,\texttt{[MASK]},r)$.
For instance, if
held-out FineWeb contains \texttt{the end of}, then \texttt{end} is an empirical
candidate for \texttt{the [MASK] of}. A non-candidate for this slot is a token
observed elsewhere in the same context bank but not between \texttt{the} and
\texttt{of}.

The primary support curve in \Cref{fig:real-dlm-support-frequency} is the
pairwise accuracy with which empirical candidate tokens outrank
frequency-matched non-candidate negatives. The negatives are not necessarily 
ungrammatical; they simply are tokens that were observed elsewhere in the FineWeb
context bank but not in this particular slot. Matching negatives to positives by
global frequency prevents the probe from rewarding a model merely for preferring
common tokens; \Cref{fig:real-dlm-support-negative-controls} shows that other choices of negatives 
yield the same transition timing. 
We compare this support probe to two
within-candidate frequency probes: pairwise empirical ranking accuracy among
candidates, and top-1 empirical mode recovery. 
\Cref{app:synthetic-fineweb} checks that these indirect probes mirror the empirical evolution of their direct equivalent in a synthetic setup where the support is known. 

\paragraph{FineWeb result.}
Across three independent training seeds, the frequency-matched support probe
reaches $90\%$ of its peak gain at $116.3\pm5.7$M training tokens
(range $113.0$--$122.9$M). The top-1 empirical-mode and pairwise empirical
ranking probes reach the same criterion at $234.3\pm22.2$M and
$247.4\pm7.5$M tokens, respectively. Thus support localization precedes
frequency sharpening by roughly $118$--$131$M tokens on average, and the
ordering holds in each seed. The absolute token counts depend on the model,
data window, and context bank; the key comparison is within the same matched
probe. We view the separation as the real-data analogue of the rate separation
identified in \Cref{sec:expansion}: the model learns a useful support-like
projector before it learns the precise empirical distribution over that support.

\paragraph{Generalization vs. memorization.}
We also ask where this separation sits relative to the usual
generalization--memorization distinction in the FineWeb experiment. 
We evaluate masked-token
reconstruction losses on fixed train and validation windows at fixed masking
levels. A growing validation-minus-train gap would indicate train-specific
memorization; a small gap while both losses improve is consistent with a
generalization-like phase. In the measured early-token window, this diagnostic
shows a small gap while training and validation improve. Full probe definitions,
scoring details, training recipes, and supporting diagnostics are given in
\Cref{app:real-dlm-fineweb}.

\subsection{Test of prediction 2: Comparing uniform and masked projectors in synthetic examples}
\label{sec:xp-abs-vs-unif}

\begin{figure}
  \centering
  \begin{subfigure}[t]{0.4\linewidth}
    \includegraphics[width=\linewidth]{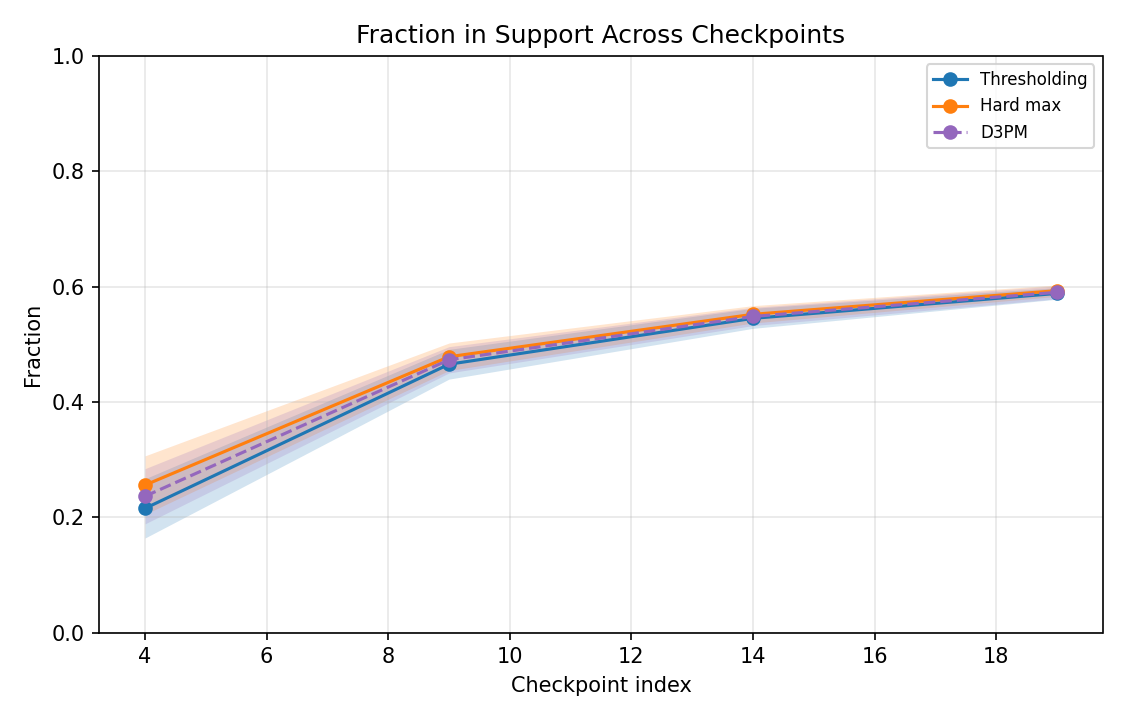}
    \caption{Absorbing diffusion}
  \end{subfigure}
  \begin{subfigure}[t]{0.4\linewidth}
    \centering
    \includegraphics[width=\linewidth]{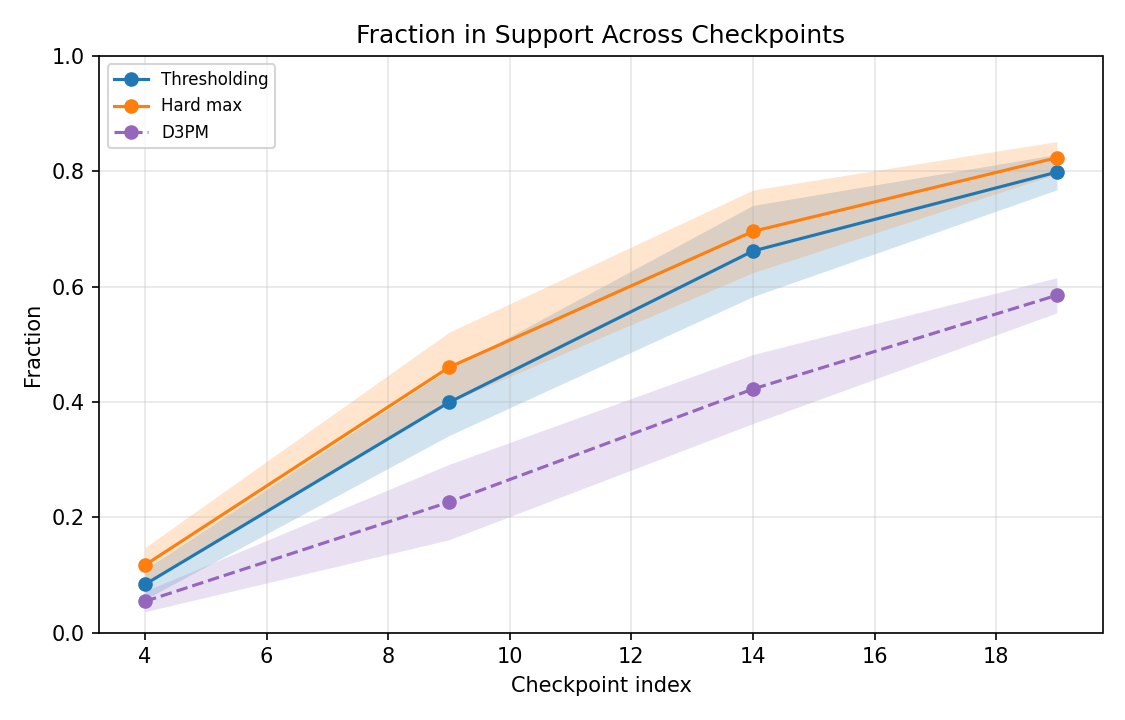}
    \caption{Uniform diffusion}
  \end{subfigure}
  \caption{\textbf{Projection-style edits help uniform diffusion more.}
  }
  \label{fig:synthetic-projector-main}
\end{figure}

The second takeaway of \Cref{sec:expansion} is that at low noise, the reverse dynamics can
act like a local projector onto data support, and that absorbing diffusion more strictly separates projective from non-projective single-token edits. We thus train uniform and absorbing D3PMs on the aforementioned synthetic distribution (\Cref{sec:xp-supp-then-freq}) and test the fraction of strings that are in the support of the synthetic data distribution. Notice that unlike the previous probes, this is a test of validity at full \emph{sequence level} rather than of local edits.

We compare the success rate of standard ancestral sampling to that of two slight modifications which, according to our theory, should leverage the information about projective directions more explicitly to enhance landing in the support. Both modified samplers, which we call ``thresholding'' and ``hard-max'' samplers respectively, first follow the standard denoising trajectory up until a small pre-defined noise level $\sigma$. Then, they repeatedly make a single-token edit among all single token edits for which the normalized reverse score (\cref{subsec:main}) is above a certain pre-defined threshold $\tau>0$ (see \cref{sec:projector-samplers} for the implementation details). Clearly, by our analysis in \cref{cor:coarse-score-projection}, this should explicitly single out directions along a projection onto the support if the threshold is well-calibrated: If the magnitudes of the reverse kernels have been learned to sufficient accuracy, theory predicts that the two samplers filter out all edits that do not strictly decrease the distance to the support. 

\Cref{fig:synthetic-projector-main} reports the fraction of generated sequences
that land in the true support under the three samplers. For absorbing diffusion,
the modified samplers perform nearly identically to ancestral sampling. For
uniform diffusion, in contrast, isolating the dominant component of the learned
reverse scores substantially improves sample validity. This behavior matches the
prediction of \cref{thm:absorbing-rate-separation}: absorbing diffusion has a
sharper low-noise separation than uniform diffusion (\cref{thm:main}), so its
permitted reverse moves are already much less affected by lower-order,
non-projective components.

This comparison should not be read as saying that absorbing diffusion is a
better global projector onto the support. The absorbing reverse process is
structurally restricted: it can only fill in currently masked tokens, and cannot
re-mask or revise tokens that have already been incorrectly decoded. Thus it may
identify the projective directions very accurately among the \emph{permitted}
denoising moves, while still lacking access to other correcting moves needed for
global support projection. This is consistent with the observation that the
modified uniform samplers can ultimately achieve a higher support-landing rate
than absorbing diffusion. Rather, the result shows that absorbing diffusion is
already close to a projector within its restricted reverse graph, whereas uniform
diffusion benefits from explicitly filtering its richer set of possible edits.

\section{Limitations and Future Directions}
\label{sec:limitations}

Our analysis focuses on the exact reverse process in the small-noise regime,
corresponding to the final denoising steps. The resulting scale separation is
not merely infinitesimal: since the gap scales as \(1/\sigma\), even
\(\sigma\approx 0.1\) gives an order-of-magnitude separation between support
and finer frequency information. Nevertheless, the extent to which trained
models recover the different terms in this expansion may depend on the loss,
architecture, parameterization, and optimization procedure. A sharper training
theory, perhaps in a stylized model or through statistical guarantees, is needed
to explain when these terms are learned in the predicted order.

Several directions remain open. First, extending the experiments to broader DLM
parameterizations and larger real-world datasets would clarify the scope of the
phenomenon. Second, support can be brittle for natural language and other
high-dimensional discrete data, where rare but valid strings may have extremely
small probability. A more robust formulation could replace the hard support
\(D=\supp(\pData)\) with an \(\epsilon\)-effective support. Finally, the
practical implications of support-before-frequency remain to be developed,
including early stopping criteria, validity-improving inference-time samplers,
and downstream applications that prioritize structural correctness over full
distributional calibration.

\section*{Acknowledgements}
This work was supported in part by the Swiss 
State Secretariat for Education, Research and Innovation (SERI) under contract number MB22.00027. We thank EPFL's RCP cluster for computational resources, and the RCP team for their organization during periods of high demand on the cluster. 
Part of this research was performed while Adrian M\"uller was visiting the Institute for Mathematical and Statistical Innovation (IMSI), which is supported by the National Science Foundation (Grant No. DMS-2425650).


\bibliographystyle{plainnat}
\bibliography{refs}

\newpage
\appendix

\part*{Appendix}
\addcontentsline{toc}{part}{Appendix} 
\etocsettocdepth{subsection} 
\localtableofcontents

\clearpage 

\section{Discrete-diffusion parameterizations for masked language modeling}
\label{app:dlm-parameterizations}

\providecommand{\nn}{\mathrm{nn}}

This appendix summarizes the modeling choices behind the masked diffusion
language model used in the FineWeb experiment in \Cref{sec:xp-supp-then-freq}. The goal is not to survey
all diffusion language models, but to make clear which common
parameterizations are equivalent in the absorbing-mask case, which ones are
not, and which family our experiments instantiate. 
An illustration is given in \Cref{fig:dt-ct-mask-score}.

The clean data sequence is denoted by \(x_0\in\mathcal V^H\), where
\(\mathcal V=[K]\) is the vocabulary. In absorbing diffusion the state space
is enlarged by one mask token \(m\notin\mathcal V\). A ``clean token'' simply
means a token in \(\mathcal V\), as opposed to the mask token \(m\).

The neural network outputs raw numbers
\(\big(\nn_{\theta}(x_t,t)\big)_{h,v}\) for each position \(h\in[H]\) and
candidate clean token \(v\in\mathcal V\). When a parameterization needs a
probability distribution on clean tokens, we write
\[
    \widehat{\nn}_{\theta,h}(v\mid x_t,t)
    :=
    \frac{
        \exp\big(\nn_{\theta}(x_t,t)\big)_{h,v}
    }{
        \sum_{u\in\mathcal V}
        \exp\big(\nn_{\theta}(x_t,t)\big)_{h,u}
    }.
\]
The learned object is the raw network output \(\nn_\theta\); the hat only means
that we have normalized it for a clean-token parameterization. In particular,
\(\widehat{\nn}_{\theta,h}\) is a distribution on \(\mathcal V\), not on
\(\mathcal V\cup\{m\}\). If one starts from logits over the enlarged alphabet,
this corresponds to the ``zero mask probability'' substitution used in SUBS
\citep{sahoo2024simple}.

We focus on this absorbing-mask SUBS/mean family for a conservative reason. 
Absorbing or masked diffusion has become the dominant route for diffusion
language modeling. It connects D3PMs to masked language modeling
\citep{austin2021structured}, admits simplified objectives and parameterizations
\citep{shi2024simplified,sahoo2024simple,ou2024your}, and has led to increasingly
competitive text models and samplers \citep{lou2023discrete,sahoo2024simple,
arriola2025block,wu2025fast,israel2025accelerating,nie2025large}. 
This absorbing
case is also convenient to consider as it sits at the intersection of several otherwise different views:
D3PM-style posterior parameterization, Bayes clean-token substitution,
mean/SUBS reverse kernels, and the normalized induced-score parameterization
all collapse to the same clean-token predictor. Thus the FineWeb experiment is
not tied to an idiosyncratic parameterization; in the mask case it probes the
common object shared by these formulations. Away from absorbing masks,
especially for uniform corruption or direct positive score models such as
SEDD, these parameterizations separate again.

\begin{figure}[H]
\centering
\resizebox{\linewidth}{!}{
\begin{tikzpicture}[
    >=Latex,
    font=\scriptsize,
    box/.style={draw, rounded corners=2mm, align=center, inner sep=5pt, text width=4.8cm},
    widebox/.style={draw, rounded corners=2mm, align=center, inner sep=5pt, text width=5.7cm},
    emphbox/.style={draw, rounded corners=2mm, align=center, inner sep=6pt, fill=gray!10, text width=5.8cm},
    arr/.style={->, thick}
]
\node[box] (dt1) at (0,0)
{\textbf{Discrete time}\\generation uses one-step kernels\\\(q(x_{t-1}\mid x_t)\)};
\node[box] (dt2) at (0,-2.2)
{model the backward step by\\\(q_\theta(x_{t-1}\mid x_t)\)};
\node[box] (dt3) at (0,-4.4)
{ELBO compares\\\(q(x_{t-1}\mid x_t,x_0)\) and \(q_\theta(x_{t-1}\mid x_t)\)};

\node[widebox] (ct1) at (7.8,0)
{\textbf{Continuous time}\\generation uses reverse rates\\\(\widetilde R_t(y\to y')=R_t(y'\to y)s_t(y\to y')\)};
\node[widebox] (ct2) at (7.8,-2.2)
{model the reverse rate through\\a score parameterization \(s_\theta\)};
\node[widebox] (ct3) at (7.8,-4.4)
{score objective compares\\the true score \(s_t\) and model score \(s_\theta\)};

\node[widebox] (induced) at (3.7,-7.0)
{\textbf{absorbing + clean-token output}\\Bayes / D3PM / mean-SUBS coincide\\and induce \(s_\theta^{\mathrm{ind}}\)};
\node[emphbox] (ce) at (3.7,-9.45)
{\textbf{weighted masked clean-token CE}\\\(\displaystyle
\sum_t w_t\,\E[-\log \widehat{\nn}_{\theta,h}(x_0^h\mid x_t,t)\indicator{x_t^h=m}]\)};

\node[widebox] (direct) at (10.3,-7.0)
{\textbf{direct positive score}\\parameterize \(s_\theta\) itself\\as in SEDD};
\node[emphbox] (scoreloss) at (10.3,-9.45)
{\textbf{score-entropy loss}\\for a general positive score field};

\draw[arr] (dt1) -- (dt2);
\draw[arr] (dt2) -- (dt3);
\draw[arr] (dt3) -- (induced);
\draw[arr] (ct1) -- (ct2);
\draw[arr] (ct2) -- (ct3);
\draw[arr] (ct3) -- (induced);
\draw[arr] (ct3) -- (direct);
\draw[arr] (induced) -- (ce);
\draw[arr] (direct) -- (scoreloss);
\end{tikzpicture}}
\caption{Two routes for training absorbing-mask diffusion models. In the
clean-token family used by D3PM-style and mean/SUBS parameterizations, the
reverse kernel and the induced score are determined by the same normalized
network output and lead to the same weighted masked-token cross-entropy. Direct
score parameterizations such as SEDD remain a separate family. 
We use the weighted masked clean-token CE in our experiments, in the mask case, 
which is equivalent to the popular SUBS/mean and D3PM parameterizations, 
and which also coincides with a score objective for the induced score field.} 
\label{fig:dt-ct-mask-score}
\end{figure}
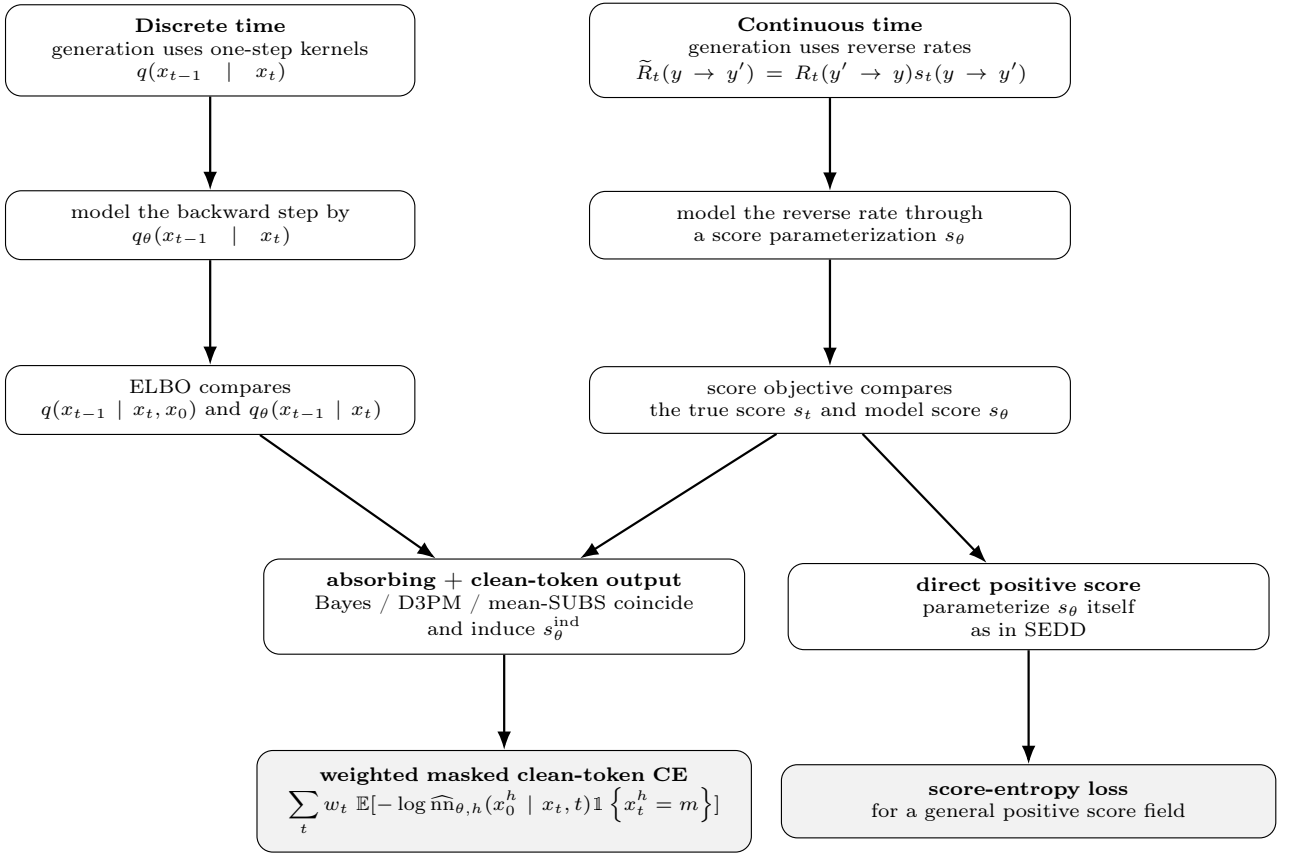

\subsection{The reverse step and where the neural network enters}

Generation proceeds by starting from a fully noised sequence and applying
reverse steps one at a time:
\[
x_T \to x_{T-1} \to \cdots \to x_0.
\]
The ideal reverse step is the true conditional distribution \(q_{t-1\mid t}\).
This object depends on the data distribution and is unknown. Standard
parameterized reverse kernels for DLMs are token-factorized:
\[
q(x_{t-1} \mid x_t) = \prod_{h\in[H]} q(x_{t-1}^h \mid x_t)
\]
so it is enough
to describe the token-level reverse steps $q(x_{t-1}^h \mid x_t)$. 
For one position
\(h\), a useful Bayes decomposition is
\begin{align}
q(x_{t-1}^h\mid x_t)
=
\sum_{v\in\mathcal V}
q(x_{t-1}^h\mid x_t,x^h_0=v)\,
q(x_0^h=v\mid x_t).
\label{eq:appA-exact-token-reverse}
\end{align}
The first factor of each summand is easy to compute in practice: once the original token \(v\)
and the current noised token \(x_t\) are fixed, the probability of the previous
noised token \(x_{t-1}^h\) is determined by the known corruption kernel. For
instance, in absorbing diffusion, if \(x_t^h\) is not masked, then
\(x_{t-1}^h\) must be the same clean token; if \(x_t^h\) is masked, then
\(x_{t-1}^h\) is either the same mask or the original clean token, with
probabilities determined by the noise schedule.

The second factor is, just like the original quantity \(q(x_{t-1}^h\mid x_t)\),
unknown. It is the posterior belief about the original token at position
\(h\), given the whole corrupted sequence \(x_t\). Several parameterizations in
the literature can be read as replacing this unknown posterior 
$q(x_0^h=v\mid x_t)$ 
by a neural
network output. Once plugged in the above formula, we deduce a surrogate \(q_{\theta}(x_{t-1}^h\mid x_t)\) to use
for sampling in place of the true reverse step \(q(x_{t-1}^h\mid x_t)\).

\paragraph{Bayes substitution as a reference formula.}
The most direct substitution in \eqref{eq:appA-exact-token-reverse} would be
\begin{align}
q^{\mathrm{Bayes}}_{\theta}(x_{t-1}^h\mid x_t)
=
\sum_{v\in\mathcal V}
q(x_{t-1}^h\mid x_t,x_0^h=v)\,
\widehat{\nn}_{\theta,h}(v\mid x_t,t).
\label{eq:appA-bayes-substitution}
\end{align}
We review below the usual constructions in the literature and how they relate to this direct Bayes substitution. 
The main takeaway is that for \emph{absorbing} diffusion, they all coincide with this Bayes substitution. 

\paragraph{D3PM posterior parameterization.}
D3PMs \citep{austin2021structured} use a closely related but not identical
posterior parameterization. Note that in \eqref{eq:appA-exact-token-reverse},
the known factor \(q(x_{t-1}^h\mid x_t,x_0^h)\) can be written as
\[
q(x_{t-1}^h\mid x_t,x_0^h)
=
\frac{q(x_{t-1}^h,x_t^h\mid x_0^h)}{q(x_t^h\mid x_0^h)}.
\]
Thus, the Bayes decomposition can be equivalently written as
\[
q(x_{t-1}^h \mid x_t)
=
\sum_{v\in\mathcal V}
q(x_{t-1}^h,x_t^h\mid x_0^h=v)
\frac{q(x_0^h=v\mid x_t)}{q(x_t^h\mid x_0^h=v)}.
\]
The D3PM parameterization replaces the fraction
\(q(x_0^h=v\mid x_t)/q(x_t^h\mid x_0^h=v)\) by the network output and then
renormalizes over \(x_{t-1}^h\):
\begin{align}
q^{\mathrm{D3PM}}_{\theta}(x_{t-1}^h\mid x_t)
\propto
\sum_{v\in\mathcal V}
q(x_{t-1}^h,x_t^h\mid x_0^h=v)\,
\widehat{\nn}_{\theta,h}(v\mid x_t,t).
\label{eq:appA-d3pm}
\end{align}
This differs from the direct Bayes substitution by the extra factor
\(q(x_t^h\mid x_0^h)\) inside the mixture weights. For general corruption
kernels that factor depends on \(x_0^h\), so the two substitutions are not the
same. For absorbing diffusion, when \(x_t^h=m\), the known factor
\(q(x_t^h=m\mid x_0^h=v)\) is independent of \(v\), so the extra factor cancels
in the normalization. When \(x_t^h\neq m\), the absorbing process forces
\(x_0^h=x_t^h\) and the reverse step copies the clean token. Hence, under the
usual copy and normalization constraints, Bayes and D3PM parameterizations are
the same for absorbing masks.

\paragraph{Mean/SUBS parameterization for absorbing masks.}
A popular parameterization is the mean/SUBS parameterization used by
\citet{shi2024simplified}, \citet{sahoo2024simple}, and
\citet{ou2024your}. This parameterization is designed for absorbing masks and
is a special case of the Bayes substitution for that corruption kernel. In
particular, it coincides with the Bayes, and therefore D3PM, parameterization
for absorbing masks. Starting again from the exact Bayes formula in
\eqref{eq:appA-exact-token-reverse}, the mean/SUBS parameterization first
replaces the known factor \(q(x_{t-1}^h\mid x_t,x_0^h)\) by its exact
closed-form expression in the absorbing case:
\begin{align}
q(x_{t-1}^h\mid x_t,x_0^h)
=
\begin{cases}
\delta_{x_t^h}(x_{t-1}^h), & x_t^h\neq m,\\[0.8ex]
\displaystyle
\frac{\sigma_{t-1}}{\sigma_t}\,\delta_m(x_{t-1}^h)
+
\frac{\sigma_t-\sigma_{t-1}}{\sigma_t}\,\delta_{x_0^h}(x_{t-1}^h),
& x_t^h=m.
\end{cases}
\label{eq:appA-absorbing-posterior}
\end{align}
The interpretation is literal. If the token at time \(t\) is not masked, it
must be carried over to time \(t-1\), since the forward process never
substitutes one clean token for another. If the current token is masked, then
one step earlier it was either already masked, or it was still equal to the
original clean token \(x_0^h\).

Plugging this absorbing-mask posterior into the Bayes substitution gives a
copy term for the unmasked case, and a mixture of the mask and clean-token
posterior for the masked case:
\begin{align}
q(x_{t-1}^h\mid x_t)
=
\begin{cases}
\delta_{x_t^h}(x_{t-1}^h), & x_t^h\neq m,\\[0.8ex]
\displaystyle
\frac{\sigma_{t-1}}{\sigma_t}\,\delta_m(x_{t-1}^h)
+
\frac{\sigma_t-\sigma_{t-1}}{\sigma_t}\sum_{v\in\mathcal V}
q(x_0^h=v\mid x_t)\,\delta_v(x_{t-1}^h),
& x_t^h=m.
\end{cases}
\label{eq:appA-absorbing-bayes-substitution}
\end{align}
The mean/SUBS parameterization replaces the remaining unknown clean-token
posterior by the normalized network output:
\begin{align}
q^{\mathrm{SUBS}}_{\theta}(x_{t-1}^h\mid x_t)
=
\begin{cases}
\delta_{x_t^h}(x_{t-1}^h), & x_t^h\neq m,\\[0.8ex]
\displaystyle
\frac{\sigma_{t-1}}{\sigma_t}\,\delta_m(x_{t-1}^h)
+
\frac{\sigma_t-\sigma_{t-1}}{\sigma_t}
\sum_{v\in\mathcal V}
\widehat{\nn}_{\theta,h}(v\mid x_t,t)\,\delta_v(x_{t-1}^h),
& x_t^h=m.
\end{cases}
\label{eq:appA-subs}
\end{align}
This is the parameterization used by the model in
\Cref{sec:xp-supp-then-freq}.

\subsection{Losses for absorbing masks}
\label{appA:losses}

We assume that we approximate the reverse process, which is given by
$q_T$ and $(q_{t-1|t})_{t\in[T]}$, with a parameterized process given by
$p_T\in\Simplex{\Xc}$ and $(p_{t-1|t}(\cdot\mid\cdot))_{t\in[T]}$. 
The usual variational upper bound on the negative log-likelihood decomposes as
\[
\E_{x_0\sim \pData}
\big[-\log p_0(x_0)\big]
\le
B_T(p_T)
+
\sum_{t=1}^T L_t(p),
\]
where the terminal boundary term is
$
B_T(p_T)
:=
\E_{x_0\sim \pData}
\Big[
\Dkl\big(q_{T|0}(\cdot\mid x_0)\,\|\,p_T\big)
\Big]$,
and the denoising term at time $t$ is
\[
L_t(p)
:=
\E_{\substack{x_0\sim \pData\\ x_t\sim q_{t|0}(\cdot\mid x_0)}}
\Big[
\Dkl\big(
q_{t-1|t,0}(\cdot\mid x_t,x_0)
\,\|\,
p_{t-1|t}(\cdot\mid x_t)
\big)
\Big].
\]
Here the case $t=1$ includes the reconstruction term, since
$q_{0|1,0}(\cdot\mid x_1,x_0)=\delta_{x_0}$.
Since the terminal term either vanishes or is independent of the learned reverse kernels when
$p_T$ is fixed to the fully corrupted marginal, we focus on the normalized denoising objective
\[
L(p):=\frac1T\sum_{t=1}^T L_t(p).
\]
Most standard discrete diffusion models, starting from D3PM, train the reverse kernels through
the empirical variational bound above; in the absorbing/masked case, this bound reduces to
equivalent weighted clean-token prediction objectives used in modern masked DLMs
\citep{austin2021structured,shi2024simplified,ou2024your,sahoo2024simple}, as we now detail, and whose formula is given in \Cref{eq:appA-masked-ce}. 

If we fix a time $t$ and consider, as in the previous subsection, $p=q_\theta$ an approximation given by a neural network with parameters $\theta$, then the variational bound is expressed in terms of denoising
terms
\[
L_t(q_\theta)
=
\E_{\substack{x_0\sim \pData\\ x_t\sim q_{t\mid0}(\cdot\mid x_0)}}
\left[
\mathrm{KL}\!\left(
q_{t-1\mid t,0}(\cdot\mid x_t,x_0)
\,\middle\|\,
q_\theta(\cdot\mid x_t)
\right)
\right].
\]
The terminal term is fixed once the terminal distribution is fixed, so the
model-dependent part is the sum of these denoising terms $\sum_{t=1}^T L_t(q_\theta)$. In the absorbing
case, the simplification above makes the loss a weighted clean-token
prediction objective. The KL is against the posterior conditioned on the clean
sample \(x_0\), not directly against the marginal reverse kernel
\(q_{t-1\mid t}(\cdot\mid x_t)\). After averaging over
\(x_0\mid x_t\), however, the population minimizer over kernels
\(q_\theta(\cdot\mid x_t)\) is exactly the true marginal reverse kernel
\(q_{t-1\mid t}(\cdot\mid x_t)\).

Indeed, for an unmasked coordinate \(x_t^h\neq m\), both the true 
$q_{t-1\mid t,0}(\cdot\mid x_t,x_0)$ and the SUBS/Bayes/D3PM model \(q_\theta(\cdot\mid x_t)\)
copy \(x_t^h\) one step back, regardless of the network output. Hence the KL divergence for that coordinate is zero. 

For a masked coordinate \(x_t^h=m\), plugging
\eqref{eq:appA-absorbing-posterior} and \eqref{eq:appA-subs} into the KL gives,
up to constants independent of \(\theta\),
\[
\frac{\sigma_t-\sigma_{t-1}}{\sigma_t}
\big[-\log \widehat{\nn}_{\theta,h}(x_0^h\mid x_t,t)\big].
\]
Therefore the finite-\(T\) objective reduces to
\begin{align}
\mathcal L_T(\theta)
=
\sum_{t=1}^T
\frac{\sigma_t-\sigma_{t-1}}{\sigma_t}\,
\E\!\left[
\sum_{h=1}^H
\indicator{x_t^h=m}
\big(-\log \widehat{\nn}_{\theta,h}(x_0^h\mid x_t,t)\big)
\right].
\label{eq:appA-masked-ce}
\end{align}
This is the discrete-time version. Taking \(T\to\infty\) turns the sum into the
continuous-time weighted masked-token objective used by modern masked DLMs,
with weight \(\dot\sigma_t/\sigma_t\) when \(\sigma_t\) denotes the cumulative
mask probability. Thus D3PM-style posterior training, Bayes substitution, and
mean/SUBS training collapse to the same clean-token prediction objective in
the absorbing-mask case.

\subsection{Scores: direct parameterization versus induced score}
\label{appA:scores}

Some other parameterizations in the literature are based 
on a continuous time formulation of the discrete-space diffusion process. 
In continuous time, there is no one-step object \(q_{t-1\mid t}\) until a time
discretization is chosen. The primitive object is instead a rate matrix \(R_t\)
that defines the infinitesimal forward transition probabilities:
\[
q_{t+\delta t\mid t}(y'\mid y)
=
\begin{cases}
R_t(y\to y')\,\delta t + o(\delta t), & y'\neq y,\\
1 + R_t(y\to y)\,\delta t + o(\delta t), & y'=y.
\end{cases}
\]
If the forward process has rate matrix \(R_t\), then
the off-diagonal reverse rates have the form
\[
\widetilde R_t(y\to y')
=
R_t(y'\to y)\,
\frac{q_t(y')}{q_t(y)}
\qquad (y'\neq y).
\]
Here the arrow in \(y\to y'\) always means ``current token \(y\), proposed
replacement \(y'\)'' for the process being discussed. The forward rate in the
reverse formula points in the opposite direction, \(y'\to y\), because a reverse
move from \(y\) to \(y'\) undoes a forward move from \(y'\) to \(y\).
For a full sequence, the same convention applies to neighboring sequences that
differ by one token replacement.
The ratio
\[
s_t(y\to y'):=\frac{q_t(y')}{q_t(y)}
\]
is called the discrete or concrete score in this literature.

SEDD \citep{lou2023discrete} directly parameterizes a positive score field and uses score entropy rather than a squared loss analogous to continuous-space score matching. In our arrow convention, the denoising
score-entropy loss has the form
\begin{align}
\mathcal L_t^{\mathrm{SE}}(\theta)
&=
\E_{\substack{x_0\sim \pData\\ x_t\sim q_{t\mid0}(\cdot\mid x_0)}}
\left[
\sum_{\tilde x\neq x_t}
R_t(\tilde x\to x_t)
\left(
s_\theta(x_t\to \tilde x\mid x_t,t)
\right.\right.
\nonumber\\[-0.2ex]
&\hspace{8em}\left.\left.
-
s_t^{x_0}(x_t\to \tilde x)\log s_\theta(x_t\to \tilde x\mid x_t,t)
+
\Phi\!\left(s_t^{x_0}(x_t\to \tilde x)\right)
\right)
\right].
\label{eq:appA-score-loss}
\end{align}
Here
\[
s_t^{x_0}(x\to \tilde x):=\frac{q_t(\tilde x\mid x_0)}{q_t(x\mid x_0)}
\qquad\text{and}\qquad
\Phi(a):=a\log a-a
\]
and the \(\Phi\) term is independent of \(\theta\). As in denoising score matching, the loss uses the
conditional score \(s_t^{x_0}\) as a tractable training target, while its
population optimum is the marginal score \(s_t(x\to \tilde x)=q_t(\tilde x)/q_t(x)\).

The parameterization of \(s_\theta\) is a modeling choice. A direct score
parameterization replaces the unknown score by a general positive function of
the proposed replacement state. Another option is to keep
\(\widehat{\nn}_{\theta,h}\) as an approximation of \(q(x_0^h\mid x_t)\), and
then re-express the score in terms of this clean-token posterior. This gives an
induced score parameterization. In the absorbing case, this induced score is a
deterministic rescaling of the same clean-token posterior that appears in the
Bayes/D3PM/SUBS reverse kernels, and the score-entropy loss reduces, up to
\(\theta\)-independent terms, to the same weighted cross-entropy as
\eqref{eq:appA-masked-ce}.

Indeed, for mask diffusion, conditioning on a
clean token \(u\in\mathcal V\),
\[
q_t(v\mid x_0^h=u)
=
(1-\sigma_t)\indicator{v=u}
+
\sigma_t\indicator{v=m}.
\]
Hence, for a move from mask to clean token \(j\),
\begin{align}
s_t(m\to j\mid x_0^h=u)
=
\frac{1-\sigma_t}{\sigma_t}\indicator{j=u}.
\label{eq:appA-true-mask-score}
\end{align}
After conditioning on the whole corrupted sequence \(x_t\), this becomes
\begin{align}
s_t(m\to j\mid x_t)
=
\frac{1-\sigma_t}{\sigma_t}\,
q(x_0^h=j\mid x_t).
\label{eq:appA-score-posterior}
\end{align}
Replacing the clean-token posterior by the network output gives the induced
score
\begin{align}
s_{\theta,t}^{\mathrm{ind}}(m\to j\mid x_t)
=
\frac{1-\sigma_t}{\sigma_t}\,
\widehat{\nn}_{\theta,h}(j\mid x_t,t),
\qquad j\in\mathcal V.
\label{eq:appA-induced-score}
\end{align}
Thus, in the masked case, the relevant unmasking score is a deterministic
rescaling of the same clean-token posterior that appeared in the reverse-kernel
parameterizations. This induced score is normalized:
\[
\sum_{j\in\mathcal V}s_{\theta,t}^{\mathrm{ind}}(m\to j\mid x_t)
=
\frac{1-\sigma_t}{\sigma_t}.
\]
Plugging \eqref{eq:appA-induced-score} into the score-entropy loss
\eqref{eq:appA-score-loss} leaves, up to constants independent of \(\theta\),
\[
\frac{\dot\sigma_t}{\sigma_t}\,
\indicator{x_t^h=m}
\big[-\log\widehat{\nn}_{\theta,h}(x_0^h\mid x_t,t)\big],
\]
the continuous-time analogue of \eqref{eq:appA-masked-ce}. In contrast, a direct
score parameterization such as SEDD replaces the unknown score by a general
positive field and optimizes score entropy directly; it is not constrained to
lie on the normalized clean-token manifold above, although the optimum recovers
the same marginal score when the model class is rich enough. 
This is summarized in \Cref{fig:dt-ct-mask-score}.

\subsection{Recipe used in the real-data experiment}

The FineWeb experiment in \Cref{sec:xp-supp-then-freq} uses the standard
absorbing-mask, clean-token-prediction family summarized above: a DDiT-style
denoiser, continuous-time log-linear mask noise, no explicit time conditioning,
and the mean/SUBS masked-token loss. Sampling is then done with a finite
discretization of the same absorbing reverse kernel.

\section{Deferred Proofs for \Cref{sec:expansion}} \label{app:proofs-expansion}

\subsection{Coarse multiplicative accuracy does not imply distributional accuracy}
\label{sec:example-large-TV}

The next example illustrates that the multiplicative control used in
\Cref{cor:coarse-score-projection} is genuinely weaker than distributional
closeness. In particular, a learned kernel may preserve the correct
\(\sigma\)-scale hierarchy while remaining far from the true kernel in additive
metrics such as total variation or KL divergence.

\begin{example}[Scale accuracy need not imply distributional accuracy]
\label{ex:multiplicative-not-tv}
Fix any \(\alpha\in(0,1/2)\). For \(\sigma\in(0,1/4)\), define probability
vectors on two points by
\[
q_\sigma
=
\left(
\frac12-\sigma,\,
\frac12+\sigma
\right),
\qquad
p_\sigma
=
\left(
\sigma^\alpha\left(\frac12-\sigma\right),\,
1-\sigma^\alpha\left(\frac12-\sigma\right)
\right).
\]
Both vectors have full support for every \(\sigma>0\). Moreover, their
symmetric multiplicative distortion satisfies
\[
\max_i
\max\left\{
\frac{p_{\sigma,i}}{q_{\sigma,i}},
\frac{q_{\sigma,i}}{p_{\sigma,i}}
\right\}
=
\Theta(\sigma^{-\alpha})
=
o(\sigma^{-1/2}),
\]
since \(\alpha<1/2\). Thus \(p_\sigma\) and \(q_\sigma\) obey the same type of
coarse multiplicative control required in \Cref{cor:coarse-score-projection}.

However, this control does not imply additive closeness. Indeed,
\[
\|p_\sigma-q_\sigma\|_1
=
1+o(1),
\qquad
\mathrm{TV}(p_\sigma,q_\sigma)
=
\frac12+o(1).
\]
Consequently, by Pinsker's inequality,
\[
\mathrm{KL}(p_\sigma\|q_\sigma)
\ge
2\,\mathrm{TV}(p_\sigma,q_\sigma)^2
=
\frac12+o(1),
\]
and the same lower bound also holds for
\(\mathrm{KL}(q_\sigma\|p_\sigma)\). Hence coarse multiplicative scale accuracy
can be sufficient for recovering the order-of-magnitude structure of the scores,
while still being far too weak to guarantee distributional calibration.
\end{example}

\subsection{Proof of \Cref{thm:main}} \label{app:uniform}

We first prove the following auxiliary lemma.

\begin{lemma} \label{lemma:expansion-simple}
    For the uniform diffusion process $q$, for $t\in[T]$ and $x \in \Xc$, we have 
    \begin{align*}
        q_t(x) =&  \pData(\proj_D(x)) \br{\frac{\sigma_t}{K}}^{d(x,D)} + O((\sigma_t/K)^{d(x,D)+1})
    \end{align*}
    as $\sigma_t \to 0$. 
\end{lemma}

\begin{proof}
    For any $x$, we have 
    \begin{align*}
        q_t(x) =& \sum_{x_0 \in D} \Pm[X_0 = x_0] \cdot \Pm[X_t = x | X_0 = x_0],
    \end{align*}
    where $\Pm[X_0 = x_0]=\pData(x_0)$ and, by definition of the forward process,
    \begin{align*}
        \Pm[X_t = x | X_0 = x_0] =& \prod_{i=1}^H \br{(1-\sigma_t) \indicator{x^i=x_0^i} + \frac{\sigma_t}{K}}\\
        =& \prod_{i=1}^H \br{\br{1-\sigma_t \frac{K-1}{K}} \indicator{x^i=x_0^i} + \frac{\sigma_t}{K}\indicator{x^i\neq x_0^i} }\\
        =& \br{1-\sigma_t \frac{K-1}{K}}^{H-d(x,x_0)} \br{\frac{\sigma_t}{K}}^{d(x,x_0)}.
    \end{align*}
    Hence rearranging the sum by the order $d(x,x_0)$ of $\sigma_t/K$,
    \begin{align*}
        q_t(x) =& \sum_{x_0\in D} \pData(x_0) \br{1-\sigma_t \frac{K-1}{K}}^{H-d(x,x_0)} \br{\frac{\sigma_t}{K}}^{d(x,x_0)}\\
        =& \sum_{d=0}^H \br{\sum_{\substack{x_0 \in D\colon\\d(x,x_0)=d}} \pData(x_0)} \br{1-\sigma_t \frac{K-1}{K}}^{H-d} \br{\frac{\sigma_t}{K}}^{d}\\
        =& \sum_{d=d(x,D)}^H \br{\sum_{\substack{x_0 \in D\colon\\d(x,x_0)=d}} \pData(x_0)} \br{1-\sigma_t \frac{K-1}{K}}^{H-d} \br{\frac{\sigma_t}{K}}^{d},
    \end{align*}
    where we used that $d(x,x_0)\geq d(x,D)$ for all $x_0\in D$ in the last line. Now when $\sigma_t \to 0$, we find that the lowest order term is
    \begin{align*}
        \bigg(\sum_{\substack{x_0 \in D\colon\\d(x,x_0)=d(x,D)}} \pData(x_0)\bigg) \br{\frac{\sigma_t}{K}}^{d(x,D)}
    \end{align*}
    and all other terms are of order at least $d(x,D)+1$ and thus contribute at most $O((\sigma_t/K)^{d(x,D)+1})$. 
\end{proof}

We are now ready to proceed with the proof of \cref{thm:main}.

\begin{proof}(\cref{thm:main}) 
    For ease of notation, we set $x:=x_t$ and $z:=x_{t-1}=(x_t^{-h},y^h)$. We have 
    \begin{align*}
        q_{t-1|t}^h(y^h | x) = \sum_{y^{-h}\in[K]^{H-1}} q_{t-1|t}(y^h|x) = \sum_{y^{-h}\in[K]^{H-1}} \frac{q_{t-1}(y)}{q_{t}(x)} q_{t|t-1}(x|y)
    \end{align*}
    by the law of total probability and Bayes' law. For the conditional term on the RHS we know that (since $x^h_t\neq y^h$)
    \begin{align*}
        q_{t|t-1}(x|y) = \br{1-\beta_t \frac{K-1}{K}}^{H-1-d(x^{-h},y^{-h})}\br{\frac{\beta_t}{K}}^{d(x^{-h},y^{-h})} \cdot \frac{\beta_t}{K}. 
    \end{align*}
    Plugging this into the previous equation and rearranging by the order in terms of $\beta_t$ reveals
    \begin{align}
        q_{t-1|t}^h(y^h | x) =& \sum_{y^{-h}\in[K]^{H-1}} \frac{q_{t-1}(y)}{q_{t}(x)} \cdot \br{1-\beta_t \frac{K-1}{K}}^{H-1-d(x^{-h},y^{-h})}\br{\frac{\beta_t}{K}}^{d(x^{-h},y^{-h})+1} \nonumber\\
        =& \sum_{d=0}^{H-1} \underbrace{\sum_{\substack{y^{-h}\in[K]^{H-1}\\d(x^{-h},y^{-h})=d}} \frac{q_{t-1}(y)}{q_{t}(x)} \cdot \br{1-\beta_t \frac{K-1}{K}}^{H-1-d}\br{\frac{\beta_t}{K}}^{d+1}}_{=: C_d}. \label{eq:sum-d}
    \end{align}
    Now invoking \Cref{lemma:expansion-simple} and using that $\proj_D(x) = \{y\in D \mid d(x,y)=d(x,D)\}$, we can approximate the ratio $\frac{q_{t-1}(y)}{q_{t}(x)}$ as
    \begin{align}
        \frac{q_{t-1}(y)}{q_{t}(x)} =& \frac{\pData(\proj_D(y)) \br{\frac{\sigma_{t-1}}{K}}^{d(y,D)}\br{1+O(\frac{\sigma_{t-1}}{K})}}{\pData(\proj_D(x)) \br{\frac{\sigma_t}{K}}^{d(x,D)}\br{1+O(\frac{\sigma_{t}}{K})}}\nonumber\\
        =& \frac{\pData(\proj_D(y))}{\pData(\proj_D(x))} \br{\frac{\sigma_t}{K}}^{d(y,D)-d(x,D)}\cdot\br{\frac{\sigma_{t-1}}{\sigma_t}}^{d(y,D)} \br{1+O\br{\frac{\sigma_t}{K}}}. \label{eq:score}
    \end{align}
    We now consider the different terms $C_d$ in the sum in \Cref{eq:sum-d} by plugging in the ratio from \Cref{eq:score}: 
    
    \textbf{Lowest order: $\mathbf{d=0}$.} 
    We have $d(x^{-h},y^{-h})=0$, so $y=(x^{-h},y^h)$. Hence, the only term here is 
    \begin{align*}
        C_0 = \frac{\pData(\proj_D(y))}{\pData(\proj_D(x))} \br{\frac{\sigma_t}{K}}^{d(y,D)-d(x,D)}\cdot\br{\frac{\sigma_{t-1}}{\sigma_t}}^{d(y,D)} \br{1+O\br{\frac{\sigma_t}{K}}} \br{1-\beta_t \frac{K-1}{K}}^{H-1}\br{\frac{\beta_t}{K}}.
    \end{align*}
    By \Cref{assumption:schedule}, we have $\frac{\sigma_{t-1}}{\sigma_t}\to 1$ and $1-\beta_t \frac{K-1}{K} \to 1$, hence 
    \begin{align*}
        \frac{C_d}{\beta_t / K } \to \frac{\pData(\proj_D(y))}{\pData(\proj_D(x))} \br{\frac{\sigma}{K}}^{d(y,D)-d(x,D)} +O\br{\br{\frac{\sigma}{K}}^{d(y,D)-d(x,D)+1}}
    \end{align*}
    as $T\to\infty$. 

    \textbf{Higher orders: $\mathbf{d>0}$.} In this case, we have 
    \begin{align*}
        \frac{C_d}{\beta_t / K} = \sum_{\substack{y^{-h}\in[K]^{H-1}\\d(x^{-h},y^{-h})=d}} \frac{q_{t-1}(y)}{q_{t}(x)} \cdot \br{1-\beta_t \frac{K-1}{K}}^{H-1-d}\br{\frac{\beta_t}{K}}^{d} \to 0
    \end{align*}
    since $(\beta_t/K)^d \to 0$ for $d\ge 1$, while $1-\beta_t\frac{K-1}{K} \to 1$ and $\frac{q_{t-1}(y)}{q_{t}(x)}$ allows for the analogous upper bound independent of $T$ as in the $d=0$ case by using \Cref{eq:score}.

    Hence, we conclude that 
    \begin{align*}
        \lim_{T\to\infty} \frac{q^h_{t-1|t}(y^h | x)}{\beta_t/K} = \frac{\pData(\proj_D(y))}{\pData(\proj_D(x))} \br{\frac{\sigma}{K}}^{d(y,D)-d(x,D)} +O\br{\br{\frac{\sigma}{K}}^{d(y,D)-d(x,D)+1}},
    \end{align*}
    where $y=(x^{-h},y^h)$. A case distinction with respect to $d(y,D) \in \{d(x,D)-1, d(x,D), d(x,D)+1\}$ now shows the claim.
\end{proof}

\subsection{Proof of \Cref{thm:main-absorb}} \label{app:absorb}

We first prove the following auxiliary lemma, which transfers the expansion of the marginal at time $t$ in \Cref{lemma:expansion-simple} to absorbing diffusion.\\

\begin{lemma} \label{lemma:expansion-simple-absorb}
    For the absorbing diffusion process $q$, for $t\in[T]$ and $x \in \Xc$, we have 
    \begin{align*}
        q_t(x) =& \begin{dcases}
            \pData(\proj_D(x)) \sigma_t^{d(x,D)} + O(\sigma_t^{d(x,D)+1}) \qquad& \text{if } \exists x_0\in D \colon \forall h\in[H]\colon (x^h\neq m \Rightarrow x^h = x_0^h)\\[2mm]
            0 \qquad& \text{else}
        \end{dcases}
    \end{align*}
    as $\sigma_t \to 0$. 
\end{lemma}

Note that we are in the first of the two cases if and only if $x\in\supp(q_t)$. 

\begin{proof}
    For any $x$, we have 
    \begin{align*}
        q_t(x) =& \sum_{x_0 \in D} \Pm[X_0 = x_0] \cdot \Pm[X_t = x | X_0 = x_0],
    \end{align*}
    where $\Pm[X_0 = x_0]=\pData(x_0)$ and, by definition of the forward process,
    \begin{align*}
        \Pm[X_t = x | X_0 = x_0] =& \prod_{i=1}^H \br{(1-\sigma_t) \indicator{x^i=x_0^i} + \sigma_t\indicator{x^i=m}}\\
        =& \prod_{i=1}^H \br{(1-\sigma_t) \indicator{x^i=x_0^i} + \sigma_t\indicator{x^i=m}}\\
        =& \br{1-\sigma_t}^{H-d(x,x_0)} \sigma_t^{d(x,x_0)} \indicator{\forall i\colon x^i\in\{x_0^i,m\}}.
    \end{align*}
    Hence rearranging the sum by the order $d(x,x_0)$ of $\sigma_t$,
    \begin{align*}
        q_t(x) =& \sum_{\substack{x_0\in D \colon\\ \forall h \colon x^{h}\in\{x_0^{h},m\}}} \pData(x_0) \br{1-\sigma_t}^{H-d(x,x_0)} \sigma_t^{d(x,x_0)}\\
        =& \sum_{d=0}^H \br{\sum_{\substack{x_0 \in D\colon\\\forall h\colon x^{h}\in\{x_0^{h},m\},\\d(x,x_0)=d}} \pData(x_0)} \br{1-\sigma_t}^{H-d} \sigma_t^{d}\\
    \end{align*}
    Now when $\sigma_t \to 0$, we find that the lowest order non-vanishing term is of order 
    \begin{align*}
        d_0 :=& \min\{ d(x,x_0) \mid x_0 \in D,~ \forall h\in[H] \colon (x^h = m \vee x^h=x_0^h ) \}\\
        =& \begin{cases}
            |\{h\in[H]\colon x_h=m\}|=d(x,D) \qquad&\text{if } \exists x_0 \in D \colon \forall i \colon (x^h\neq m \Rightarrow x^h=x_0^h)\\
            +\infty \qquad &\text{else}
        \end{cases}
    \end{align*}
    since any $x_0 \in D$ has no masked tokens. Hence $q_t(x)=0$ in the second of the above cases. Otherwise, in the first case, the lowest order term is
    \begin{align*}
        \bigg(\sum_{\substack{x_0 \in D\colon\\d(x,x_0)=d(x,D)}} \pData(x_0)\bigg) \sigma_t^{d(x,D)}
    \end{align*}
    and all other terms are of order at least $d(x,D)+1$ and thus contribute at most $O(\sigma_t^{d(x,D)+1})$. 
\end{proof}

\paragraph{Proof of \Cref{thm:main-absorb}.}

\begin{proof}
    For ease of notation, we set $x:=x_t$ and $z:=x_{t-1}=(x_t^{-h},y^h)$. We have 
    \begin{align*}
        q_{t-1|t}^h(y^h | x) = \sum_{y^{-h}\in[K]^{H-1}} q_{t-1|t}(y^h|x) = \sum_{\substack{y^{-h}\in[K]^{H-1}\colon\\\forall i\colon y^{i}=x^{i}\vee x^{i}=m}} \frac{q_{t-1}(y)}{q_{t}(x)} q_{t|t-1}(x|y)
    \end{align*}
    by the law of total probability and Bayes' law. For the conditional term on the RHS we know that (since $x^h_t=m$, $y^h\neq m$ and $\forall i\neq h \colon y^{i}=x^{i}\vee x^{i}=m$)
    \begin{align*}
        q_{t|t-1}(x|y) = \br{1-\beta_t}^{H-1-d(x^{-h},y^{-h})}\beta_t^{d(x^{-h},y^{-h})} \cdot \beta_t. 
    \end{align*}
    Plugging this into the previous equation and rearranging by the order in terms of $\beta_t$ reveals
    \begin{align}
        q_{t-1|t}^h(y^h | x) =& \sum_{\substack{y^{-h}\in[K]^{H-1}\colon\\\forall i\colon y^{i}=x^{i}\vee x^{i}=m}} \frac{q_{t-1}(y)}{q_{t}(x)} \cdot \br{1-\beta_t}^{H-1-d(x^{-h},y^{-h})}\beta_t^{d(x^{-h},y^{-h})+1} \nonumber\\
        =& \sum_{d=0}^{H-1} \underbrace{\sum_{\substack{y^{-h}\in[K]^{H-1}\colon\\\forall i\colon y^{i}=x^{i}\vee x^{i}=m,\\d(x^{-h},y^{-h})=d}} \frac{q_{t-1}(y)}{q_{t}(x)} \cdot \br{1-\beta_t}^{H-1-d}\beta_t^{d+1}}_{=: C_d}. \label{eq:sum-d-absorb}
    \end{align}
    Consider each term in the sum in \Cref{eq:sum-d-absorb}. 
    Invoking \Cref{lemma:expansion-simple-absorb}, we can see that $\frac{q_{t-1}(y)}{q_{t}(x)}=0$ if there is no $x_0\in D $ such that $\forall i\colon y^i=m \vee y^i=x_0^i$ (second case in \Cref{lemma:expansion-simple-absorb}). Else (first case in \Cref{lemma:expansion-simple-absorb}), we can approximate the ratio $\frac{q_{t-1}(y)}{q_{t}(x)}$ as
    \begin{align}
        \frac{q_{t-1}(y)}{q_{t}(x)} =& \frac{\pData(\proj_D(y)) \sigma_{t-1}^{d(y,D)}\br{1+O(\sigma_{t-1})}}{\pData(\proj_D(x))\sigma_t^{d(x,D)}\br{1+O(\sigma_{t})}}\nonumber\\
        =& \frac{\pData(\proj_D(y))}{\pData(\proj_D(x))} \sigma_t^{d(y,D)-d(x,D)}\cdot\br{\frac{\sigma_{t-1}}{\sigma_t}}^{d(y,D)} \br{1+O(\sigma_t)}. \label{eq:score-absorb}
    \end{align}
    We now consider the different terms $C_d$ in the sum in \Cref{eq:sum-d-absorb} by plugging in the ratio from \Cref{eq:score-absorb}: 
    
    \textbf{Lowest order: $\mathbf{d=0}$.} 
    We have $d(x^{-h},y^{-h})=0$, so the sum in \Cref{eq:sum-d-absorb} is only over $y=(x^{-h},y^h)$. 

    \textit{Case 1:} $d(y,D)=d(x,D)-1$ \\
    Since $x\in\supp(q_t)$, this implies $y\in\supp(q_{t-1})$ (i.e. $\exists x_0\in D \colon \forall i\colon (y^i=m \vee y^i=x_0^i)$). Thus, the only term here is 
    \begin{align*}
        C_0 = \frac{\pData(\proj_D(y))}{\pData(\proj_D(x))}\sigma_t^{d(y,D)-d(x,D)}\cdot\br{\frac{\sigma_{t-1}}{\sigma_t}}^{d(y,D)} \br{1+O(\sigma_t)} \br{1-\beta_t}^{H-1}\beta_t.
    \end{align*}
    By \Cref{assumption:schedule}, we have $\frac{\sigma_{t-1}}{\sigma_t}\to 1$ and $1-\beta_t \to 1$, hence 
    \begin{align*}
        \frac{C_d}{\beta_t} \to \frac{\pData(\proj_D(y))}{\pData(\proj_D(x))} \sigma^{d(y,D)-d(x,D)} +O\br{\sigma^{d(y,D)-d(x,D)+1}} = \frac{\pData(\proj_D(y))}{\pData(\proj_D(x))} \sigma^{-1} +O\br{1}
    \end{align*}
    as $T\to\infty$. 

    \textit{Case 2:} $d(y,D) \geq d(x,D)$
    Note that since $x^h=m$ and all $x\in D$ do not contain the mask token $m$, we must in fact have $d(y,D)=d(x,D)$. Now since $x^h$ has been unmasked to $y^h$ but the distance to $D$ does not decrease, this means that there is no $x_0 \in D $ such that $\forall i \colon (y^i=m \vee y^i = x_0^i)$. Hence $q_{t-1}(y)=0$ and thus 
    \begin{align*}
        \frac{C_d}{\beta_t} = 0.
    \end{align*}

    \textbf{Higher orders: $\mathbf{d>0}$.} 
    Note that by \Cref{eq:score-absorb}
    \begin{align*}
        \frac{q_{t-1}(y)}{q_{t}(x)} \leq& \frac{\pData(\proj_D(y))}{\pData(\proj_D(x))} \sigma_t^{d(y,D)-d(x,D)}\cdot\br{\frac{\sigma_{t-1}}{\sigma_t}}^{d(y,D)} \br{1+O(\sigma_t)} \\
        \to ~& \frac{\pData(\proj_D(y))}{\pData(\proj_D(x))} \sigma^{d(y,D)-d(x,D)} +O\br{\sigma^{d(y,D)-d(x,D)+1}}
    \end{align*}
    for the fixed $\sigma$ as $T\to\infty$ like in the $d=0$ case. Hence, in this case,using that $\beta_t^d \to 0$ for $d>1$ and $1-\beta_t \to 1$, we have 
    \begin{align*}
        \frac{C_d}{\beta_t} = \sum_{\substack{y^{-h}\in[K]^{H-1}\colon\\\forall i\colon y^{i}=x^{i}\vee x^{i}=m,\\d(x^{-h},y^{-h})=d}} \frac{q_{t-1}(y)}{q_{t}(x)} \cdot \br{1-\beta_t}^{H-1-d}\beta_t^{d} \to 0
    \end{align*}

    Hence, we conclude that 
    \begin{align*}
        \lim_{T\to\infty} \frac{q^h_{t-1|t}(y^h | x)}{\beta_t} = \begin{dcases}
            \frac{\pData(\proj_D(y))}{\pData(\proj_D(x))} \sigma^{-1} +O\br{1} \qquad &\text{ if } d(y,D)=d(x,D)-1\\
            0 \qquad &\text{ else.}
        \end{dcases},
    \end{align*}
\end{proof}

\subsection{General Versions of \Cref{thm:main} and \Cref{thm:main-absorb}} \label{app:general}

We now prove generalizations of \Cref{thm:main} and \Cref{thm:main-absorb} including non-vanishing discretization errors, which is necessary to relate the error of the learned reversed transition to the separation without actually having to resort to the continuous time limit. 

For both uniform and absorbing diffusion, we make the following assumption, which is stricter that \Cref{assumption:schedule} but is still satisfied by the common noise schedules like the cosine or the ``$1/(T-t+1)$''-schedule.\\

\begin{assumption}[Noise Schedule; Discrete Time] \label{assumption:schedule-general}
    The noise schedule $(\sigma_t)_{t\in[T]}$ induced by $(\beta_t)_{t\in[T]}$ is such that for any fixed $\sigma \in(0,1)$ and time step $t=t_{\sigma}(T):= \arg\min_{s\in[T]} |\sigma_s - \sigma|$ (approximately) corresponding to this, we have
    \begin{align*}
        \sigma_t - \sigma_{t-1} = O\br{T^{-1}}, \qquad\qquad \beta_t = O(T^{-1})
    \end{align*}
    as $T\to\infty$.
\end{assumption}

\paragraph{Uniform diffusion.} The only difference to \Cref{thm:main} is that we do not argue about $\lim_{T\to\infty} \frac{q_{t-1|t}}{\beta_t/K}$ but directly about $\frac{q_{t-1|t}}{\beta_t/K}$ under the additional assumption that the order $1/T$ of the discretization error is sufficiently small ($O(\sigma^2)$) so as to not influence the orders of the noise level.\\

\begin{theorem}[Rate Separation (Uniform); Discrete Time] \label{thm:main-generalized}
    Let $q$ be the law of the uniform diffusion process. Let $x_t \in \Xc$ and $y^h \in [K]\setminus\{x_t^h\}$ and set $x_{t-1}:=(x_t^{-h},y^h)$. Consider any noise level $\sigma < 1$ and the corresponding time step $t=t_\sigma(T)=\arg\min_{s\in[T]}|\sigma_s-\sigma|$. Under \Cref{assumption:schedule-general}, for a sufficiently fine-grained discretization $\frac{1}{T} = O(\sigma^2)$, we have 
    \begin{align*}
        \frac{q_{t-1 | t}^h(y^h | x_t)}{\beta_t / K} = \begin{dcases}
            \br{\frac{\sigma}{K}}^{-1} \frac{\pData(\proj_D(x_{t-1}))}{\pData(\proj_D(x_t))} + O(1) & \text{if } d(x_{t-1},D)<d(x_t,D),\\[2mm] \frac{\pData(\proj_D(x_{t-1}))}{\pData(\proj_D(x_t))} + O(\sigma) & \text{if } d(x_{t-1},D)=d(x_t,D),\\[2mm]            
            \br{\frac{\sigma}{K}}\frac{\pData(\proj_D(x_{t-1}))}{\pData(\proj_D(x_t))} + O(\sigma^2) & \text{if } d(x_{t-1},D)>d(x_t,D),
        \end{dcases}
    \end{align*}
    in terms of $\sigma \to 0$. 
\end{theorem}

\begin{proof}
    For ease of notation, we set $x:=x_t$ and $z:=x_{t-1}=(x_t^{-h},y^h)$. Following first part of the proof of \Cref{thm:main}, we have (see \Cref{eq:sum-d})
    \begin{align}
        q_{t-1|t}^h(y^h | x) =& \sum_{d=0}^{H-1} \underbrace{\sum_{\substack{y^{-h}\in[K]^{H-1}\\d(x^{-h},y^{-h})=d}} \frac{q_{t-1}(y)}{q_{t}(x)} \cdot \br{1-\beta_t \frac{K-1}{K}}^{H-1-d}\br{\frac{\beta_t}{K}}^{d+1}}_{=: C_d}, \label{eq:sum-d-general}
    \end{align}
    where (see \Cref{eq:score})
    \begin{align}
        \frac{q_{t-1}(y)}{q_{t}(x)} =& \frac{\pData(\proj_D(y))}{\pData(\proj_D(x))} \br{\frac{\sigma_t}{K}}^{d(y,D)-d(x,D)}\cdot\br{\frac{\sigma_{t-1}}{\sigma_t}}^{d(y,D)} \br{1+O\br{\frac{\sigma_t}{K}}}. \label{eq:score-general}
    \end{align}
    We now consider the different terms $C_d$ in the sum in \Cref{eq:sum-d-general} by plugging in the ratio from \Cref{eq:score-general}: 
    
    \textbf{Lowest order: $\mathbf{d=0}$.} 
    We have $d(x^{-h},y^{-h})=0$, so $y=(x^{-h},y^h)$. Hence, the only term here is 
    \begin{align*}
        C_0 = \frac{\pData(\proj_D(y))}{\pData(\proj_D(x))} \br{\frac{\sigma_t}{K}}^{d(y,D)-d(x,D)}\cdot\br{\frac{\sigma_{t-1}}{\sigma_t}}^{d(y,D)} \br{1+O\br{\frac{\sigma_t}{K}}} \br{1-\beta_t \frac{K-1}{K}}^{H-1}\br{\frac{\beta_t}{K}}.
    \end{align*}
    Now recall that we assumed $\frac{1}{T}=O(\sigma^2)$. Thus by \Cref{assumption:schedule-general}, we have $|\sigma_t - \sigma| \leq \max\{|\sigma_t-\sigma_{t-1}|,|\sigma_t-\sigma_{t+1}|\} = O(T^{-1})$, as well as $\frac{\sigma_{t-1}}{\sigma_t}= 1 - \frac{\sigma_t-\sigma_{t-1}}{\sigma_t} = 1 + = \frac{O(T^{-1})}{\sigma + O(T^{-1})} = O(\sigma^{-1}T^{-1})$ and $1-\beta_t \frac{K-1}{K} =1 + O(T^{-1})$, hence 
    \begin{align*}
        \frac{C_0}{\beta_t / K } =& \frac{\pData(\proj_D(y))}{\pData(\proj_D(x))} \br{\frac{\sigma}{K}}^{d(y,D)-d(x,D)} \br{1 + O(\sigma^{-1} T^{-1})}\\
        =& \frac{\pData(\proj_D(y))}{\pData(\proj_D(x))} \br{\frac{\sigma}{K}}^{d(y,D)-d(x,D)} \br{1 + O(\sigma)}.
    \end{align*}

    \textbf{Higher orders: $\mathbf{d>0}$.} In this case, we have 
    \begin{align*}
        \frac{C_d}{\beta_t / K} = \sum_{\substack{y^{-h}\in[K]^{H-1}\\d(x^{-h},y^{-h})=d}} \frac{\pData(\proj_D(y))}{\pData(\proj_D(x))} \br{\frac{\sigma_t}{K}}^{d(y,D)-d(x,D)}\cdot\br{\frac{\sigma_{t-1}}{\sigma_t}}^{d(y,D)} \cdot \br{1-\beta_t \frac{K-1}{K}}^{H-1-d}\br{\frac{\beta_t}{K}}^{d},
    \end{align*}
    Once more, $1-\beta_t\frac{K-1}{K} = 1+O(T^{-1})$ and $\sigma_{t-1}/\sigma_t = 1 + O(\sigma^{-1}T^{-1})=1+O(\sigma)$ since we assumed $T^{-1}=O(\sigma^2)$. Now $(\beta_t/K)^d = O((TK)^{-d})$. Hence, the above simplifies to
    \begin{align*}
        \frac{C_d}{\beta_t / K} = \sum_{\substack{y^{-h}\in[K]^{H-1}\\d(x^{-h},y^{-h})=d}} \frac{\pData(\proj_D(y))}{\pData(\proj_D(x))} \br{\frac{\sigma}{K}}^{d(y,D)-d(x,D)} \cdot O\br{\br{\frac{1}{TK}}^{d}}(1+O(\sigma))
    \end{align*}
    We consider each term separately and distinguish cases according to $d(z,D)$ where $z:=(x^{-h},y^h)$. In each case we will make use of the fact that by the triangle inequality,
    \begin{align}
        d(z,D) = d((x^{-h},y^h),D) \leq d((x^{-h},y^h),y) + d(y,D) = d(x^{-h},y^{-h}) + d(y,D) = d + d(y,D) . \label{eq:tri}
    \end{align}

    \textit{Case 1:} $d(z,D)=d(x,D)-1$.\\
    Then by \Cref{eq:tri}, $d(y,D)-d(x,D)=d(y,D)-d(z,D)-1\geq -d-1$, and thus $(\sigma/K)^{d(y,D)-d(x,D)} \leq (\sigma/K)^{-(d+1)}$. Hence,
    \begin{align*}
        \frac{C_d}{\beta_t / K} \leq& O(\sigma^{-(d+1)}T^{-d}K(1+\sigma^{-1}T^{-1})) 
        \leq O(\sigma^{2d-d-1}(1+\sigma)))
        \leq O(1),
    \end{align*}
    since we assumed $\frac{1}{T}=O(\sigma^2)$ and $d\geq1$. 

    \textit{Case 2:} $d(z,D)=d(x,D)$.\\
    Then by \Cref{eq:tri}, $d(y,D)-d(x,D)=d(y,D)-d(z,D)\geq -d$, and thus $(\sigma/K)^{d(y,D)-d(x,D)} \leq (\sigma/K)^{-d}$. Hence,
    \begin{align*}
        \frac{C_d}{\beta_t / K} \leq& O(\sigma^{-d}T^{-d}(1+\sigma^{-1}T^{-1})) 
        \leq O(\sigma^{2d-d}(1+\sigma)))
        \leq O(\sigma),
    \end{align*}
    since we assumed $\frac{1}{T}=O(\sigma^2)$ and $d\geq1$. 

    \textit{Case 3:} $d(z,D)=d(x,D)+1$.\\
    Then by \Cref{eq:tri}, $d(y,D)-d(x,D)=d(y,D)-d(z,D)+1\geq -d+1$, and thus $(\sigma/K)^{d(y,D)-d(x,D)} \leq (\sigma/K)^{-d+1}$. Hence,
    \begin{align*}
        \frac{C_d}{\beta_t / K} \leq& O(\sigma^{-d+1}T^{-d}(1+\sigma^{-1}T^{-1})) 
        \leq O(\sigma^{2d-d+1}(1+\sigma)))
        \leq O(\sigma^2),
    \end{align*}
    since we assumed $\frac{1}{T}=O(\sigma^2)$ and $d\geq1$. 

    Hence, distinguishing these cases for the overall value $q^h_{t-1|t}(y^h|x)=C_0 + \sum_{d=1}^H C_d$ and summing up yields, for $z=(x^{-h},y^h)$
    \begin{align*}
        \frac{q^h_{t-1|t}(y^h|x)}{\beta_t/K} = \begin{dcases}
            \br{\frac{\sigma}{K}}^{-1} \frac{\pData(\proj_D(z))}{\pData(\proj_D(x))} + O(1) & \text{if } d(z,D)=d(x,D)-1,\\[2mm] \frac{\pData(\proj_D(z))}{\pData(\proj_D(x))} + O(\sigma) & \text{if } d(z,D)=d(x,D),\\[2mm]            
            \br{\frac{\sigma}{K}}\frac{\pData(\proj_D(z))}{\pData(\proj_D(x))} + O(\sigma^2) & \text{if } d(z,D)>d(x,D).
        \end{dcases}
    \end{align*}
\end{proof}

\paragraph{Absorbing diffusion.} The statement and proof for absorbing diffusion only differ slightly.\\

\begin{theorem}[Rate Separation (Absorbing); Discrete Time] \label{thm:main-absorb-generalized}
    Let $q$ be the law of the absorbing diffusion process. Let $x_t \in \supp(q_t)$ with $x^h=m$ and $y^h \in [K]$ and set $x_{t-1}:=(x_t^{-h},y^h)$. Consider any noise level $\sigma < 1$ and the corresponding time step $t=t_\sigma(T)=\arg\min_{s\in[T]}|\sigma_s-\sigma|$. Under \Cref{assumption:schedule-general}, for a sufficiently fine-grained discretization $\frac{1}{T} = O(\sigma^2)$, we have 
    \begin{align*}
        \frac{q_{t-1 | t}^h(y^h | x_t)}{\beta_t} = \begin{dcases}
            \sigma^{-1} \frac{\pData(\proj_D(x_{t-1}))}{\pData(\proj_D(x_t))} + O(1) & \text{if } d(x_{t-1},D)<d(x_t,D),\\[2mm] 
            O(\sigma) & \text{if } d(x_{t-1},D)=d(x_t,D),
        \end{dcases}
    \end{align*}
    in terms of $\sigma \to 0$. 
\end{theorem}

\begin{proof}
    For ease of notation, we set $x:=x_t$ and $z:=x_{t-1}=(x_t^{-h},y^h)$. Following first part of the proof of \Cref{thm:main-absorb}, we have (see \Cref{eq:sum-d-absorb})
    \begin{align}
        q_{t-1|t}^h(y^h | x) =& \sum_{d=0}^{H-1} \underbrace{\sum_{\substack{y^{-h}\in[K]^{H-1}\colon\\\forall i\colon y^{i}=x^{i}\vee x^{i}=m,\\d(x^{-h},y^{-h})=d}} \frac{q_{t-1}(y)}{q_{t}(x)} \cdot \br{1-\beta_t}^{H-1-d}\beta_t^{d+1}}_{=: C_d}.
        \label{eq:sum-d-absorb-general}
    \end{align}
    where (see \Cref{eq:score-absorb})
    \begin{align}
        \frac{q_{t-1}(y)}{q_{t}(x)} =& \frac{\pData(\proj_D(y))}{\pData(\proj_D(x))} \sigma_t^{d(y,D)-d(x,D)}\cdot\br{\frac{\sigma_{t-1}}{\sigma_t}}^{d(y,D)} \br{1+O(\sigma_t)}. \label{eq:score-absorb-general}
    \end{align}
    if $y\in \supp(q_{t-1})$ and $0$ otherwise. 
    
    We now consider the different terms $C_d$ in the sum in \Cref{eq:sum-d-absorb-general} by plugging in the ratio from \Cref{eq:score-absorb-general}: 
    
    \textbf{Lowest order: $\mathbf{d=0}$.} 

    We have $d(x^{-h},y^{-h})=0$, so the sum in \Cref{eq:sum-d-absorb} is only over $y=(x^{-h},y^h)$. 

    \textit{Case 1:} $d(y,D)=d(x,D)-1$
    Since $x\in\supp(q_t)$, this implies $y\in\supp(q_{t-1})$ (i.e. $\exists x_0\in D \colon \forall i\colon (y^i=m \vee y^i=x_0^i)$). Thus, the only term here is 
    \begin{align*}
        C_0 = \frac{\pData(\proj_D(y))}{\pData(\proj_D(x))}\sigma_t^{d(y,D)-d(x,D)}\cdot\br{\frac{\sigma_{t-1}}{\sigma_t}}^{d(y,D)} \br{1+O(\sigma_t)} \br{1-\beta_t}^{H-1}\beta_t.
    \end{align*}
    Now recall that we assumed $\frac{1}{T}=O(\sigma^2)$. Thus by \Cref{assumption:schedule-general}, we have $|\sigma_t - \sigma| \leq \max\{|\sigma_t-\sigma_{t-1}|,|\sigma_t-\sigma_{t+1}|\} = O(T^{-1})$, as well as $\frac{\sigma_{t-1}}{\sigma_t}= 1 - \frac{\sigma_t-\sigma_{t-1}}{\sigma_t} = 1 +  \frac{O(T^{-1})}{\sigma + O(T^{-1})} = O(\sigma^{-1}T^{-1})$ and $1-\beta_t =1 + O(T^{-1})$, hence 
    \begin{align*}
        \frac{C_0}{\beta_t / K } =& \frac{\pData(\proj_D(y))}{\pData(\proj_D(x))} \sigma^{d(y,D)-d(x,D)} \br{1 + O(\sigma^{-1} T^{-1})}\\
        =& \frac{\pData(\proj_D(y))}{\pData(\proj_D(x))} \sigma^{d(y,D)-d(x,D)} \br{1 + O(\sigma)}\\
        =& \frac{\pData(\proj_D(y))}{\pData(\proj_D(x))} \sigma^{-1} + O(1).
    \end{align*}

    \textit{Case 2:} $d(y,D)\geq d(x,D)$\\
    Note that since $x^h=m$ and all $x\in D$ do not contain the mask token $m$, we must in fact have $d(y,D)=d(x,D)$. Now since $x^h$ has been unmasked to $y^h$ but the distance to $D$ does not decrease, this means that there is no $x_0 \in D $ such that $\forall i \colon (y^i=m \vee y^i = x_0^i)$. Hence $q_{t-1}(y)=0$ and thus 
    \begin{align*}
        \frac{C_d}{\beta_t} = 0.
    \end{align*}

    \textbf{Higher orders: $\mathbf{d>0}$.} In this case, we have 
    \begin{align*}
        \frac{C_d}{\beta_t } \leq \sum_{\substack{y^{-h}\in[K]^{H-1}\colon\\\forall i\colon y^{i}=x^{i}\vee x^{i}=m,\\d(x^{-h},y^{-h})=d}} \frac{\pData(\proj_D(y))}{\pData(\proj_D(x))} \sigma_t^{d(y,D)-d(x,D)}\cdot\br{\frac{\sigma_{t-1}}{\sigma_t}}^{d(y,D)} \cdot \br{1-\beta_t}^{H-1-d}\beta_t^{d},
    \end{align*}
    Once more, $1-\beta_t = 1+O(T^{-1})$ and $\sigma_{t-1}/\sigma_t = 1 + O(\sigma^{-1}T^{-1})=1+O(\sigma)$ since we assumed $T^{-1}=O(\sigma^2)$. Now $\beta_t^d = O(T^{-d})$. Hence, the above simplifies to
    \begin{align*}
        \frac{C_d}{\beta_t } \leq \sum_{\substack{y^{-h}\in[K]^{H-1}\colon\\\forall i\colon y^{i}=x^{i}\vee x^{i}=m,\\d(x^{-h},y^{-h})=d}} \frac{\pData(\proj_D(y))}{\pData(\proj_D(x))} \sigma^{d(y,D)-d(x,D)} \cdot O(T^{-d})(1+O(\sigma))
    \end{align*}
    We consider each term separately and distinguish cases according to $d(z,D)$ where $z:=(x^{-h},y^h)$. In each case we will make use of the fact that by the triangle inequality,
    \begin{align}
        d(z,D) = d((x^{-h},y^h),D) \leq d((x^{-h},y^h),y) + d(y,D) = d(x^{-h},y^{-h}) + d(y,D) = d + d(y,D) . \label{eq:tri-absorb}
    \end{align}
    Note that since $y^h\neq m$ and $x^h=m$, we can only have $d(z,D)=d(x,D)-1$ or $d(z,D)=d(x,D)$.

    \textit{Case 1:} $d(z,D)=d(x,D)-1$.\\
    Then by \Cref{eq:tri-absorb}, $d(y,D)-d(x,D)=d(y,D)-d(z,D)-1\geq -d-1$, and thus $\sigma^{d(y,D)-d(x,D)} \leq \sigma^{-(d+1)}$. Hence,
    \begin{align*}
        \frac{C_d}{\beta_t} \leq& O(\sigma^{-(d+1)}T^{-d}(1+\sigma^{-1}T^{-1})) 
        \leq O(\sigma^{2d-d-1}(1+\sigma)))
        \leq O(1),
    \end{align*}
    since we assumed $\frac{1}{T}=O(\sigma^2)$ and $d\geq1$. 

    \textit{Case 2:} $d(z,D)=d(x,D)$.\\
    Then by \Cref{eq:tri-absorb}, $d(y,D)-d(x,D)=d(y,D)-d(z,D)\geq -d$, and thus $\sigma^{d(y,D)-d(x,D)} \leq \sigma^{-d}$. Hence,
    \begin{align*}
        \frac{C_d}{\beta_t} \leq& O(\sigma^{-d}T^{-d}(1+\sigma^{-1}T^{-1})) 
        \leq O(\sigma^{2d-d}(1+\sigma)))
        \leq O(\sigma),
    \end{align*}
    since we assumed $\frac{1}{T}=O(\sigma^2)$ and $d\geq1$. 

    Hence, distinguishing these cases for the overall value $q^h_{t-1|t}(y^h|x)=C_0 + \sum_{d=1}^H C_d$ and summing up yields, for $z=(x^{-h},y^h)$
    \begin{align*}
        \frac{q^h_{t-1|t}(y^h|x)}{\beta_t/K} = \begin{dcases}
            \sigma^{-1} \frac{\pData(\proj_D(z))}{\pData(\proj_D(x))} + O(1) & \text{if } d(z,D)=d(x,D)-1,\\[2mm] O(\sigma) & \text{if } d(z,D)=d(x,D).\\[2mm]            
        \end{dcases}
    \end{align*}
\end{proof}

\section{Further Details on the Synthetic Experiments} \label{app:synthetic}

\subsection{A simple regular language baseline} \label{app:def-lang}

We first describe the underlying synthetic data distribution. We consider a vocabulary of size $K$ indexed by the natural numbers from $1$ to $K$ and sequences of length $H$. We randomly sample a string of this length by rolling out the following random walk: The first token $X^1$ is sampled uniformly at random from $[K]$. Next, for any $1 \leq h \leq H-1$, given $X^{h}=x^h$ we sample $X^{h+1}$ according to an $x_h$-dependent probability distribution as
\begin{align*}
    x_{h+1} = \begin{cases}
        x_h+1 \quad &\text{w.p. } p_{\text{up}}(x_h)\\
        x_h &\text{w.p. } p_{\text{stay}}(x_h)\\
        x_h-1&\text{w.p. } p_{\text{down}}(x_h)\\
    \end{cases}
\end{align*}
for $x_h\in \{2,\dots,K-1\}$, where 
\begin{align*}
\begin{pmatrix}
    p_{\text{up}}(x)\\
    p_{\text{stay}}(x)\\
    p_{\text{down}}(x)
\end{pmatrix}
&\propto 
\begin{pmatrix}
    0.45 \exp(1.0 \cdot \sin(0.61(x-1) + 1.37))\\
    0.2 \exp(1.0 \cdot \sin(0.37(x-1) + 0.11))\\
    0.35 \exp(1.0 \cdot \cos(0.53 \cdot (x-1) - 0.29)),
\end{pmatrix}
\end{align*}
simply meaning that the transitions vary in an oscillating way from the biased random walk $(p_{\text{up}},p_{\text{stay}}, p_{\text{down}})=(0.45,0.2,0.35)$. Whenever the above random walk would leave $[K]$, we clamp it to $1$ or $K$, respectively. Deciding whether a string is in the support of this distribution can trivially be done by checking the range of neighboring tokens for each token position. Computing the exact Hamming distance can be done via a simple and standard dynamic programming subroutine.

\subsection{Further experimental details} \label{app:syn-exp-det}

As for \Cref{fig:synthetic-support-frequency,fig:synthetic-projector-main}, we consider learning the distribution of the above section with $n=100,000$ samples by training standard discrete denoising diffusion models \citep{austin2021structured}. We equip this model with a conventional bidirectional transformer architecture with $108,256$ trainable parameters (two layers, four heads, and feedforward dimension $64$) and train for $50,000$ steps with a batch size of $512$. As for the distribution, we vary sequence length $H\in\{32,64\}$ and vocabulary size $K\in\{32,64\}$. On the model side, we modify the number of diffusion steps and consider $T \in \{H, 2H\}$. We use a ``linear'' noise schedule $\sigma_t = t/T$ (corresponding to $\beta_t=(T-t+1)^{-1}$) throughout. We use the AdamW optimizer with learning rate $0.0003$ and set the hybrid loss coefficient in D3PM to $0.001$. The validation set size is $1,000$ for all runs and all listed metrics. All runs are averaged across five independent training runs. 

\subsection{Support and frequency probes}

We report the additional support versus frequency results for the synthetic setup.

\begin{figure}[H]
  \centering
  \begin{subfigure}[t]{0.47\linewidth}
    \centering
    \includegraphics[width=\linewidth]{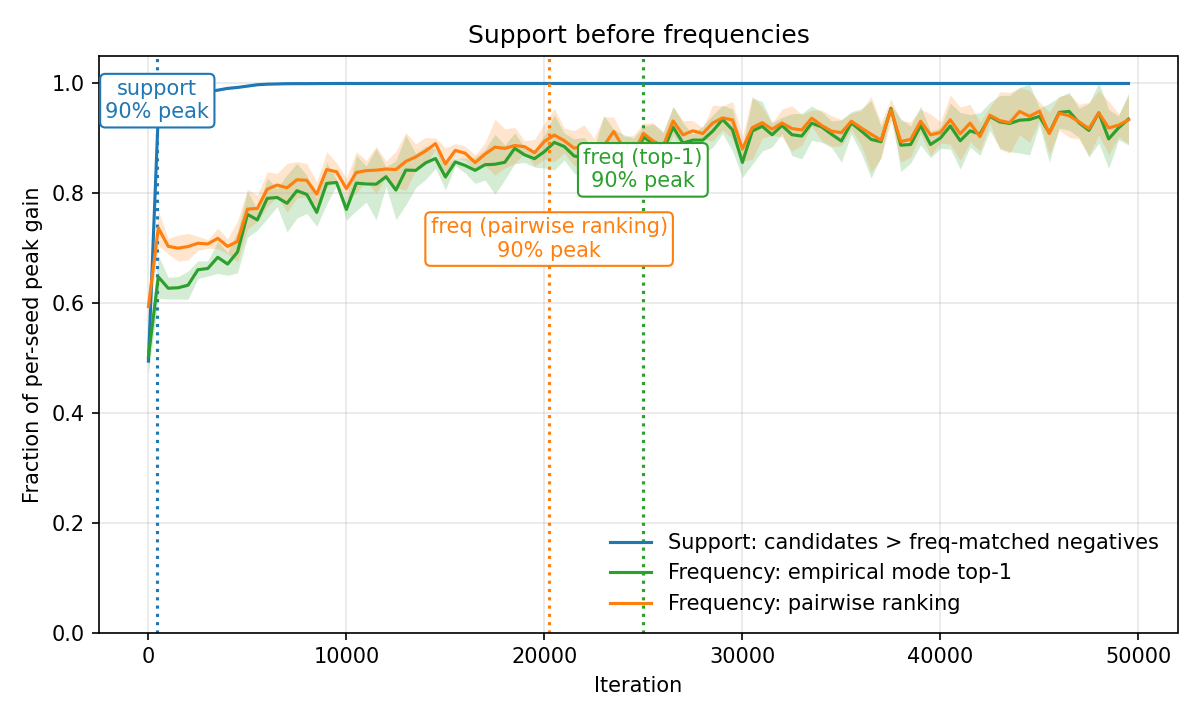}
    \caption{Absorbing diffusion}
  \end{subfigure}\hfill
  \begin{subfigure}[t]{0.47\linewidth}
    \centering
    \includegraphics[width=\linewidth]{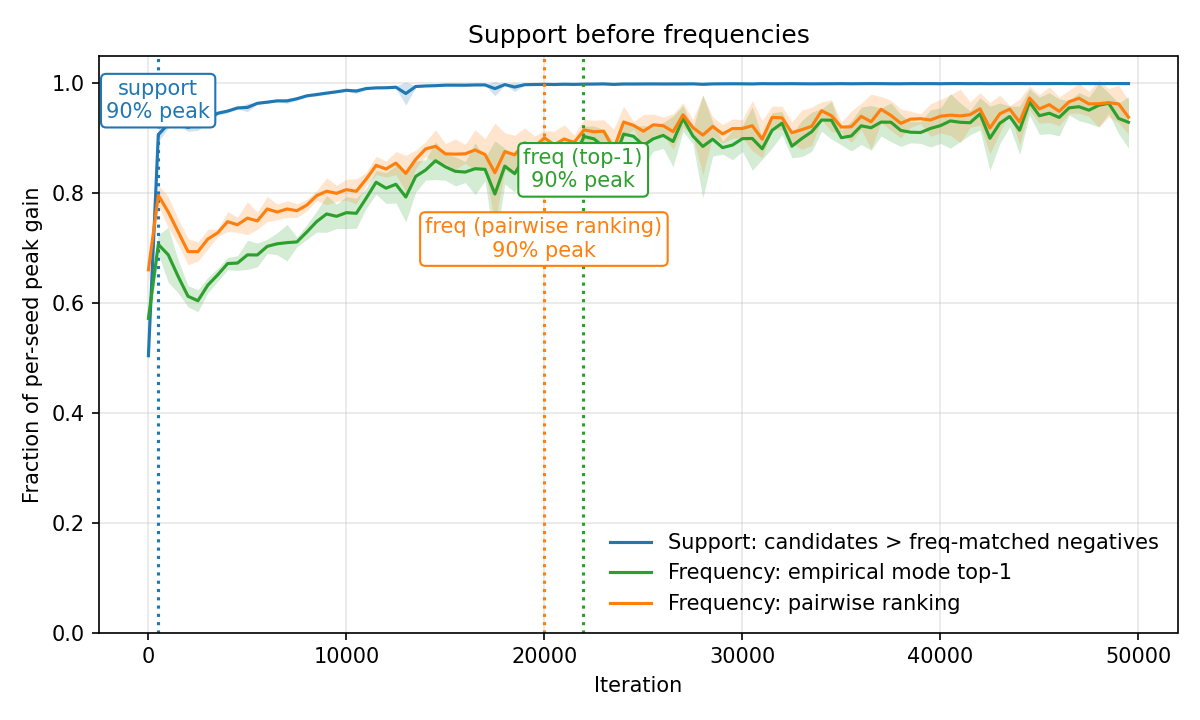}
    \caption{Uniform diffusion}
  \end{subfigure}
  \caption{Results for $H=32$, $K=64$, $T=32$.}
  \label{fig:synthetic-support-frequency-app-1}
\end{figure}

\begin{figure}[H]
  \centering
  \begin{subfigure}[t]{0.47\linewidth}
    \centering
    \includegraphics[width=\linewidth]{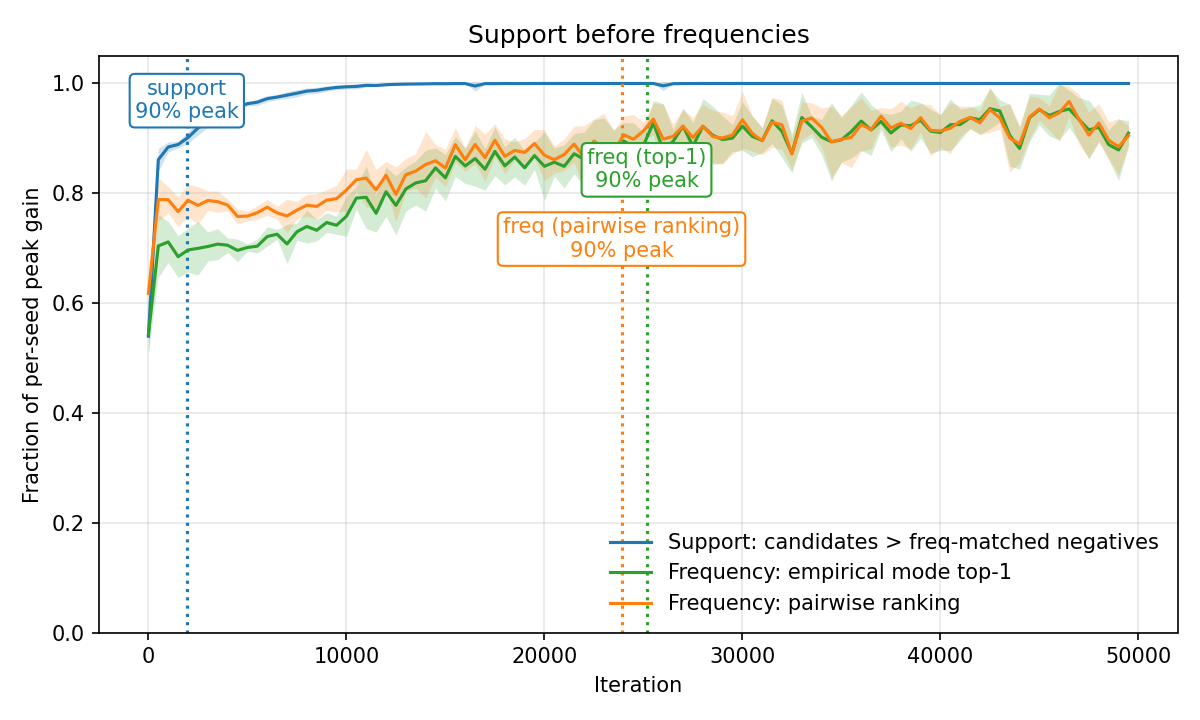}
    \caption{Absorbing diffusion}
  \end{subfigure}\hfill
  \begin{subfigure}[t]{0.47\linewidth}
    \centering
\includegraphics[width=\linewidth]{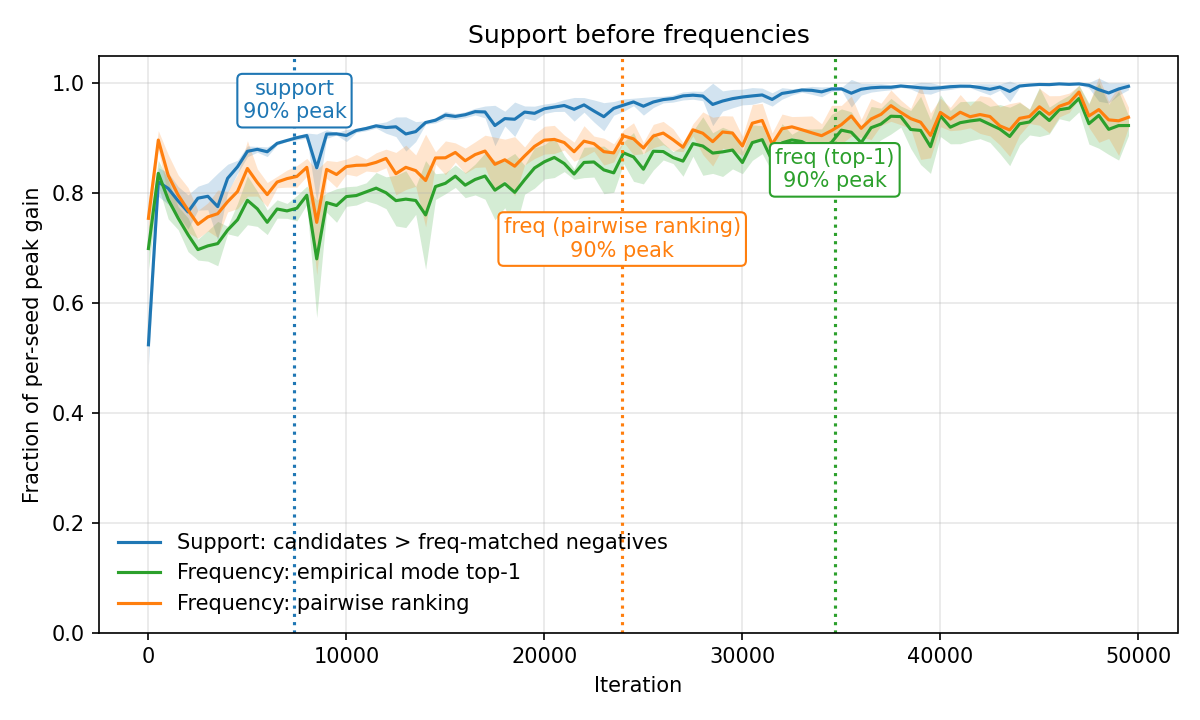}
    \caption{Uniform diffusion}
  \end{subfigure}
  \caption{Results for $H=64$, $K=32$, $T=64$.}
  \label{fig:synthetic-support-frequency-app-2}
\end{figure}

\begin{figure}[H]
  \centering
  \begin{subfigure}[t]{0.47\linewidth}
    \centering
    \includegraphics[width=\linewidth]{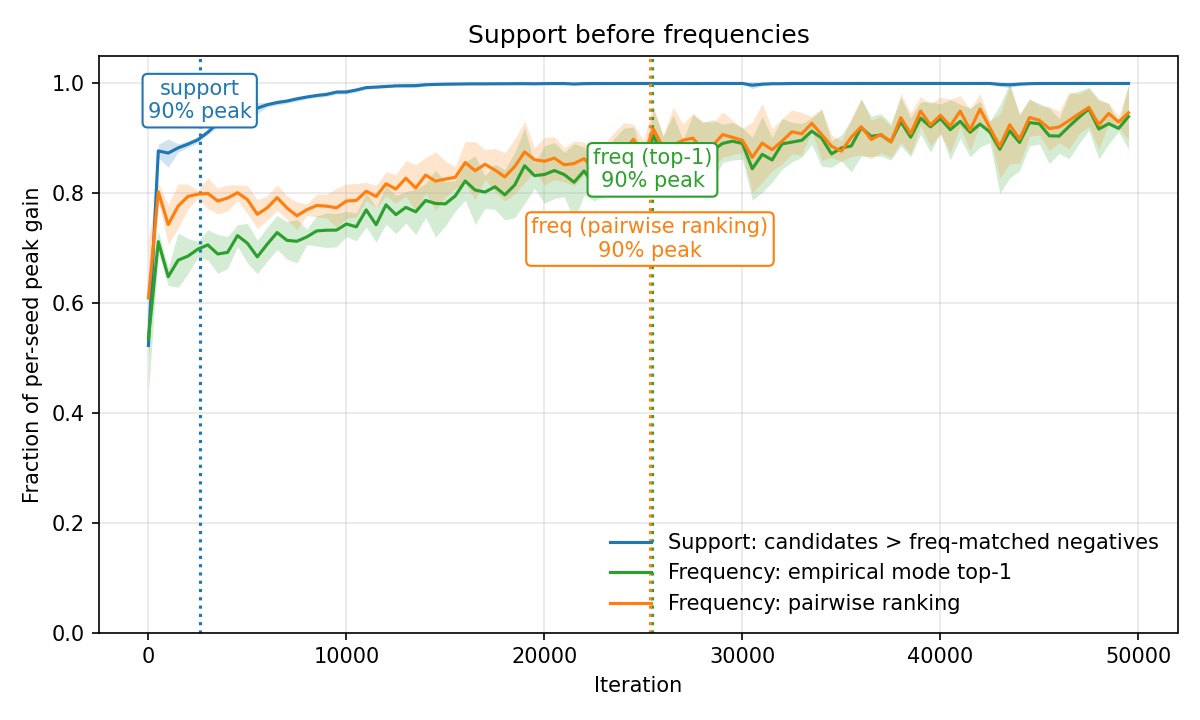}
    \caption{Absorbing diffusion}
  \end{subfigure}\hfill
  \begin{subfigure}[t]{0.47\linewidth}
    \centering
\includegraphics[width=\linewidth]{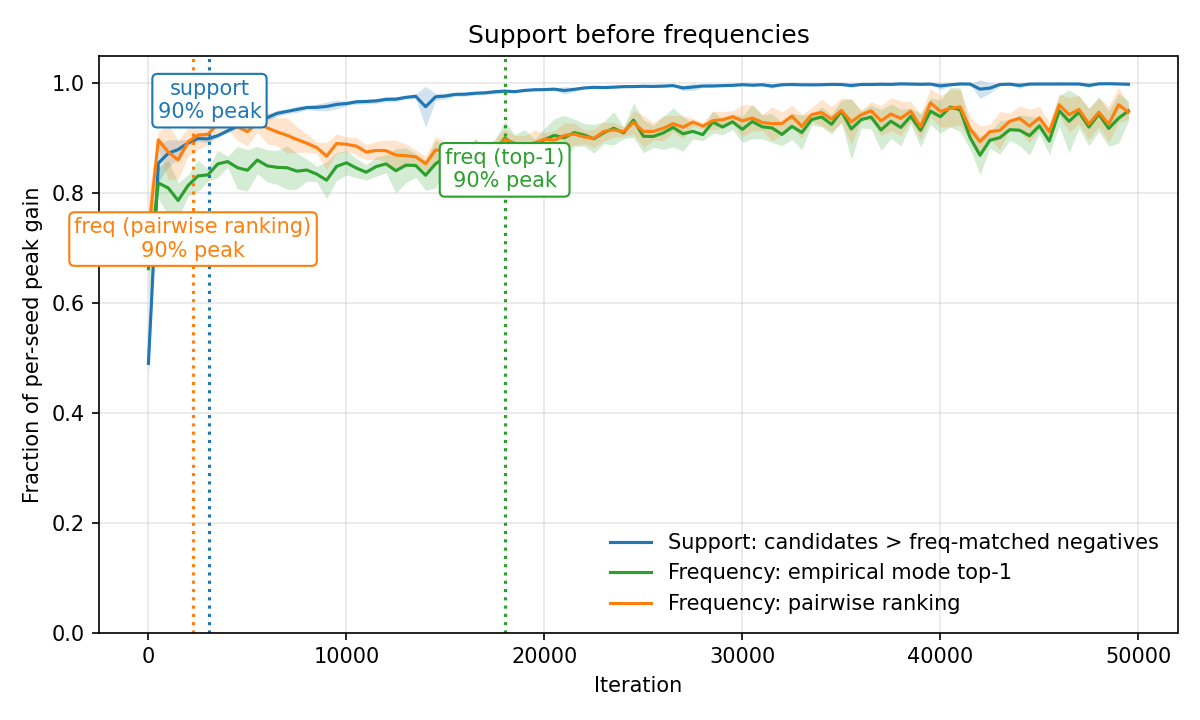}
    \caption{Uniform diffusion}
  \end{subfigure}
  \caption{Results for $H=64$, $K=32$, $T=128$.}
  \label{fig:synthetic-support-frequency-app-3}
\end{figure}

\subsection{Modified Inference-Time Algorithms} \label{sec:projector-samplers}

As motivated by our theory, we split the inference-time procedure into two phases:
\begin{itemize}
    \item \textbf{Phase 1:} Standard decoding up to a small noise level $\sigma$,
    \item \textbf{Phase 2:} Modified sampling based on the derived rate separation.
\end{itemize}
The role of the first phase is that of the standard sampler---it denoises the initial sample $x_0\sim\pUnif$, but only up until a certain point. Since our theory suggests that $p^h_{t-1|t}$ implicitly learns a classifier for the distance to the support $D=\supp(\pData)$, we explicitly make use of this fact by emulating such a projection. 

\Cref{cor:coarse-score-projection} directly motivates the following algorithm, which assumes that there is a threshold separating the direction, but the $\pData$-dependent coefficient in the expansion in \Cref{thm:main} is not well-estimated. (We write the algorithms for the uniform case; in the absorbing case, the only difference is that we should replace $K$ by $1$).

\begin{algorithm}[H] 
    \caption{Threshold sampling} \label{algo:threshold}
    \begin{algorithmic}
        \Require Initial distribution $\pHat_\sigma$, normalized $s := \frac{p^h_{t-1|t}}{\beta_t/K}$ (where $t=t_\sigma(T)$), threshold $\tau$
        \State Sample $X_0 \sim \pHat_{\sigma}$. 
        \For{$t=1, 2, \dots$}
            \State Set $S_{\text{proj}}:= \{(h,y^h)\in[H]\times[K] \mid (\sigma/K)s(X_t)_{h,y^h} > \tau\}$
            \If{$S_{\text{proj}} \neq \emptyset$}
                \State Sample $(h,y^h) \in S_{\text{proj}}$ u.a.r.
                \State $X_{t+1}=(X_{t}^{-h},y^h)$
            \Else
                \State \Return $X_t$
            \EndIf
        \EndFor
    \end{algorithmic}
\end{algorithm}

A simple variant that is more robust to the error assumption but produces samples with worse diversity is the following greedy algorithm.

\begin{algorithm}[H] 
    \caption{Hard-max sampling} \label{algo:greedy}
    \begin{algorithmic}
        \Require Initial distribution $\pHat_\sigma$, effective score $s := \frac{p^h_{t-1|t}}{\beta_t/K}$ (where $t=t_\sigma(T)$), threshold $\tau$
        \State Sample $X_0 \sim \pHat_{\sigma}$. 
        \For{$t=1, 2, \dots$}
                \State Sample $(h,y^h) \in \arg\max\bc{(\sigma/K)s(X_t)_{h,y^h}}$
                \If{$(\sigma/K)s(X_t)_{h,y^h}>\tau$}
                    \State $X_{t+1}=(X_{t}^{-h},y^h)$
                \Else
                    \State \Return $X_t$
            \EndIf
        \EndFor
    \end{algorithmic}
\end{algorithm}

For all our experiments, we choose a noise level of $\sigma=0.1$ and threshold of $0.05$. We have also observed the same effects robustly across a range of other noise levels and thresholds for both uniform and masked diffusion, which we omit here for the sake of clarity of presentation.

\subsection{Extended sampler comparison}

We provide extended details about the evaluation of masked and uniform diffusion with the modified samplers. Throughout, (a) and (b) measure the fraction of samples in the support and the average distance to the support of the samples across training checkpoints, respectively. In (c), we evaluate the final checkpoint with and without the modified samplers, and measure the fraction of their samples landing in the support. In (d), we directly compare the distance of these samples, displaying the fraction of samples in which the modified sampler wins (green), draws (gray), or loses against the standard sampler in terms of getting closer to the support in Hamming distance. 

\begin{figure}[t]
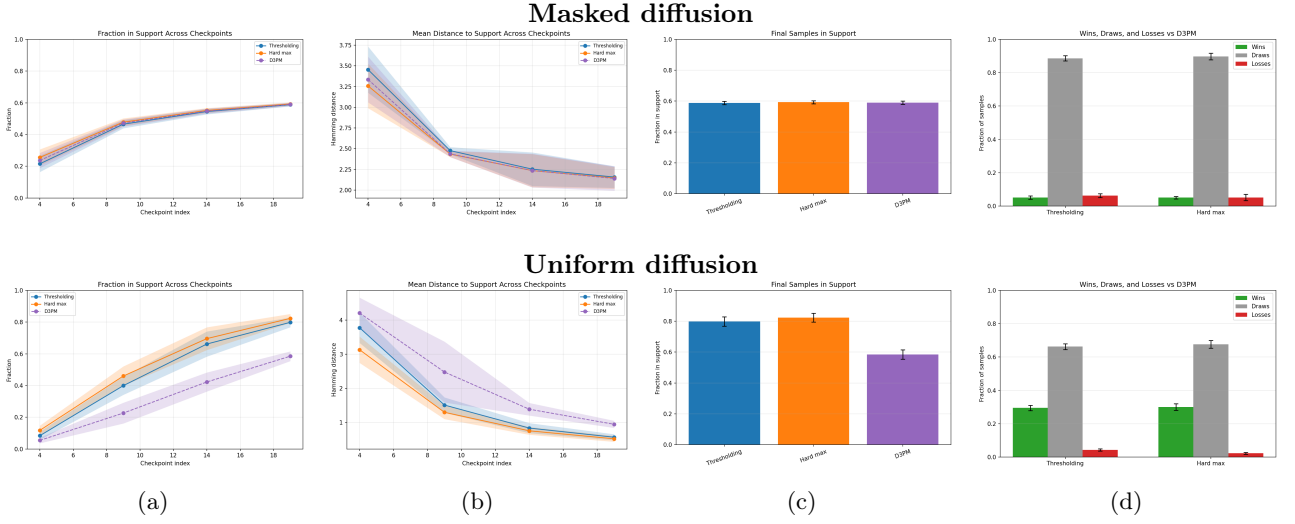

  \centering

  {\centering\textbf{Masked diffusion}}\\
  \begin{subfigure}[t]{0.24\textwidth}
    \centering
    \metricplot{figures/app_synthetic_shortened/absorb_L32K64T64/sum_n0p1_t6_thr0p05}{summary_frac_support.png}
  \end{subfigure}\hfill
  \begin{subfigure}[t]{0.24\textwidth}
    \centering
\metricplot{figures/app_synthetic_shortened/absorb_L32K64T64/sum_n0p1_t6_thr0p05}{summary_mean_dist.png}
  \end{subfigure}\hfill
  \begin{subfigure}[t]{0.24\textwidth}
    \centering
    \metricplot{figures/app_synthetic_shortened/absorb_L32K64T64/ckpt19_n0p1_t6_thr0p05}{ckpt19_frac_support.png}
  \end{subfigure}\hfill
  \begin{subfigure}[t]{0.24\textwidth}
    \centering
\metricplot{figures/app_synthetic_shortened/absorb_L32K64T64/ckpt19_n0p1_t6_thr0p05}{ckpt19_wdl.png}
  \end{subfigure}

  \vspace{0.75em}

  {\centering\textbf{Uniform diffusion}}\\
  \begin{subfigure}[t]{0.24\textwidth}
    \centering
        \metricplot{figures/app_synthetic_shortened/unif_L32K64T64/sum_n0p1_t6_thr0p05}{summary_frac_support.png}
    \caption{}
  \end{subfigure}\hfill
  \begin{subfigure}[t]{0.24\textwidth}
    \centering
        \metricplot{figures/app_synthetic_shortened/unif_L32K64T64/sum_n0p1_t6_thr0p05}{summary_mean_dist.png}
    \caption{}
  \end{subfigure}\hfill
  \begin{subfigure}[t]{0.24\textwidth}
    \centering
        \metricplot{figures/app_synthetic_shortened/unif_L32K64T64/ckpt19_n0p1_t6_thr0p05}{ckpt19_frac_support.png}
    \caption{}
  \end{subfigure}\hfill
  \begin{subfigure}[t]{0.24\textwidth}
    \centering
    \metricplot{figures/app_synthetic_shortened/unif_L32K64T64/ckpt19_n0p1_t6_thr0p05}{ckpt19_wdl.png}
    \caption{}
  \end{subfigure}

  \caption{Results for $H=32$, $K=64$, $T=64$ (see \cref{fig:synthetic-projector-main}): (a) Fraction in support across training; (b) Distance to support across training; (c) Fraction in support (final checkpoint); (d) Win rate: Distance to support (final checkpoint)}
\end{figure}

\begin{figure}[t]
  \centering

  {\centering\textbf{Masked diffusion}}\\
  \begin{subfigure}[t]{0.24\textwidth}
    \centering
        \metricplot{figures/app_synthetic_shortened/absorb_L32K64T32/sum_n0p1_t3_thr0p05}{summary_frac_support.png}
  \end{subfigure}\hfill
  \begin{subfigure}[t]{0.24\textwidth}
    \centering
        \metricplot{figures/app_synthetic_shortened/absorb_L32K64T32/sum_n0p1_t3_thr0p05}{summary_mean_dist.png}
  \end{subfigure}\hfill
  \begin{subfigure}[t]{0.24\textwidth}
    \centering
    \metricplot{figures/app_synthetic_shortened/absorb_L32K64T32/ckpt19_n0p1_t3_thr0p05}{ckpt19_frac_support.png}
  \end{subfigure}\hfill
  \begin{subfigure}[t]{0.24\textwidth}
    \centering
    \metricplot{figures/app_synthetic_shortened/absorb_L32K64T32/ckpt19_n0p1_t3_thr0p05}{ckpt19_wdl.png}
  \end{subfigure}

  \vspace{0.75em}

  {\centering\textbf{Uniform diffusion}}\\
  \begin{subfigure}[t]{0.24\textwidth}
    \centering
        \metricplot{figures/app_synthetic_shortened/unif_L32K64T32/sum_n0p1_t3_thr0p05}{summary_frac_support.png}
    \caption{}
  \end{subfigure}\hfill
  \begin{subfigure}[t]{0.24\textwidth}
    \centering
        \metricplot{figures/app_synthetic_shortened/unif_L32K64T32/sum_n0p1_t3_thr0p05}{summary_mean_dist.png}
    \caption{}
  \end{subfigure}\hfill
  \begin{subfigure}[t]{0.24\textwidth}
    \centering
    \metricplot{figures/app_synthetic_shortened/unif_L32K64T32/ckpt19_n0p1_t3_thr0p05}{ckpt19_frac_support.png}
    \caption{}
  \end{subfigure}\hfill
  \begin{subfigure}[t]{0.24\textwidth}
    \centering
    \metricplot{figures/app_synthetic_shortened/unif_L32K64T32/ckpt19_n0p1_t3_thr0p05}{ckpt19_wdl.png}
    \caption{}
  \end{subfigure}

  \caption{Results for $H=32$, $K=64$, $T=32$: (a) Fraction in support across training; (b) Distance to support across training; (c) Fraction in support (final checkpoint); (d) Win rate: Distance to support (final checkpoint)}
\end{figure}

\begin{figure}[H]
  \centering

  {\centering\textbf{Masked diffusion}}\\
  \begin{subfigure}[t]{0.24\textwidth}
    \centering
        \metricplot{figures/app_synthetic_shortened/absorb_L64K32T64/sum_n0p1_t6_thr0p05}{summary_frac_support.png}
  \end{subfigure}\hfill
  \begin{subfigure}[t]{0.24\textwidth}
    \centering
        \metricplot{figures/app_synthetic_shortened/absorb_L64K32T64/sum_n0p1_t6_thr0p05}{summary_mean_dist.png}

  \end{subfigure}\hfill
  \begin{subfigure}[t]{0.24\textwidth}
    \centering
        \metricplot{figures/app_synthetic_shortened/absorb_L64K32T64/ckpt19_n0p1_t6_thr0p05}{ckpt19_frac_support.png}

  \end{subfigure}\hfill
  \begin{subfigure}[t]{0.24\textwidth}
    \centering
        \metricplot{figures/app_synthetic_shortened/absorb_L64K32T64/ckpt19_n0p1_t6_thr0p05}{ckpt19_wdl.png}

  \end{subfigure}

  \vspace{0.75em}

  {\centering\textbf{Uniform diffusion}}\\
  \begin{subfigure}[t]{0.24\textwidth}
    \centering
        \metricplot{figures/app_synthetic_shortened/unif_L64K32T64/sum_n0p1_t6_thr0p05}{summary_frac_support.png}
    \caption{}
  \end{subfigure}\hfill
  \begin{subfigure}[t]{0.24\textwidth}
    \centering
        \metricplot{figures/app_synthetic_shortened/unif_L64K32T64/sum_n0p1_t6_thr0p05}{summary_mean_dist.png}
    \caption{}
  \end{subfigure}\hfill
  \begin{subfigure}[t]{0.24\textwidth}
    \centering
        \metricplot{figures/app_synthetic_shortened/unif_L64K32T64/ckpt19_n0p1_t6_thr0p05}{ckpt19_frac_support.png}
    \caption{}
  \end{subfigure}\hfill
  \begin{subfigure}[t]{0.24\textwidth}
    \centering
        \metricplot{figures/app_synthetic_shortened/unif_L64K32T64/ckpt19_n0p1_t6_thr0p05}{ckpt19_wdl.png}
    \caption{}
  \end{subfigure}

  \caption{Results for $H=64$, $K=32$, $T=64$: (a) Fraction in support across training; (b) Distance to support across training; (c) Fraction in support (final checkpoint); (d) Win rate: Distance to support (final checkpoint)}
\end{figure}

\begin{figure}[H]
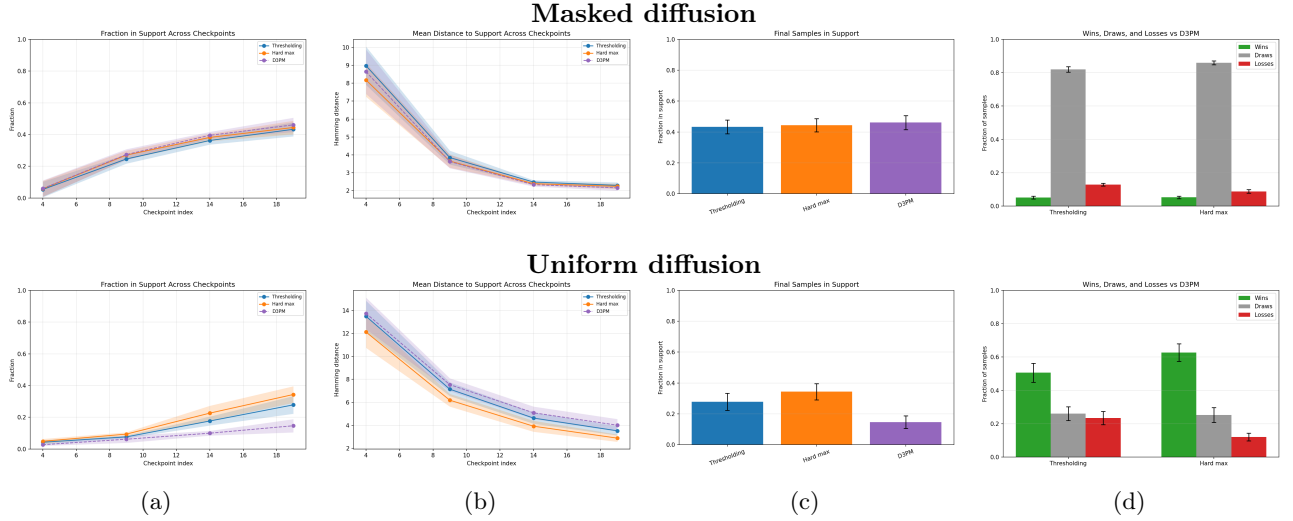

  \centering

  {\centering\textbf{Masked diffusion}}\\
  \begin{subfigure}[t]{0.24\textwidth}
    \centering
        \metricplot{figures/app_synthetic_shortened/absorb_L64K32T128/sum_n0p1_t13_thr0p05}{summary_frac_support.png}
  \end{subfigure}\hfill
  \begin{subfigure}[t]{0.24\textwidth}
    \centering
        \metricplot{figures/app_synthetic_shortened/absorb_L64K32T128/sum_n0p1_t13_thr0p05}{summary_mean_dist.png}

  \end{subfigure}\hfill
  \begin{subfigure}[t]{0.24\textwidth}
    \centering
        \metricplot{figures/app_synthetic_shortened/absorb_L64K32T128/ckpt19_n0p1_t13_thr0p05}{ckpt19_frac_support.png}

  \end{subfigure}\hfill
  \begin{subfigure}[t]{0.24\textwidth}
    \centering
        \metricplot{figures/app_synthetic_shortened/absorb_L64K32T128/ckpt19_n0p1_t13_thr0p05}{ckpt19_wdl.png}

  \end{subfigure}

  \vspace{0.75em}

  {\centering\textbf{Uniform diffusion}}\\
  \begin{subfigure}[t]{0.24\textwidth}
    \centering
        \metricplot{figures/app_synthetic_shortened/unif_L64K32T128/sum_n0p1_t13_thr0p05}{summary_frac_support.png}

    \caption{}
  \end{subfigure}\hfill
  \begin{subfigure}[t]{0.24\textwidth}
    \centering
        \metricplot{figures/app_synthetic_shortened/unif_L64K32T128/sum_n0p1_t13_thr0p05}{summary_mean_dist.png}

    \caption{}
  \end{subfigure}\hfill
  \begin{subfigure}[t]{0.24\textwidth}
    \centering
        \metricplot{figures/app_synthetic_shortened/unif_L64K32T128/ckpt19_n0p1_t13_thr0p05}{ckpt19_frac_support.png}

    \caption{}
  \end{subfigure}\hfill
  \begin{subfigure}[t]{0.24\textwidth}
    \centering
        \metricplot{figures/app_synthetic_shortened/unif_L64K32T128/ckpt19_n0p1_t13_thr0p05}{ckpt19_wdl.png}

    \caption{}
  \end{subfigure}

  \caption{Results for $H=64$, $K=32$, $T=128$: (a) Fraction in support across training; (b) Distance to support across training; (c) Fraction in support (final checkpoint); (d) Win rate: Distance to support (final checkpoint)}
\end{figure}

\subsection{Computational resources} \label{app:syn-det} 

All runs in this section were performed on a single  NVIDIA H100 Tensor Core GPU with 96GB available GPU memory (which was never used to completion by our experiments). Each run and evaluation was performed in under one hour on one such machine. The total compute to reproduce the experiment is thus this, multiplied by the number of seeds (five) and the number of setups (eight in total). Any additional compute was spent on developing and testing these training runs.

\section{Real-data masked-DLM probes on FineWeb}
\label{app:real-dlm-fineweb}

This appendix makes the FineWeb experiment in
\Cref{sec:xp-supp-then-freq} self-contained. The experiment asks whether the
support--frequency separation predicted by our small-noise analysis is visible
in a diffusion language model trained on ordinary web text.

\subsection{Why the real-data probe is indirect}

In the synthetic experiments, the support $D=\supp(\pData)$ is known: we can
directly test whether a sample lies in $D$, estimate its distance to $D$, and
identify local moves that decrease this distance. For web text, no such oracle
exists. There is no tractable list of all valid strings. 

We therefore use a held-out one-token restoration proxy. From held-out
FineWeb text, we collect local contexts $c=(\ell,r)$ and the set of center
tokens that occur between $\ell$ and $r$. For each context, this gives an
empirical candidate set $C_c$. A token in $C_c$ is not guaranteed to be the
full linguistic support for the context, but it is a concrete support-like
object measured from held-out data. The real-data question is whether the
model first learns to separate $C_c$ from non-candidates, and only later learns
the empirical frequency ordering inside $C_c$.

For instance, suppose the local context is \texttt{the [MASK] of}. We scan
held-out FineWeb and count every center token that appears between the left
token \texttt{the} and the right token \texttt{of}. If the held-out text
contains \texttt{the end of}, then \texttt{end} is an empirical candidate for
this context. A non-candidate is a token that appears elsewhere in the same
context bank but was not observed between \texttt{the} and \texttt{of}. The
support probe asks whether the model scores attested fillers above such
slot-specific non-candidates. The frequency probe asks a different question:
among attested fillers such as \texttt{end}, \texttt{rest}, and \texttt{state},
does the model rank the more frequent fillers above the rarer ones? The
implementation uses GPT-2 tokens rather than word-level strings, but the logic
is exactly this slot-filling comparison.

\subsection{Model and training recipe}
\label{subsec:fineweb-model-training-recipe}
We train a masked diffusion language model on FineWeb. 
The model belongs to the absorbing-mask, clean-token-prediction family
reviewed in \Cref{app:dlm-parameterizations}.

The model uses the GPT-2 tokenizer vocabulary of size \(K=50{,}257\) and appends
one absorbing mask token. The context length is \(H=1024\). The denoiser is a DDiT
transformer with 12 layers, 12 attention heads, hidden width 768, RoPE
position embeddings, AdaLN conditioning layers with conditioning dimension
128, dropout $0.1$, no time conditioning, and untied input/output embeddings. The resulting model has
$169.6$M trainable parameters, of which $131.0$M are non-embedding parameters. 
In the pseudocode below we keep \(\sigma\) in the notation to identify the
standard MDLM mask probability, but this particular network's logits depend only on
the corrupted token sequence.


\begin{algorithm}[H]
\caption{FineWeb masked-DLM training update}
\label{alg:real-dlm-training-update}
\begin{algorithmic}[1]
\Require Clean token batch $x\in[K]^{B\times H}$, mask token $m$, noise floor $\varepsilon_{\rm noise}=10^{-6}$
\State Draw $u_i\in[0,1]$ and set $t_i=\varepsilon_{\rm noise}+(1-\varepsilon_{\rm noise})u_i$
\State Set mask probability $\sigma_i=(1-\varepsilon_{\rm noise})t_i$ and keep probability $\alpha_i=1-\sigma_i$
\State Form $\tilde x$ by replacing each token $x_i^h$ by $m$ independently with probability $\sigma_i$
\State Compute logits $\nn_\theta(\tilde x)$ and clean-token probabilities $\widehat{\nn}_{\theta,h}(\cdot\mid \tilde x_i)$ for all positions $h$
\State Minimize the token-mean SUBS loss
\[
    \frac{1}{BH}
    \sum_{i,h}
    \mathbf{1}\{\tilde x_i^h=m\}
    \frac{1-\varepsilon_{\rm noise}}{\sigma_i}
    \bigl[-\log \widehat{\nn}_{\theta,h}(x_i^h\mid \tilde x_i)\bigr].
\]
\State Apply one AdamW step with learning rate $3\cdot 10^{-4}$, betas $(0.9,0.999)$, and $\varepsilon_{\rm Adam}=10^{-8}$
\end{algorithmic}
\end{algorithm}

FineWeb training runs use 8 H100 80GB GPUs with a total batch size of
\(409{,}600\) tokens (50 sequences of length 1024 per GPU).
Each run (170M parameters, 354M tokens) takes about 4 minutes of
training time (\(\sim\)7 minutes wall-clock) per seed.
Probe jobs are run separately on a single GPU and take about 1h44 in total.

Note that the support and frequency probes in \Cref{fig:real-dlm-support-frequency} do
not depend on a sampler choice: they score one-token masked contexts with the
clean-token predictor described in \Cref{app:dlm-parameterizations}. However, 
we had to choose a sampler to check the sampling capabilities of the trained models, 
as reported in \Cref{app:real-dlm-training-curves}. The sampler chosen for this auxiliary diagnostic will be detailed then.

For reference, we use the cached GPT-2-token shard
stream from \texttt{kjj0/fineweb10B-gpt2}. This is under MIT license,
while the underlying text remains FineWeb and is covered by the Open Data Commons Attribution License (ODC-By) v1.0 and is also
subject to the CommonCrawl terms of use. 

\subsection{Training loss curves and samples}
\label{app:real-dlm-training-curves}

\Cref{fig:real-dlm-logged-training-loss} shows the denoising loss logged
during the FineWeb masked-diffusion runs used in
\Cref{fig:real-dlm-support-frequency}. Each run logs a stochastic training loss
at every optimizer step. To make the trend readable, we smooth each seed with a
trailing 25-step average, meaning that each plotted training point averages the
current logged loss with the previous 24 logged losses when available, and then
we average the resulting traces across seeds on the shared training window. They show a smooth optimization trajectory over the 0.354B-token window used in \Cref{fig:real-dlm-support-frequency}. 
Diamonds mark the validation evaluations at the initial and final steps.

\begin{figure}[H]
  \centering
  \includegraphics[width=0.72\linewidth]{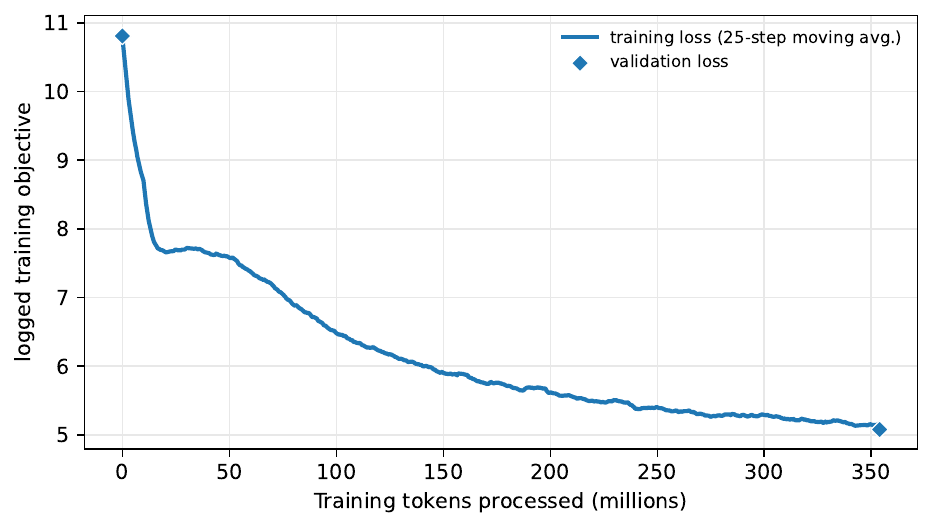}
  \caption{
  Logged denoising losses for the FineWeb masked-diffusion runs.
  The training line is the seed mean after a trailing 25-step moving average.
  Diamonds mark the validation evaluations available in the logs.
  }
  \label{fig:real-dlm-logged-training-loss}
\end{figure}

We also extend training to $3.93$B-token for one of the seeds to give an illustrative example of how losses continue to evolve past the window used in \Cref{fig:real-dlm-support-frequency}. For this extended run, we
also logged validation every 120 optimization steps. This gives a denser
single-seed view of the same quantity over the longer trajectory, see
\Cref{fig:real-dlm-logged-training-loss-long}.

\begin{figure}[H]
  \centering
  \includegraphics[width=0.72\linewidth]{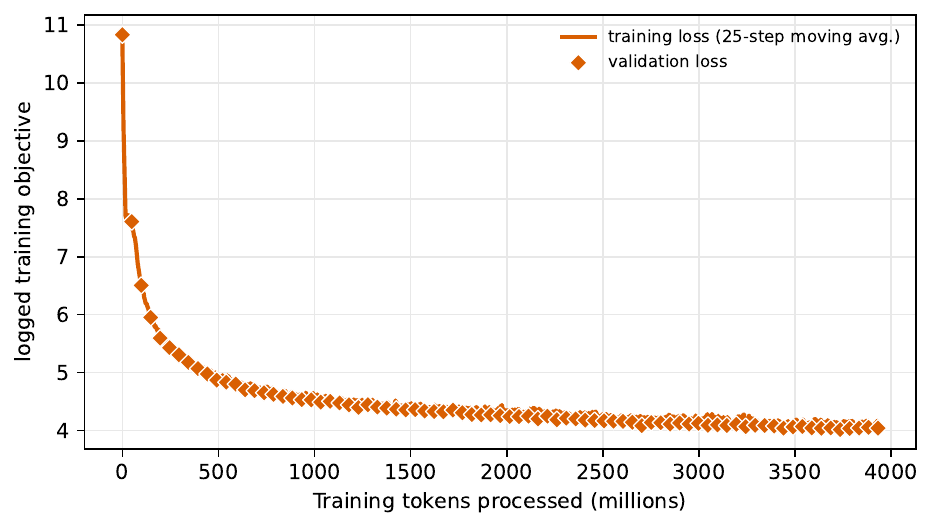}
  \caption{
  Logged denoising losses for the $3.93$B-token training run. The training curve uses the same trailing 25-step moving
  average as \Cref{fig:real-dlm-logged-training-loss}; diamonds mark the   validation evaluations from that continuation.
  }
  \label{fig:real-dlm-logged-training-loss-long}
\end{figure}

We also sampled at different stages during the extended run on $3.9$B tokens. 
The sampler (\Cref{alg:real-dlm-ancestral-sampler}) is basic ancestral sampling, as per Eq.~7 in
\citet{sahoo2024simple}, specialized to the log-linear mask schedule.

\begin{algorithm}[H]
\caption{Default absorbing-mask ancestral sampler}
\label{alg:real-dlm-ancestral-sampler}
\begin{algorithmic}[1]
\Require Number of reverse steps $T$, sampling endpoint $\varepsilon_{\rm sample}=10^{-5}$, target length $H$
\State Initialize $x$ as the all-mask sequence of length $H$
\For{$t=0,\ldots,T-1$}
    \State Set mask probabilities $\sigma_t=1-t(1-\varepsilon_{\rm sample})/T$ and $\sigma_s=\sigma_t-(1-\varepsilon_{\rm sample})/T$
    \State Evaluate $\widehat{\nn}_{\theta,h}(\cdot\mid x)$ for all positions $h$; positions already unmasked are forced to keep their current token
    \For{each still-masked position $h$}
        \State Sample from the unnormalized distribution with mass $(\sigma_t-\sigma_s)\widehat{\nn}_{\theta,h}(a\mid x)$ on each vocabulary token $a\in[K]$
        \State Assign mass $\sigma_s$ to staying masked
    \EndFor
\EndFor
\State Fill any remaining mask at position $h$ by $\arg\max_{a\in[K]} \widehat{\nn}_{\theta,h}(a\mid x)$
\end{algorithmic}
\end{algorithm}

While strong generation quality is not expected—given the model’s small size (under 200M parameters vs. billions for models like LlaDA \citep{nie2025large}) and limited training data (a few billion tokens vs. trillions for LlaDA)—our goal is simply to assess qualitative improvements over training. We observe that samples do improve over time and report representative examples below.

At $4.9$M tokens,
outputs are mostly token soup. By the support peak at $~265$M tokens and the
frequency peaks at $~334$M and $~452$M tokens, the model can already produce
short spans with recognizable local syntax or domain templates, but many
generations still collapse into repetitions, malformed entities, or abrupt
topic shifts. Extending the run to $3.93$B tokens, some samples start
sustaining a longer topic for several sentences, yet the outputs remain
clearly below high-quality web text.

\paragraph{Illustrative samples.}
\textbf{$4.9$M tokens.}
{\small\ttfamily
best) arise classification companies him community first... mechanism game P
weuner didnall Amenweb years Justice emailf willing childen active run
designedty! York your options bit Gods allSized Matthew alerts' institutions.
\par}

\textbf{$265.4$M tokens (support peak).}
{\small\ttfamily
The Pulitzer award was the award for sponsoring Workstanding Department of Kids
Walk for a free attendance performed in one of the upcoming family time. She
currently previously include the Basketball Line Scholarship and the National
Theater Heritage Center of Year's members of The International Park.
\par}

\textbf{$334.2$M tokens (pairwise-frequency peak).}
{\small\ttfamily
A reputable or company cannot claim protection in the neighbourhood that give
your loan to a real estate. Yes, you will currently re-dating a reasonable
e-gor claim.
\par}

\textbf{$452.2$M tokens (top-1-frequency peak).}
{\small\ttfamily
To locate that language corresponds very very quickly, so it is difficult to
accomplish effective representations. In fact, the analogy is simplified so
allows the language in the subject.
\par}

\textbf{$3.93$B tokens.}
{\small\ttfamily
As a freelance writing editor, it have been important to consider precisely
what it may be if you're working on writing for a New York agency. You usually
need someone as you might to manage this project in place and work to work
with someone who will ultimately help your clients.
\par}

\subsection{How one-token scores are read}

The notation follows \Cref{app:dlm-parameterizations}. The denoiser produces
one logit for each vocabulary token at each position. After a softmax, these
logits define the clean-token predictor
\[
    \widehat{\nn}_{\theta,h}(a\mid \tilde x),
\]
where $\tilde x$ is the corrupted sequence given to the model, \(h\) is the
queried position, and \(a\in[K]\) is a
candidate clean token. In the real-data model used here, the network is not
explicitly time-conditioned, so we suppress the time argument that appears in
the more general notation \(\widehat{\nn}_{\theta,h}(a\mid x_t,t)\).

For a one-token context $c=(\ell,r)$, we form the corrupted sequence
\[
    \tilde x^{(c)}=(\ell,\texttt{[MASK]},r)
\]
and read the score at the masked center position \(h_c\):
\[
    s_\theta(a\mid c)=\log \widehat{\nn}_{\theta,h_c}(a\mid \tilde x^{(c)})
\]
for each possible center token $a$. Equivalently, one may use the raw logits,
since softmax and logarithm preserve the ranking. Candidates and negatives for
the same context are always compared using the same corrupted input
$\tilde x^{(c)}$, so all probe metrics depend only on the relative denoiser
scores of alternative fillers for that slot.

\subsection{Context bank}

The held-out context bank is built from FineWeb validation token shards. We
use one left token and one right token around the center token,
\[
    c=(x^{h-1},x^{h+1}).
\]
We keep contexts with at least 8 observed center-token occurrences, evaluate
up to 8192 contexts, and cap each empirical candidate set at 32 tokens. The
cap keeps evaluation tractable while retaining the repeated-context structure
needed to estimate empirical within-context frequencies.

In the context bank used for the FineWeb figures, the median context
has 22 observed center-token occurrences and 12 distinct observed center
tokens. The mean context has 54 observed occurrences and 23 distinct observed
center tokens. Because candidate sets are capped at 32 tokens, extremely broad
contexts are truncated to their most frequent fillers; this cap retains
$96.7\%$ of the empirical center-token mass on average and $100\%$ at the
median context.

\subsection{Support probe}

For each context $c$, let $C_c$ be its empirical candidate set and let $N_c$
be a set of non-candidate negatives. A non-candidate is not claimed to be
ungrammatical in English. It is simply a token that was not observed as a
center filler for this particular held-out context. The support probe is the
candidate-vs-negative pairwise accuracy
\begin{align}
\label{eq:indirect-supp-probe}
S(\theta)
    =
    \frac{1}{|\mathcal C|}
    \sum_{c\in\mathcal C}
    \frac{1}{|C_c||N_c|}
    \sum_{a\in C_c}
    \sum_{b\in N_c}
    \indicator{s_\theta(a\mid c)>s_\theta(b\mid c)}.
\end{align}
The primary version uses 128 \emph{bank frequency-matched} negatives per
context. The matching is by global token frequency in the held-out context
bank, not by syntax or semantics. Concretely, define
\[
    g(a)=\sum_{c\in\mathcal C} n_c(a),
\]
the total number of times token $a$ appears as an empirical center candidate
anywhere in the context bank. For a context $c$, we select negatives
$b\notin C_c$ from the same context bank with $\log(1+g(b))$ close to the
$\log(1+g(a))$ values of the positive candidates $a\in C_c$. The logarithm is
intentional: context-bank token counts are heavy-tailed, and matching in
log-count space matches candidates and negatives by frequency scale rather than
by raw absolute count.
In words, if a
positive filler is a common token in the context bank, its negative competitor
is also chosen to be common; if the positive is rare, its negative competitor is
chosen to be rare. This prevents the model from passing the support probe by
simply assigning high scores to globally frequent words. We also test other types of negative examples, such as random ones and very
frequent tokens from the whole bank.
The support signal appears at the same stage under all choices, see \Cref{fig:real-dlm-support-negative-controls}. We therefore use frequency-matched negatives in the main text, as a balanced choice: they are neither too hard (like separating against the most frequent negative tokens) nor too easy (like separating against random negative
tokens). 
\begin{figure}[H]
    \centering
    \includegraphics[width=.78\linewidth]{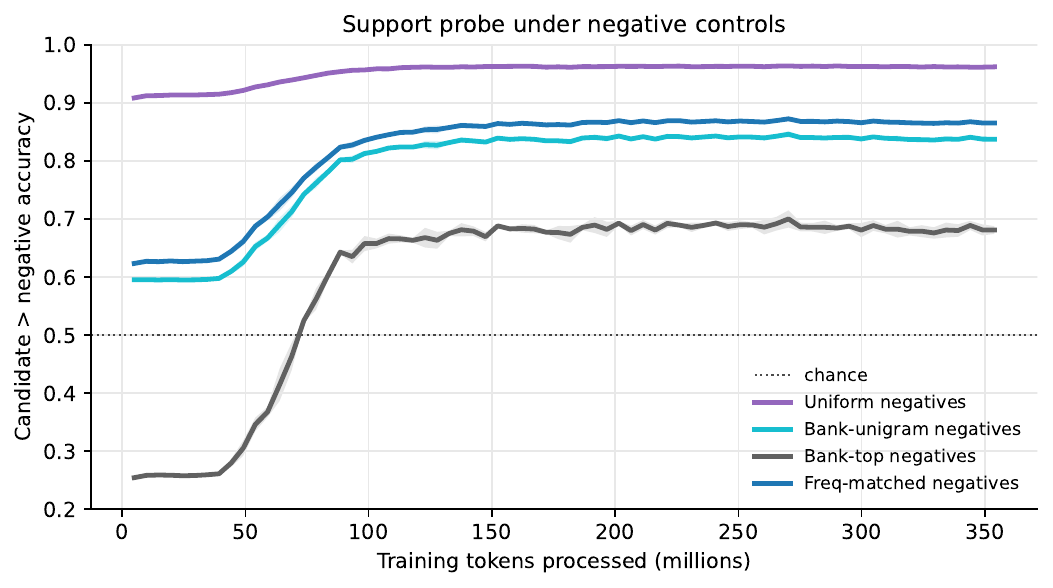}
    \caption{Support probe under different choices of non-candidate negatives.
    Uniform negatives are easiest; negatives drawn from frequent tokens in the
    context bank are harder; frequency-matched negatives are the primary curve
    used in the main text because they compare candidates against
    non-candidates with similar global frequency. The support signal rises at the same time under
    all choices.}
    \label{fig:real-dlm-support-negative-controls}
\end{figure}

\subsection{Frequency probes}

The frequency probes are evaluated only inside the empirical candidate set
$C_c$. Returning to the example above, after we know that \texttt{end},
\texttt{rest}, and \texttt{state} are attested fillers for
\texttt{the [MASK] of}, the frequency probes ignore all non-candidates and ask
whether the model orders those attested fillers according to their held-out
counts. Let $n_c(a)$ be the held-out count of candidate \(a\) in context \(c\).
The pairwise frequency probe is
\begin{align}
\label{eq:indirect-freq-probe}
F_{\mathrm{pair}}(\theta)
    =
    \frac{1}{|\mathcal C|}
    \sum_{c\in\mathcal C}
    \frac{1}{|\mathcal P_c|}
    \sum_{(a,b)\in\mathcal P_c}
    \indicator{s_\theta(a\mid c)>s_\theta(b\mid c)},
    \quad
    \mathcal P_c=\{(a,b):n_c(a)>n_c(b)\}.
\end{align}
The top-1 frequency probe asks whether the highest-scored candidate is the
empirical mode:
\begin{align}
F_{\mathrm{top1}}(\theta)
    =
    \frac{1}{|\mathcal C|}
    \sum_{c\in\mathcal C}
    \indicator{
    \arg\max_{a\in C_c}s_\theta(a\mid c)
    =
    \arg\max_{a\in C_c}n_c(a)}.
\end{align}
Thus the support probe asks whether the model has localized the right set of
plausible center tokens, while the frequency probes ask whether it has learned
their relative empirical probabilities.

\subsection{Transition criterion}

For a trajectory metric $M_t$, we report the first checkpoint reaching a fixed
fraction of its peak gain:
\begin{align}
    \tau_q(M)
    =
    \min\{t:M_t\ge M_0+q(\max_{t'}M_{t'}-M_0)\},
    \qquad q=0.9.
\end{align}
This criterion compares the onset of support localization and frequency
sharpening without requiring either curve to saturate at exactly one.

For the main learning-rate schedule, averaged over three independent seeds,
the frequency-matched support metric reaches $\tau_{0.9}$ at
$116.3\pm5.7$M tokens, while top-1 and pairwise frequency reach it at
$234.3\pm22.2$M and $247.4\pm7.5$M tokens. The seed-wise ranges are
$113.0$--$122.9$M, $211.4$--$255.6$M, and $240.8$--$255.6$M, respectively.
The absolute token counts are specific to this model, data window, context bank,
and transition criterion. Their role is to compare support and frequency under
matched conditions.

\begin{figure}[t]
    \centering
    \includegraphics[width=\linewidth]{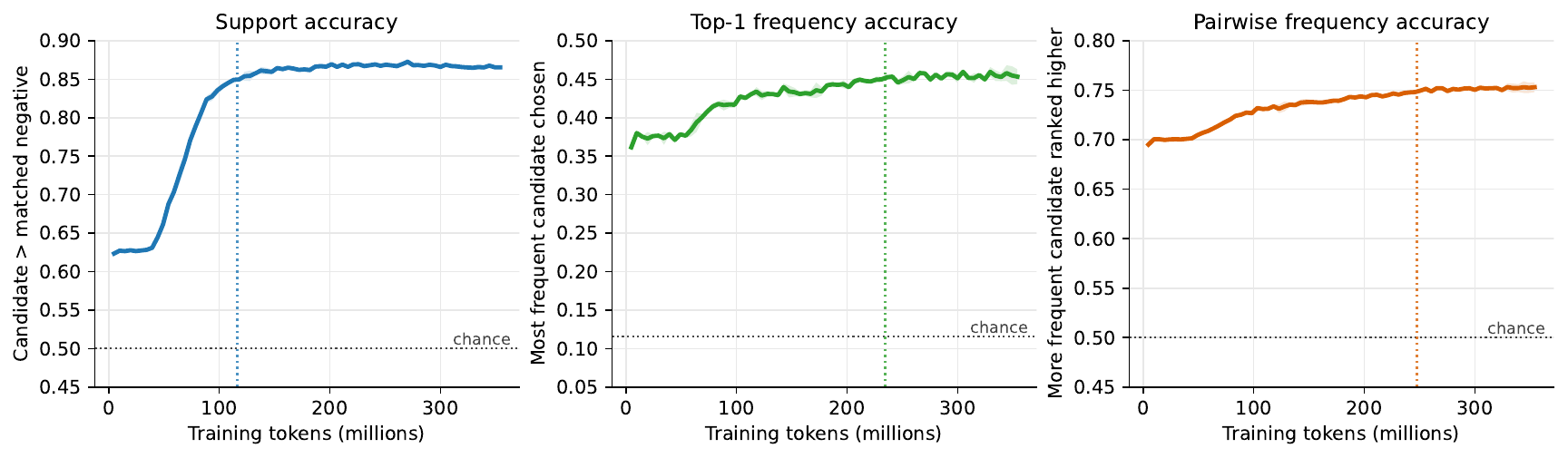}
    \caption{Raw metric levels for the FineWeb masked-DLM probes. The support
    probe rises from $0.624$ to a peak of $0.873$ candidate-vs-negative
    accuracy, with chance at $0.5$. Top-1 frequency recovery rises from
    $0.362$ to $0.464$, compared with a random-candidate baseline of $0.115$.
    Pairwise frequency recovery rises from $0.695$ to $0.756$, with chance at
    $0.5$. Dotted vertical lines mark the $90\%$ peak-gain transition times.}
    \label{fig:real-dlm-raw-levels}
\end{figure}

\subsection{Generalization and memorization diagnostic}

Several recent works study whether diffusion models generalize or memorize
training samples. In particular, \citet{bonnaire2025why} identify two training
timescales in continuous diffusion models: an early time at which high-quality
samples appear, and a later time at which memorization can emerge. Our real-DLM
experiment targets this earlier generalization window. Within that window, we identify a  substructure: 
support localization and within-support frequency matching can be temporally
decoupled.

To see that, we include here a train/validation diagnostic to locate the observed support--frequency separation relative to train-specific memorization.

For each checkpoint, we evaluate the same clean-token predictor on fixed train
and validation windows using fixed masking levels. Let $x$ be a clean token
window, let $\tilde x_\rho$ be the same window after masking a fraction
approximately $\rho$ of positions, and let $M_\rho$ be the masked positions. We
compute the masked-token denoising cross-entropy
\[
    \mathcal L_{\rho}(\theta;x)
    =
    \frac{1}{|M_\rho|}
    \sum_{h\in M_\rho}
    -\log \widehat{\nn}_{\theta,h}(x^h\mid \tilde x_\rho).
\]
This is the clean-token prediction loss family used by the masked-DLM
objective, evaluated at fixed masking levels. We use three masking levels,
masking approximately $15\%$, $50\%$, or almost all tokens in each window. For
a masking level $\rho$, we track
\begin{align}
    G_\rho(t)=
    \mathcal L_{\mathrm{val},\rho}(t)
    -
    \mathcal L_{\mathrm{train},\rho}(t).
\end{align}
A growing positive gap would indicate train-specific memorization. In the
early-token window used for the support and frequency probes, both train and
validation losses improve while the gap remains small. We
do not see the widening positive train advantage expected from train-specific
memorization in the window where support and frequency separate. See \Cref{fig:real-dlm-train-val-gap} for this gap, and \Cref{fig:real-dlm-train-val-nll} for the corresponding absolute train
and validation masked-token denoising cross-entropies, averaged over the three
seeds. 

\begin{figure}[H]
    \centering
    \includegraphics[width=.78\linewidth]{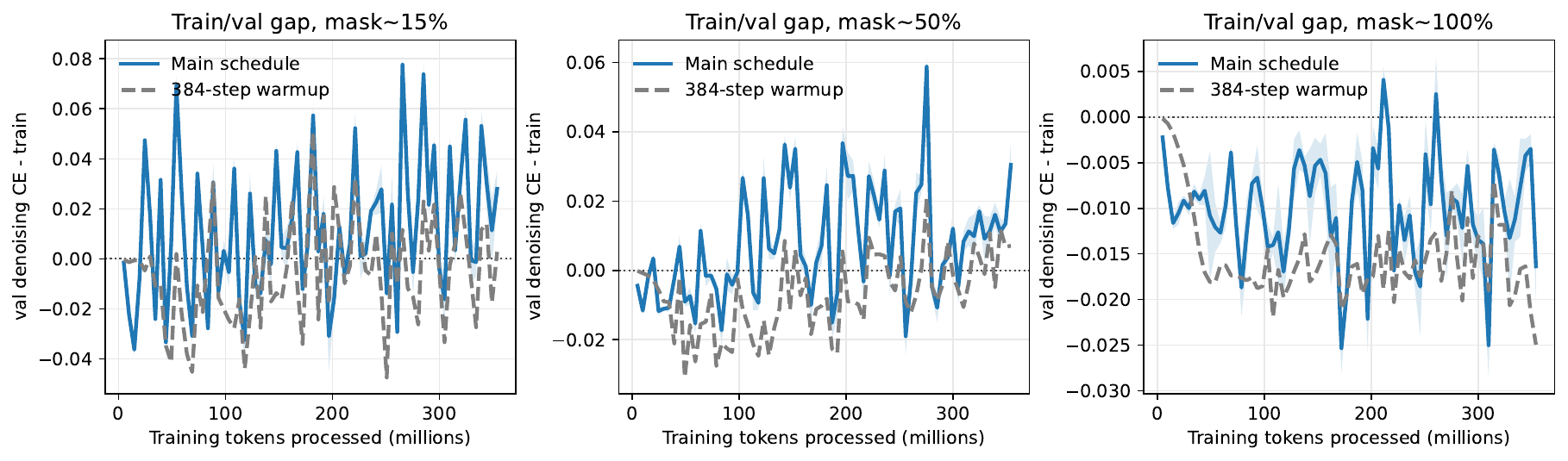}
    \caption{Train/validation diagnostic for the FineWeb masked-DLM trajectory.
    In the support--frequency transition window, train and validation
    masked-token denoising losses improve together and the
    validation-minus-train gap remains small across masking levels.}
    \label{fig:real-dlm-train-val-gap}
\end{figure}

\begin{figure}[H]
    \centering
    \includegraphics[width=.95\linewidth]{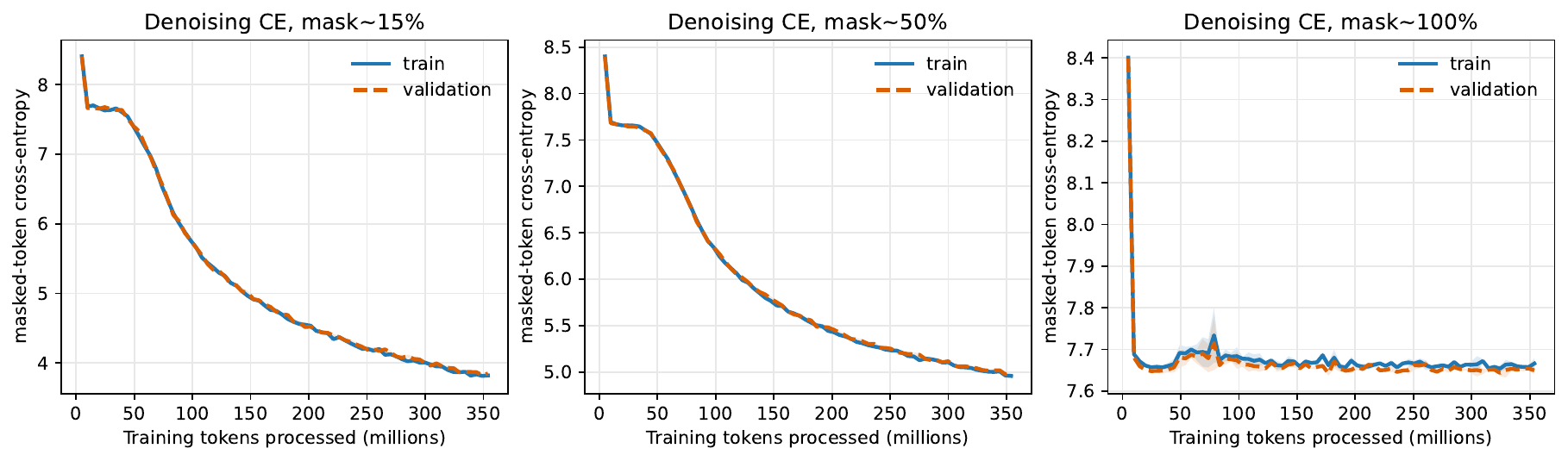}
    \caption{Absolute train and validation masked-token denoising cross-entropies
    for the FineWeb masked-DLM trajectory. Curves are averages over the three seeds at three masking levels.}
    \label{fig:real-dlm-train-val-nll}
\end{figure}

\subsection{Learning-rate schedule robustness}
\label{app:real-dlm-warmup}

Because these transitions occur early in training, we check that the conclusion
is not tied to a single learning-rate schedule. We repeat the probe suite with
the same data, context bank, model, and token window, but insert a 384-step
learning-rate warmup. The support transition moves later under this schedule,
while the frequency transitions do not move by the same amount. In this
single-seed control, the frequency-matched support transition occurs at
$186.8$M tokens, top-1 mode recovery at $191.7$M tokens, and pairwise frequency
ranking at $255.6$M tokens. Thus the exact margin between support and top-1
frequency is schedule-sensitive, but the ordering is not reversed and the
pairwise frequency ranking remains later.

\begin{figure}[H]
    \centering
    \includegraphics[width=\linewidth]{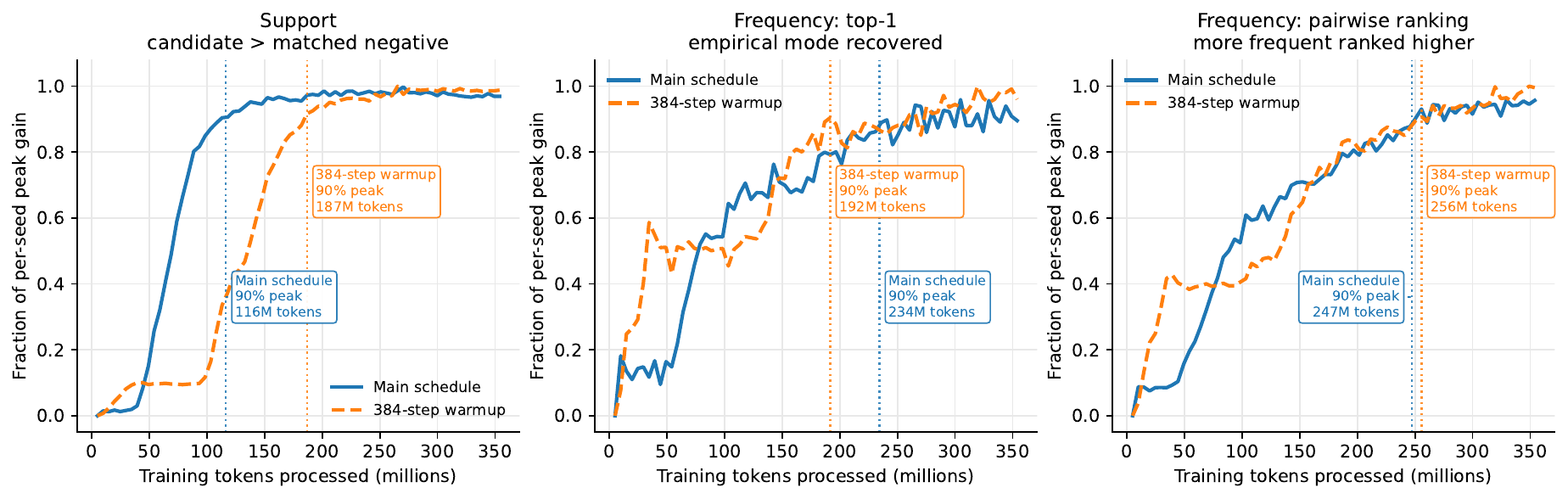}
    \caption{Learning-rate schedule robustness. Labeled dotted lines show the
    $90\%$ peak-gain transition for each curve. Adding a 384-step learning-rate
    warmup delays the support transition and brings the top-1 frequency
    transition close to it, while pairwise frequency ranking remains later.}
    \label{fig:real-dlm-warmup-robustness}
\end{figure}

\subsection{Synthetic counterpart}
\label{app:synthetic-fineweb}

As a check on the interpretation of the probes used on FineWeb, we also run the same experiment in a synthetic setting where the true
distribution $\pData$ is known. In this setting, direct versions of the probes considered before are available, so we can check whether the indirect versions used on FineWeb track their exact oracle counterpart in this synthetic case. \Cref{fig:synthetic-fineweb-counterpart} shows that they do. We now detail the setup below. 

\begin{figure}[H]
    \centering
    \includegraphics[width=\linewidth]{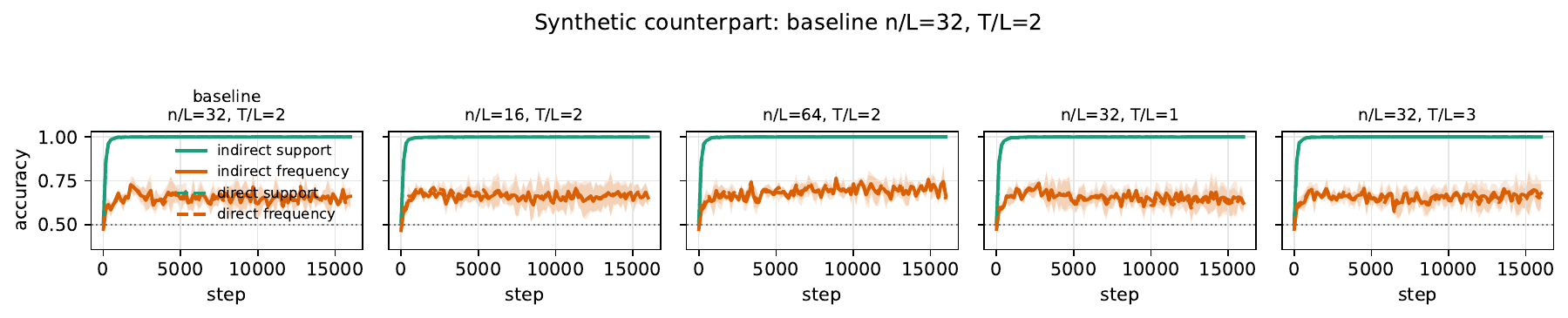}
    \caption{\textbf{Synthetic counterpart to the FineWeb probes in \Cref{eq:indirect-supp-probe,eq:indirect-freq-probe}.}
    Each panel uses length-$64$ sequences over 32 symbols and changes either
    the number of training sequences relative to sequence length ($n/L$) or the
    diffusion horizon relative to sequence length ($T/L$). Solid curves use
    held-out empirical context counts, as in FineWeb, dashed curves use the
    exact random-walk oracle and are almost everywhere coinciding with the indirect versions.}
    \label{fig:synthetic-fineweb-counterpart}
\end{figure}

Instead of defining candidate center tokens by held-out counts as in \Cref{eq:indirect-supp-probe}, 
we define the support probe from the exact conditional support
$\{a:p(a\mid \ell,r)>0\}$, and instead of ranking candidates by empirical
held-out frequencies as in \Cref{eq:indirect-freq-probe}, we rank them by the exact conditional probabilities
$p(a\mid \ell,r)$. 

For a local context $c=(\ell,r)$, let
\begin{align}
    q_c(a)
    :=
    \pData(X^h=a\mid X^{h-1}=\ell,X^{h+1}=r)
    =
    \frac{P(a\mid \ell)P(r\mid a)}
    {\sum_{b}P(b\mid \ell)P(r\mid b)} ,
    \label{eq:synthetic-direct-conditional}
\end{align}
where $P(\cdot\mid\cdot)$ is the known random-walk transition matrix. Let
$C_c^\star=\{a:q_c(a)>0\}$ be the exact candidate set, and let
$N_c^\star$ be the frequency-matched non-candidate set obtained by the same
global-frequency matching rule as in the FineWeb support probe, but excluding
$C_c^\star$. We use the comparison (with $1/2$ in case of tie)
\[
    \psi_\theta^c(a,b)
    =
    \indicator{s_\theta(a\mid c)>s_\theta(b\mid c)}
    +\frac12\indicator{s_\theta(a\mid c)=s_\theta(b\mid c)}.
\]
The direct support probe computed here is
\begin{align}
S_{\mathrm{direct}}(\theta)
    =
    \frac{1}{|\mathcal C|}
    \sum_{c\in\mathcal C}
    \frac{1}{|C_c^\star||N_c^\star|}
    \sum_{a\in C_c^\star}
    \sum_{b\in N_c^\star}
    \psi_\theta^c(a,b).
    \label{eq:synthetic-direct-supp-probe}
\end{align}
The plotted direct frequency probe is the pairwise ranking accuracy
\begin{align}
F_{\mathrm{direct}}(\theta)
    =
    \frac{1}{|\mathcal C|}
    \sum_{c\in\mathcal C}
    \frac{1}{|\mathcal P_c^\star|}
    \sum_{(a,b)\in\mathcal P_c^\star}
    \psi_\theta^c(a,b),
    \quad
    \mathcal P_c^\star
    =
    \{(a,b)\in C_c^\star\times C_c^\star:q_c(a)>q_c(b)\}.
    \label{eq:synthetic-direct-freq-probe}
\end{align}

\Cref{fig:synthetic-fineweb-counterpart} shows
the mean $\pm$ one standard deviation over seeds. The indirect probes based on held-out counts and the direct
oracle probes have the same qualitative behavior. For the
support metric, they coincide on these runs. For the pairwise frequency metric,
their mean absolute difference across checkpoints is about $0.02$ in accuracy. Support saturates early, and
frequency ranking improves later and remains noisier. In this synthetic case
the support probe is particularly easy and saturates much earlier than
frequency.

This synthetic result is only a sanity check in a synthetic case: it shows that, in a
case where the true distribution is known, the held-out-count probes do track the exact
support and frequency versions of the metrics in \Cref{eq:indirect-supp-probe,eq:indirect-freq-probe}.

The data are sequences of length \(H=64\) over the vocabulary
\(\{0,\ldots,K-1\}\) with \(K=32\). The first token is uniform. If the current token is \(v\), the
next token is drawn from the three-token set \(\{v-1,v,v+1\}\), clipped at the
vocabulary boundaries. The baseline transition probabilities favor upward
moves: \(p(v+1\mid v)=0.45\), \(p(v\mid v)=0.35\), and
\(p(v-1\mid v)=0.20\), with a deterministic token-dependent renormalization that
changes the probabilities across positions while leaving the three-token
support unchanged. The baseline setting uses $n=2048$ training sequences
($n/H=32$) and diffusion horizon $T=128$ ($T/H=2$). We also evaluate
$n/H=16$ and $n/H=64$ at fixed $T/H=2$, and $T/H=1$ and $T/H=3$ at fixed
$n/H=32$. Lower $n/H$ means fewer distinct training sequences under the same
optimizer-step budget; changing $T/H$ changes the number of discrete mask-noise
levels.

\end{document}